\documentclass{article} 
\usepackage{iclr2027_conference,times}
\usepackage[table]{xcolor}

\usepackage{amsmath,amsfonts,bm}

\def\eqref#1{equation~\ref{#1}}

\def\1{\bm{1}}

\DeclareMathAlphabet{\mathsfit}{\encodingdefault}{\sfdefault}{m}{sl}
\SetMathAlphabet{\mathsfit}{bold}{\encodingdefault}{\sfdefault}{bx}{n}

\newcommand{\R}{\mathbb{R}}

\DeclareMathOperator*{\argmin}{arg\,min}

\usepackage{hyperref}

\hypersetup{
	colorlinks=true,
	citecolor=blue!65!black!65,
	linkcolor=blue!55!black!65,
	urlcolor=purple!40!black!60,
}

\usepackage{url}

\usepackage{etoc}
\usepackage{titletoc}

\usepackage{subcaption}
\usepackage{amsthm}
\usepackage{amsmath,amssymb,amsfonts,mathtools,bm,mathrsfs}

\usepackage{booktabs}
\usepackage{multirow}
\usepackage{microtype}
\usepackage{graphicx}
\usepackage{xcolor}
\usepackage{algorithm}
\usepackage{algpseudocode}
\usepackage{tikz}
\usetikzlibrary{arrows.meta,positioning,fit}
\usepackage{enumitem}
\usepackage{hyperref}
\usepackage{url}

\usepackage{colortbl} 

\usepackage[capitalize,noabbrev]{cleveref}

\newtheorem{definition}{Definition}
\newtheorem{theorem}{Theorem}
\newtheorem{proposition}{Proposition}
\newtheorem{corollary}{Corollary}
\newtheorem{lemma}{Lemma}

\providecommand{\R}{\mathbb{R}}

\providecommand{\cG}{\boldsymbol{\mathcal{G}}}
\providecommand{\cM}{\boldsymbol{\mathcal{M}}}

\providecommand{\cP}{\boldsymbol{\mathcal{P}}}

\providecommand{\polar}{\operatorname{polar}}

\providecommand{\argmin}{\operatorname*{arg\,min}}

\usepackage{amsmath,amsfonts,amssymb,bm,mathtools}

\title{TUON: A Tensor-Aware Optimizer Utilizing Tensor-SVD and Polarization for Large Language Model Pretraining}

\author{Anda Tang \\
	the State Key Lab of General AI\\
	School of Intelligence Science and Technology\\
	Peking University\\
	Beijing, China\\
	\texttt{tanganda@pku.edu.cn} \\
	\And
	Xuanlin Yang\\
	the State Key Lab of General AI\\
	School of Intelligence Science and Technology\\
	Peking University\\
	Beijing, China\\
	\texttt{2401112159@stu.pku.edu.cn}
	\And
	Junjie Wang \\
	the State Key Lab of General AI\\
	School of Intelligence Science and Technology\\
	Peking University\\
	Beijing, China\\
	\texttt{wangjunjie25@stu.pku.edu.cn} \\
	\And
	Zhouchen Lin \thanks{Corresponding author.} \\
	the Pazhou Laboratory (Huangpu)\\
	the State Key Lab of General AI\\
	School of Intelligence Science and Technology\\
	the Institute for Artificial Intelligence\\
	Peking University\\
	Beijing, China\\
	\texttt{zlin@pku.edu.cn}
}

\iclrfinalcopy 
\begin{document}

\maketitle

\begin{abstract}
	Large scale language model~(LLM) pretraining is increasingly bottlenecked by optimizer efficiency, driving optimizer design from vector-aware methods such as AdamW to matrix-aware methods. 
	Matrix-aware optimizers such as Muon exploit matrix structure but treat layers independently, neglecting cross-layer correlations. Tensor-aware optimizers are therefore needed, yet existing ones rely on a matricization-based tensor algebra and remain restricted to shallow stacking and specific modules. 
	In this paper, we propose \textbf{TUON}, a novel optimizer built on a transform-based tensor algebra derived from the Tensor-SVD (T-SVD). Unlike methods that polarize a globally unfolded matrix, TUON mixes layers via a unitary transform $\boldsymbol{Q}$ and then polarizes each transformed slice, and inverts the transform to recover the update. This mix-then-separate strategy captures cross-layer correlations while preserving per-layer diversity, avoiding the structural dimension imbalance and cross-layer misalignment of unfolding. Thus it successfully  scales to more layers and diverse modules beyond QKV. 
	Theoretically, we prove that TUON performs steepest descent under a $\boldsymbol{Q}$-transformed tensor spectral norm, and identify its update as the tensor polar factor of the T-SVD. We further establish transform non-redundancy, norm and smoothness comparisons with Muon and TEON, and a convergence guarantee. Empirically, TUON improves consistently over Muon and TEON on GPT-124M/350M and Qwen-0.6B/4B, with ablations confirming the transformation mechanism, cross-layer structure beyond two-layer stacking, adaptability beyond QKV, and robustness to the polarization method.
\end{abstract}

\section{Introduction}

As model and data scales continue to grow, the resource intensity of large scale language model~(LLM) pretraining has placed optimization efficiency at the forefront of research~\cite{liu2025cola,liu2025muon, zhang2026lax}.
In this context, optimizers are critical for training efficiency.
The past years have witnessed optimizer advances along different structural granularity levels, from \textit{vector-aware} to \textit{matrix-aware} methods.	
Among them, Adam \citep{kingma2015adam} and AdamW~\citep{loshchilov2017decoupled} remain the default and most widely-used optimizers. It treats parameters via coordinate-wise momentum statistics, placing it within the \textit{vector-aware} family~\citep{xie2024adan,zhang2024adam}. 
Then, \textit{matrix-aware} optimizers~\citep{ gupta2018shampoo,vyas2025soap} 	 are emerging as a promising paradigm for next-generation efficient training algorithms. As a representative, Muon \citep{jordan2024muon} has achieved record-breaking pretraining performance on LLMs \citep{liu2025muon, team2025kimi}. 
Compared with vector-aware optimizers, the improvement may stem from matrix-aware optimizers preserving more information from the matrix structure in Transformers. 

A Transformer naturally has richer structure than the matrix view suggests. Each module repeats across depth with compatible shapes, forming a tensor \(\boldsymbol{\mathcal{W}} \in \mathbb{R}^{m \times n \times K}\) indexing layers. Yet to fit into optimizer frameworks, this structure is usually split into separate matrices \(\boldsymbol{W}^{(1)}, \ldots, \boldsymbol{W}^{(K)} \in \mathbb{R}^{m \times n}\), losing cross-layer information. Existing matrix-aware methods operate layer-wise and thus neglect cross-layer correlations. 
This has motivated \textit{tensor-aware} optimizers. TEON \citep{zhang2026teon} matricizes the third-order gradient tensors along the first and second modes and orthogonalizes them, extending Muon to the tensor level. Tensorion \citep{bogachev2026tensorion} instead relaxes the tensor spectral norm into a tractable surrogate, operating on native higher-order parameters within a single layer and adaptively selecting the unfolding.

Despite their empirical success, existing tensor-aware optimizers exhibit two practical limitations. 
First, they are restricted in scope. TEON is applied for QKV-specific structures and performs worse on other modules such as output projections and MLPs, so its benefits remain confined to a narrow subset of the network.  Tensorion has been applied to vision tasks, leaving its applicability to LLMs unclear. 
Second, they are confined to shallow stacking. TEON supports only two-layer stacking, and its performance degrades as more layers are concatenated, preventing it from realizing the theoretical information gain that larger stacking sizes are expected to yield \citep{zhang2026teon}.


These limitations stem from two structural problems. The first is a dimension imbalance induced by unfolding. As the number of layers grows, the unfolded matrix becomes increasingly rectangular, so its number of singular directions is bounded by its smaller dimension, and may be far smaller than the larger dimension. Thus, these directions span only a small subspace of the full parameter space, causing information loss. 
The second is cross-layer misalignment. Even within the same module type, each slice has its own semantic role and evolves differently during training, which means each dominant subspace may not align. 
Global unfolding concatenates all slices into one matrix to solve for a single shared subspace that approximates all slices collectively. This compromise subspace smooths out slice-specific characteristics and suppresses head- or layer-specific signals. 
Ultimately, these problems stem from the tensor algebra underlying existing tensor-aware optimizers. Addressing them requires a different tensor algebra, which induces a geometry better suited to cross-layer optimization. This raises a natural question:
\par
{\centering  \textit{Can we design a tensor-aware optimizer that captures cross-layer structure beyond two-layer stacking, adapts to diverse module types beyond QKV, and scales across model sizes, while being theoretically interpretable under a proper tensor geometry?} \par
}

In this work, we answer this question affirmatively by introducing \textbf{TUON}, a \textbf{T}ensor-aware optimizer \textbf{U}tilizing tens\textbf{O}r-SVD and polarizatio\textbf{N}, built upon a \textit{transform-based tensor algebra} grounded in the Tensor-SVD (T-SVD)~\citep{kilmer2011factorization,kernfeld2015tensor,zheng2026refta}. 
For third-order tensors in Transformer, TUON applies a unitary transform \(\boldsymbol{Q}\) along the layer mode to exploit cross-layer structure, then polarizes each transformed slice independently, and inverts the transform to recover the update. 
While the algebra only requires \(\boldsymbol{Q}\) to be invertible,  we further require it to decorrelate layer-wise coupling to prevent mutual interference between slices during polarization. Thus we propose \textbf{DecorQ}, a data-dependent transform that diagonalizes the layer Gram matrix and decouples the update into independent orthogonal subproblems. This \textit{mix-then-separate} design captures cross-layer structure while preserving per-layer diversity, avoiding the information loss of global unfolding and scaling gracefully to many layers and diverse modules beyond QKV.

Theoretically, we first give the tensor spectral under  ${\boldsymbol{Q}}$-transformation and its dual. 
Then we prove that TUON is steepest descent for this norm. 
We further show that the update is the tensor polar factor of a \(Q\)-product T-SVD and prove that the transform is non-redundant.
At last, we establish norm comparability and smoothness comparison with Muon and TEON, and provide a convergence guarantee. 
Empirically, TUON delivers consistent improvements across  GPT-124M/350M and Qwen-0.6B/4B. 
Ablations further confirm the effectiveness of the transform DecorQ, its extension beyond two-layer stacking, its adaptability beyond QKV, and its robustness to the polarization method.


Our main contributions are as follows:

%

\begin{itemize}
	\item We propose TUON, a tensor-aware optimizer based on transform-based tensor algebra that better exploits cross-layer information, with the DecorQ transform to decouple layers without mutual interference during polarization. The optimizer applies across layer depths, module types, and model sizes.

	\item We prove that TUON performs steepest descent under the transform-induced norm. Moreover, we identify the update as the tensor polar factor of a \(\boldsymbol{Q}\)-product T-SVD, and prove the non-redundancy of this transform.
	
	\item We establish geometry comparisons with Muon and TEON. The transform distorts the Muon geometry by at most a factor of \(\sqrt{K}\), and at a common operator radius, the TUON oracle is at least as strong as the TEON oracle, and can be up to \(\sqrt{K}\) stronger.
	
	\item We provide comprehensive experiments. TUON improves over Muon and TEON on GPT-124M/350M and Qwen-0.6B/4B, with ablations confirming the transformation mechanism, cross-layer structure beyond two-layer stacking, adaptability beyond QKV, and robustness to the polarization method.
\end{itemize}

\S \ref{sec:method}  presents the TUON method, \S \ref{sec:theory} its theoretical properties, \S \ref{sec:experiments} the experimental results. \S \ref{conclusion} concludes with future directions.
\textbf{Notation}, \textbf{preliminaries}, and a comprehensive \textbf{related work} review are provided in Appendices~\ref{unified-notation}, \ref{app:preliminaries}, and \ref{related-work-appendix}, respectively.


\section{Methodology of TUON}
\label{sec:method}

\subsection{Momentum tensor and layer transform}
For one compatible group, stack the momenta at $t$-th iteration as
$
\cM_t(:,:,k)=\boldsymbol{M}_t^{(k)}$ and $
\cM_t=\beta\cM_{t-1}+(1-\beta)\cG_t$
where $\cM_t\in\R^{m\times n\times K}$. Given a unitary $\boldsymbol{Q}\in \mathbb{C}^{K\times K}$, $\boldsymbol{Q}^{\mathsf H} \boldsymbol{Q}=I$, where $\boldsymbol{Q}^{\mathsf H}$ denotes the conjugate transpose of $\boldsymbol{Q}$, our tensor-product-transform treats every tube as a length-$K$ column,
\begin{equation}
	\widehat{\cM}_t=\cM_t\times_3\boldsymbol{Q},
	\quad
	\text{where the $j$-th slice of $\widehat{\cM}_t$}\quad \widehat{\boldsymbol{M}}_{t,j}=\widehat{\boldsymbol{\mathcal{M}}}_t(:,:,j)=\sum_{k=1}^K \boldsymbol{Q}_{jk}\boldsymbol{M}_t^{(k)}.
	\label{eq:method-transform}
\end{equation}

\paragraph{Choice of the layer transform \(\boldsymbol{Q}\).}
\label{sec:choice-of-q}

In this tensor algebra derived from T-SVD, the layer transform \(\boldsymbol{Q}\) only needs to be unitary (or invertible). Many standard transforms satisfy this requirement, including the DCT, the Haar wavelet transform, and the FFT (discussed in Appendix~\ref{Selection-of-Q}). These fixed transforms are simple and computationally efficient. However, a fixed transform imposes a prior mismatched to actual layer statistics. To see this, let \(\mathsf{M}_t \in \mathbb{R}^{K \times mn}\) be the mode-3 unfolding of the momentum tensor, whose \(k\)-th row is \(\operatorname{vec}(\boldsymbol{M}_t^{(k)})^{\mathsf T}\), and define the layer-wise Gram matrix
\begin{equation}
	\mathsf{C}_t := \mathsf{M}_t \mathsf{M}_t^{\mathsf T} \in \mathbb{R}^{K \times K},
	\qquad
	(\mathsf{C}_t)_{ij} = \langle \boldsymbol{M}_t^{(i)}, \boldsymbol{M}_t^{(j)} \rangle_{\mathrm F}.
\end{equation}
This matrix encodes the pairwise coupling among layers. Large off-diagonal entries indicate strongly correlated update directions, while near-zero entries indicate decoupled evolution. A good choice of \(\boldsymbol{Q}\) should decorrelate this structure so that the transformed slices evolve independently.

This motivates our default choice, \textbf{DecorQ}. We set \(\boldsymbol{Q} = U_3^{\mathsf T}\), where \(U_3\) is the left singular matrix of the SVD \(\mathsf{M}_t = \boldsymbol{U}_3 \boldsymbol{\Sigma}_3 \boldsymbol{V}_3^{\mathsf T}\). Under this transform, the Gram matrix of the transformed slices becomes
\begin{equation}
	\widehat{\mathsf{C}}_t
	= \boldsymbol{Q} \mathsf{M}_t \mathsf{M}_t^{\mathsf T} \boldsymbol{Q}^{\mathsf T}
	= \boldsymbol{U}_3^{\mathsf T} \bigl( \boldsymbol{U}_3 \boldsymbol{\Sigma}_3^2 \boldsymbol{U}_3^{\mathsf T} \bigr) \boldsymbol{U}_3
	= \boldsymbol{\Sigma}_3^2,
\end{equation}
which is diagonal. Hence the transformed slices satisfy
\(\langle \widehat{\boldsymbol{M}}_{t,i}, \widehat{\boldsymbol{M}}_{t,j} \rangle_{\mathrm F} = 0\) for \(i \neq j\), which are mutually orthogonal.  The original coupled $K$-layer of a tensor is therefore decoupled into \(K\) independent single-channel problems, \textit{each evolving in its own orthogonal subspace without mutual interference}. 
DecorQ is data-dependent and is refreshed from the momentum statistics during training, so the transform adapts to the evolving cross-layer structure rather than imposing a fixed prior. We compare DecorQ against DCT, Haar, and DFT in  experiments in \S\ref{Q-ablation}, where DecorQ consistently achieves the best performance. 
Unlike existing approaches that orthogonalize the unfolded matrix as a whole, TUON uses unfolding only to construct \(\boldsymbol{Q}\). 
Full derivations and additional transform choices are provided in Appendix~\ref{Selection-of-Q}.

\subsection{Polarization and inverse transform}\label{Polarization and inverse transform}
Then for each transformed slice $\widehat{\boldsymbol{M}}_{t,j}$, we polarize (or orthogonalize) them, defined as
\begin{equation}
	\widehat{\boldsymbol{P}}_{t,j}=\polar(\widehat{\boldsymbol{M}}_{t,j})=\boldsymbol{U}_{t,j}\boldsymbol{V}_{t,j}^{\top},
\end{equation}
where $\boldsymbol{U}_{t,j}$ and $\boldsymbol{V}_{t,j}$ correspond to the singular value decomposition~(SVD) of $\widehat{\boldsymbol{M}}_{t,j}=\boldsymbol{U}_{t,j}\boldsymbol{\Sigma}_{t,j}\boldsymbol{V}_{t,j}^{\top}$. Then the polarized tensor $\widehat{\cP}_t$ is made up by each polarized slice $\widehat{\boldsymbol{P}}_{t,j}$, i.e., $\widehat{\boldsymbol{P}}_{t,j}=\widehat{\boldsymbol{\mathcal{P}}}_t(:,:,j), \left(j=1,\cdots , K\right)$. 	
Instead of exact SVD, the polar factor is approximated by fast iterative methods such as Newton--Schulz~(NS, Appendix~\ref{Newton--Schulz-polynomial}).

After obtaining the approximate polar factors $\widehat{\cP}_t$, we invert the transform to recover the update in the original tensor domain as 
\begin{equation}
	\cP_t=\widehat{\cP}_t\times_3 \boldsymbol{Q}^{\mathsf H}.
	\label{eq:method-exact-polar}
\end{equation}
To control the overall step magnitude while preserving the direction of the update, we then scale the entire stack by $
\gamma_t = \frac{1}{\operatorname{RMS}(\boldsymbol{\mathcal{P}}_t) + \epsilon}
$, where $\operatorname{RMS}(\boldsymbol{\mathcal{P}}_t)
= \sqrt{ \frac{1}{mnK} \sum_{i=1}^m \sum_{j=1}^n \sum_{k=1}^K |\boldsymbol{\mathcal{P}}_t(i,j,k)|^2 }$ is the root-mean-square value (typically \(1/\sqrt{\max(m,n)}\), as suggested by  the rank--energy identity in Proposition~\ref{prop:rank-energy}), and $\epsilon$  avoids division by zero. 
Since \(\gamma_t\) is shared across all layers, it only changes the effective trust-region radius, not the relative direction across slices. Finally, with decoupled weight decay, $\boldsymbol{\mathcal{W}}_{t+1}
= (1 - \eta_t \lambda) \, \boldsymbol{\mathcal{W}}_t
- \eta_t \, \widetilde{\boldsymbol{\mathcal{P}}}_t, \quad \widetilde{\boldsymbol{\mathcal{P}}}_t = \gamma_t \, \boldsymbol{\mathcal{P}}_t$ 
where $\eta_t$ is the learning rate and $\lambda$ is the weight-decay coefficient. 

\begin{algorithm}[t]
	\caption{TUON}
	\label{alg:aqtmuon}
	\small
	\begin{algorithmic}[1]
		\Require compatible matrices \(\{\boldsymbol{W}_t^{(k)}\}_{k=1}^K\),  gradients \(\{\boldsymbol{G}_t^{(k)}\}_{k=1}^K\) and momentum \(\{\boldsymbol{M}_t^{(k)}\}_{k=1}^K\); learning rate $\eta_t$, momentum coefficient $\beta$, weight decay coefficient  $\lambda$, NS steps \(J\) 
		\State \(\boldsymbol{\mathcal{G}}_t\gets \operatorname{stack}_3(\boldsymbol{G}_t^{(1)},\ldots,\boldsymbol{G}_t^{(K)})\)
		\State \(\boldsymbol{\mathcal{M}}_t\gets\beta \boldsymbol{\mathcal{M}}_{t-1}+(1-\beta)\boldsymbol{\mathcal{G}}_t\)
		\State Compute the Gram matrix \(\mathsf{C}_t := \mathsf{M}_t \mathsf{M}_t^{\mathsf T} \in \mathbb{R}^{K \times K}\) with
		\((\mathsf{C}_t)_{k\ell} = \langle \boldsymbol{\mathcal{M}}_t(:,:,k), \boldsymbol{\mathcal{M}}_t(:,:,\ell) \rangle_{\mathrm F}\)
		\State Compute eigendecomposition \(\mathsf{C}_t = \mathbf{U}_t \boldsymbol{\Lambda}_t \mathbf{U}_t^{\mathsf T}\)
		\State Set \(\boldsymbol{Q}_t \gets \mathbf{U}_t^{\mathsf T}\) \Comment{Refresh per 200 iterations}
		\State \(\widehat{\boldsymbol{\mathcal{M}}}_t\gets\boldsymbol{\mathcal{M}}_t\times_3\boldsymbol{Q}_t\)
		\For{\(j=1,\ldots,K\) \textbf{in parallel}}
		\State \(\widehat{\boldsymbol{\mathcal{P}}}_t(:,:,j)\gets \operatorname{polarization}(\widehat{\boldsymbol{\mathcal{M}}}_t(:,:,j))\) \Comment{Newton--Schulz, \(J\) steps}
		\EndFor
		\State \(\boldsymbol{\mathcal{P}}_t\gets\widehat{\boldsymbol{\mathcal{P}}}_t\times_3\boldsymbol{Q}_t^{\mathsf H}\)
		\State \(\boldsymbol{\mathcal{W}}_{t+1}\gets(1-\eta_t\lambda)\boldsymbol{\mathcal{W}}_t-\eta_t\gamma_t\boldsymbol{\mathcal{P}}_t\) with \(\gamma_t\gets 1/(\operatorname{RMS}(\boldsymbol{\mathcal{P}}_t)+\epsilon)\)
	\end{algorithmic}
\end{algorithm}

\begin{figure}[!ht]
	\centering
	{
		\includegraphics[width=0.7\textwidth]{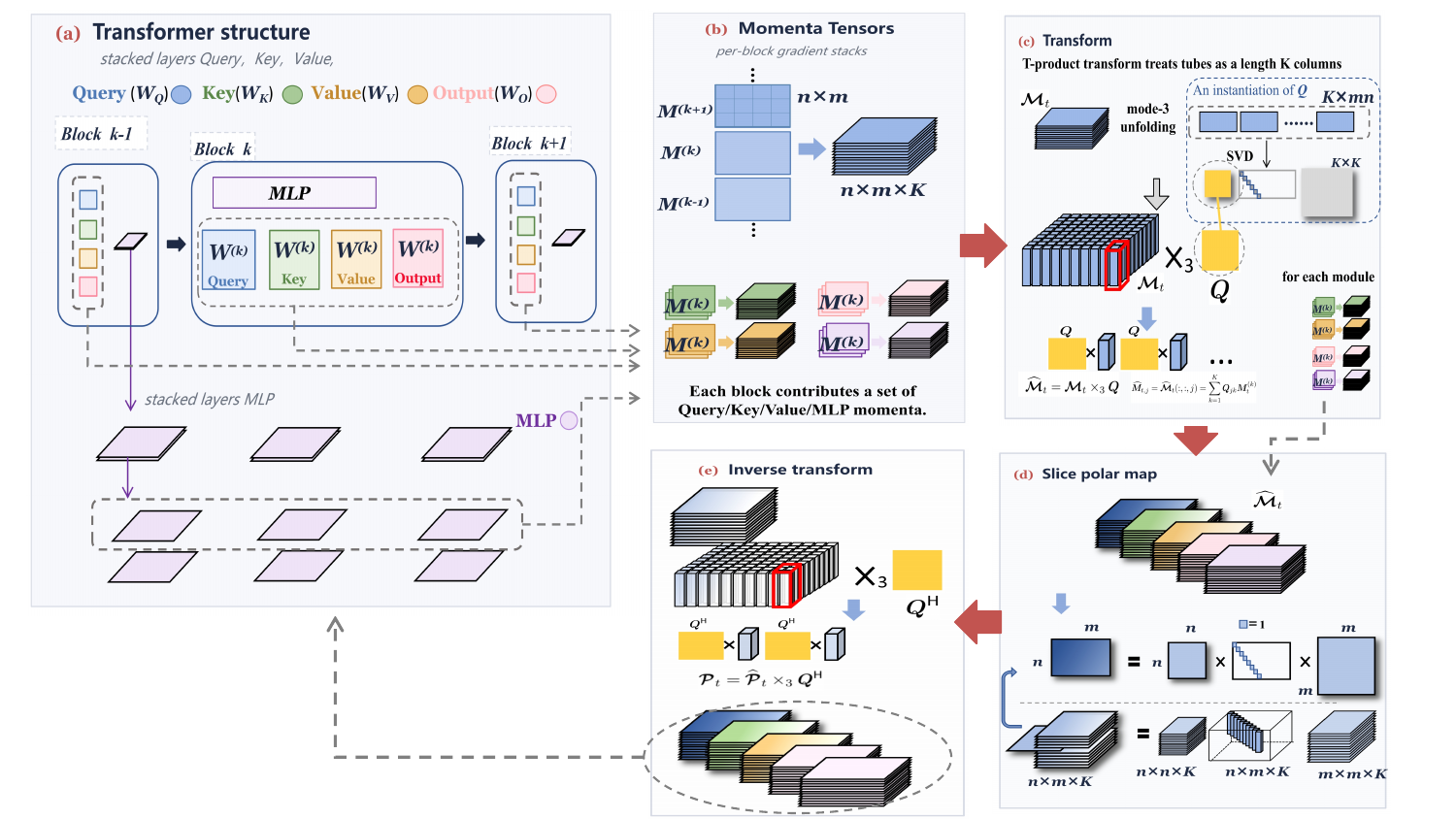}
		\label{}
	}
	\caption{Illustration of the TUON update}
	\label{Illustrations of flowchart} 
\end{figure}

\paragraph{Complexity analysis.}\label{complexity}
Table~\ref{tab:complexity} compares the resulting complexities. In the typical regime \(K \ll \min(m,n)\), TUON has the comparable asymptotic cost as Muon and TEON.  A detailed derivation of the per-iteration complexity is deferred to Appendix~\ref{complexity-appendix}.

\begin{table}[!ht]
	\centering
	\caption{Per-iteration computational complexity for a single tensor group of \(K \times m \times n\).}
	\label{tab:complexity}
	\small
	\setlength{\tabcolsep}{6pt}
	\begin{tabular}{@{}lll@{}}
		\toprule
		\textbf{Optimizer} & \textbf{Per-iteration complexity} & \textbf{Dominant term} \\
		\midrule
		Muon & \(O\!\left(K\,mn \cdot \min(m,n)\right)\) & \(K\,mn \cdot \min(m,n)\) \\
		TEON & \(O\!\left(K\,mn \cdot \min(m,nK)\right)\) & \(Km n  \cdot \min(m,nK)\) \\
		TUON & \(O\!\left(K\,mn \cdot \min(m,n) + K^2 mn\right)\) & \(K\,mn \cdot \min(m,n)\) \\
		\bottomrule
	\end{tabular}
\end{table}

\section{Steepest descent theory and comparative advantages}
\label{sec:theory}

This section develops the complete theory of TUON.   
Preliminaries including Matrix SVD, operator/nuclear norm duality, polar factors, norm-induced LMO, T-SVD, and norm-smoothness are reviewed in Appendix~\ref{app:preliminaries}.  The Muon and TEON geometries, including their trust-region updates, are stated in Appendix~\ref{app:baselines}.  Detailed step-by-step proofs of every result in this section are given in Appendix~\ref{app:proofs}.  
The theoretical development answers \textbf{four questions}.  \textbf{i)} Is the update in our method the exact steepest descent direction under some well-defined norm geometry~(Sec.~\ref{subsec:exact-geometry})? 
\textbf{ii)} Does the layer transform cancel algebraically with the inverse transform (Sec.~\ref{subsec:nonredundancy-energy})? \textbf{iii)} How does the geometry compare with Muon and TEON (Sec.~\ref{subsec:muon-comparison})? \textbf{iv)} Can a strict convergence be certified  (Sec.~\ref{subsec:convergence-theory})?


Let $\boldsymbol{\mathcal{X}}\in\mathbb{R}^{m\times n\times K}$ stack $K$ compatible matrices, with frontal slices $\boldsymbol{X}_k:=\boldsymbol{\mathcal{X}}(:,:,k)$.  
Define mode-1 unfolding as
$
\mathscr U(\boldsymbol{\mathcal{X}}):=[\boldsymbol{X}_1,\ldots,\boldsymbol{X}_{nK}]\in\mathbb{C}^{m\times nK}.
$
Under this convention, 
$
\mathscr U(\boldsymbol{\mathcal{X}}\times_3\boldsymbol{Q})
=\mathscr U(\boldsymbol{\mathcal{X}})R_{\boldsymbol{Q}}$
, $
R_{\boldsymbol{Q}}:={\boldsymbol{Q}}^{\mathsf T}\otimes I_n,
$
where $\otimes$ is the Kronecker product and $I_n$ is the identity matrix of size $n$. Since both $\boldsymbol{Q}$ and $I_n$ are unitary, their Kronecker product $R_{\boldsymbol{Q}}$ is unitary.  Below, \(\|\cdot\|_{\mathrm{Muon}}\) and \(\|\cdot\|_{\mathrm{TEON}}\) denote the spectral norms used by Muon and TEON, respectively; their duals and optimizer maps are given in Appendix~\ref{app:baselines}. Appendix~\ref{app:q-block-is-norm} verifies directly that $\|\cdot\|_{{\boldsymbol{Q}},\mathrm{Ts}}$ is a norm. Proposition~\ref{prop:q-dual} identifies the dual norm associated with TUON, which measures the largest first-order alignment with a direction in the unit trust region. The proof can be found in Appendix~\ref{app:proofs}.

\begin{definition}[Tensor spectral and nuclear norms under  ${\boldsymbol{Q}}$-transform]
	\label{def:q-norms}
	Define $
	\|\boldsymbol{\mathcal{X}}\|_{{\boldsymbol{Q}},\mathrm{Ts}}
	:=\max_{1\le j\le K}\|\widehat{\boldsymbol{X}}_j\|_{\mathrm{op}}$ and
	$
	\|\boldsymbol{\mathcal{X}}\|_{{\boldsymbol{Q}},\mathrm{Tn}}
	:=\sum_{j=1}^K\|\widehat{\boldsymbol{X}}_j\|_* 
	$
	where \(\|\cdot\|_{\mathrm{op}}\) and \(\|\cdot\|_*\) denote the matrix operator (spectral) norm and nuclear norm, respectively.
\end{definition}


\begin{proposition}[Dual norm]
	\label{prop:q-dual}
	For unitary ${\boldsymbol{Q}}$, $\|\cdot\|_{{\boldsymbol{Q}},\mathrm{Tn}}$ is the dual norm of $\|\cdot\|_{{\boldsymbol{Q}},\mathrm{Ts}}$, defined as  
			$\sup_{\|\boldsymbol{\mathcal{D}}\|_{{\boldsymbol{Q}},\mathrm{Ts}}\le1}
			\langle\boldsymbol{\mathcal{G}},\boldsymbol{\mathcal{D}}\rangle_{\mathrm F}
			=\|\boldsymbol{\mathcal{G}}\|_{{\boldsymbol{Q}},\mathrm{Tn}}$. 
\end{proposition}



\subsection{Steepest descent and T-SVD polar}
\label{subsec:exact-geometry}



For clarity, we fix the notation in this part as follows. Define the update of TUON by a general $\boldsymbol{Q}$ transform as  $
\boldsymbol{\mathcal{P}}
:=
\bigl( \operatorname{polar}\left((\boldsymbol{\mathcal{M}} \times_3 \boldsymbol{Q})_{k\in \{1,\cdots, K\}} \right)\bigr) \times_3 \boldsymbol{Q}^{\mathsf H}.
$ 
From Theorem~\ref{thm:q-lmo}, we will show that TUON is normalized steepest descent under the \(\boldsymbol{Q}\)-transformed tensor spectral norm, rather than a heuristic composition of matrix normalizations. A shared positive scale on all slices only changes the effective trust-region radius \(\rho\), not the update direction.

\begin{theorem}
	\label{thm:q-lmo}
	For every radius $\rho > 0$, $-\rho \boldsymbol{\mathcal{P}}
	\in
	\argmin_{\| \boldsymbol{\mathcal{D}} \|_{\boldsymbol{Q},\mathrm{Ts}} \le \rho}
	\langle \boldsymbol{\mathcal{M}}, \boldsymbol{\mathcal{D}} \rangle_{\mathrm F}$ 
	and the optimal value is $
	-\rho \| \boldsymbol{\mathcal{M}} \|_{\boldsymbol{Q},\mathrm{Tn}}
	= -\rho \sum_{j=1}^K \| \widehat{\boldsymbol{M}}_j \|_*.
	$
\end{theorem}

The update also admits a tensor-algebraic interpretation. Let
$
\boldsymbol{\mathcal{M}}
= \boldsymbol{\mathcal{U}} *_{\boldsymbol{Q}} \boldsymbol{\mathcal{S}} *_{\boldsymbol{Q}} \boldsymbol{\mathcal{V}}^{\top}
$
be the T-SVD  of \(\boldsymbol{\mathcal{M}}\) (see Appendix~\ref{app:qproduct}), where \(\boldsymbol{\mathcal{S}}\) is \(f\)-diagonal with the nonzero singular values of the transformed slices on its diagonal. Let \(\boldsymbol{\mathcal{J}}_{\mathrm{supp}}\) be the \(f\)-diagonal tensor obtained from \(\boldsymbol{\mathcal{S}}\) by replacing every nonzero diagonal entry with \(1\). Analogously to the matrix polar factor \(\boldsymbol{U}\boldsymbol{V}^{\top}\) of $\boldsymbol{A} = \boldsymbol{U} \boldsymbol{\Sigma} \boldsymbol{V}^{\top}$, the tensor polar factor is
$
\operatorname{tpolar}_{\boldsymbol{Q}}(\boldsymbol{\mathcal{M}})
:= \boldsymbol{\mathcal{U}} *_{\boldsymbol{Q}} \boldsymbol{\mathcal{J}}_{\mathrm{supp}} *_{\boldsymbol{Q}} \boldsymbol{\mathcal{V}}^{\top}.
$
The following proposition shows that this tensor-algebraic object is exactly the optimizer map from Theorem~\ref{thm:q-lmo}.

\begin{proposition}[tensor polar equivalence]
	\label{prop:q-tpolar}
	The optimizer map $\boldsymbol{\mathcal{P}}$ equals the partial-isometry factor of the generalized T-SVD:
	$
	\boldsymbol{\mathcal{P}}
	= \operatorname{tpolar}_{\boldsymbol{Q}}(\boldsymbol{\mathcal{M}})
	= \boldsymbol{\mathcal{U}} *_{\boldsymbol{Q}} \boldsymbol{\mathcal{J}}_{\mathrm{supp}} *_{\boldsymbol{Q}} \boldsymbol{\mathcal{V}}^{\top}.
	$
	For $K = 1$, this reduces to matrix Muon.  For $\boldsymbol{Q} = \boldsymbol{I}$, it reduces to independent layer-wise polar maps.
\end{proposition}
The proposed direction is thus both the exact LMO for a tensor-spectral trust region and the generalized polar factor of a third-order tensor. Our novelty lies in specializing the \(t\)-product and tensor polar decomposition to cross-layer optimization, which enables the comparative, convergence, and oracle analyses in \S~\ref{subsec:muon-comparison}. These proofs can be found in Appendices~\ref{app:proofs2} and \ref{prop:q-tpolar-proof-appendix}.


\paragraph{Non-redundancy of the Layer Transform.}\label{subsec:nonredundancy-energy} A natural concern is that the layer transform $\boldsymbol{Q}$ might be algebraically redundant, since the update applies $\boldsymbol{Q}$ and later $\boldsymbol{Q}^{\mathsf H}$, the two might simply cancel, leaving no real effect. 
For a TEON-style unfolding polar, a unitary change of layer basis cancels exactly, so it cannot adapt the update to the transform. In TUON, however, the transform is applied \emph{before} the nonlinear slice-wise polar maps, so it changes the tensor partition. This establishes an algebraic distinction between the two approaches. 
The following theorem clarifies when this cancellation occurs.

\begin{theorem}[Unfolding cancellation versus slice-wise non-cancellation]
	\label{thm:nonredundancy}
	Let \(\boldsymbol{M}_{(1)} = \mathscr U(\boldsymbol{\mathcal{M}})\) and \(\boldsymbol{R}_{\boldsymbol{Q}} = \boldsymbol{Q}^{\mathsf T} \otimes \boldsymbol{I}_n\). For the matrix polar map, 
	$
	\operatorname{polar}(\boldsymbol{M}_{(1)} \boldsymbol{R}_{\boldsymbol{Q}}) \boldsymbol{R}_{\boldsymbol{Q}}^{\mathsf H}
	= \operatorname{polar}(\boldsymbol{M}_{(1)}).
	$ 
	Hence a unitary layer transform is redundant under a global unfolding polar. In contrast, under the algebra used by TUON,  
	$
	\bigl( \operatorname{polar}_{\mathrm{slice}}(\boldsymbol{\mathcal{M}} \times_3 \boldsymbol{Q}) \bigr) \times_3 \boldsymbol{Q}^{\mathsf H}
	\neq
	\operatorname{polar}_{\mathrm{slice}}(\boldsymbol{\mathcal{M}}),
	$ cancellation does not occur in general.
\end{theorem}



\begin{proposition}
	\label{prop:rank-energy}
	Let $r_j := \operatorname{rank}(\widehat{\boldsymbol{M}}_j)$.  Then $
	\| \boldsymbol{\mathcal{P}} \|_{\mathrm F}^2
	= \sum_{j=1}^K r_j
	$ and $
	\operatorname{RMS}(\boldsymbol{\mathcal{P}})
	= \sqrt{ \sum_j r_j/(mnK)  }.
	$
	If slices are full rank, then
	$
	\operatorname{RMS}(\boldsymbol{\mathcal{P}}_{\boldsymbol{Q}})
	= 1/(\sqrt{\max(m,n)}),
	$
	independently of $K$.
\end{proposition}
Since the polar map normalizes all nonzero singular values to one, the update energy depends only on the number of active singular directions, not on the momentum magnitude. Proposition~\ref{prop:rank-energy} gives this energy exactly and shows that stacking more layers does not inflate the RMS by a factor of \(\sqrt{K}\), which justifies the stack-scale factor for normalization in \S\ref{Polarization and inverse transform}. Proofs are in Appendix~\ref{app:proofs4} and~\ref{app:proofs5}.


\subsection{Comparison with Muon and TEON}
\label{subsec:muon-comparison}


In this part,  we show that any unitary layer transform distorts the Muon geometry by at most a factor of \(\sqrt K\), where \(K\) is the number of stacked layers.  
The precise Muon norm pair and its polar LMO are reviewed in Appendix~\ref{app:muon-baseline}.  Here we only use the notation \(\|\cdot\|_{\mathrm{Muon}}\) and \(\|\cdot\|_{\mathrm{Muon},*}\).


\begin{theorem}
	\label{thm:q-muon-comparison}
	Define $
	\alpha_{\boldsymbol{Q}}:=\max_j\sum_{k=1}^K|\boldsymbol{Q}_{jk}|$ and $
	\beta_{\boldsymbol{Q}}:=\max_k\sum_{j=1}^K|\boldsymbol{Q}_{jk}|.
	$ Since \(\boldsymbol{Q}\) is unitary, \(\alpha_{\boldsymbol{Q}},\beta_{\boldsymbol{Q}}\le\sqrt K\).
	Then, for every tensor \(\boldsymbol{\mathcal{X}}\) and for their duals \(\boldsymbol{\mathcal{G}}\),
	\begin{equation}
		\begin{aligned}
			&\frac{1}{\beta_{\boldsymbol{Q}}}\|\boldsymbol{\mathcal{X}}\|_{\mathrm{Muon}}
			\le\|\boldsymbol{\mathcal{X}}\|_{\boldsymbol{Q},\mathrm{Ts}}
			\le\alpha_{\boldsymbol{Q}}\|\boldsymbol{\mathcal{X}}\|_{\mathrm{Muon}}, \quad
			\frac{1}{\alpha_{\boldsymbol{Q}}}\|\boldsymbol{\mathcal{G}}\|_{\mathrm{Muon},*}
			\le\|\boldsymbol{\mathcal{G}}\|_{\boldsymbol{Q},\mathrm{Tn}}
			\le\beta_{\boldsymbol{Q}}\|\boldsymbol{\mathcal{G}}\|_{\mathrm{Muon},*}.
		\end{aligned}
		\label{eq:q-muon-primal}
	\end{equation}
\end{theorem}

The bounds, 
$
\frac{1}{\beta_{\boldsymbol{Q}}}\|\boldsymbol{\mathcal{X}}\|_{\mathrm{Muon}}
\le \|\boldsymbol{\mathcal{X}}\|_{\boldsymbol{Q},\mathrm{Ts}}
\le \alpha_{\boldsymbol{Q}}\|\boldsymbol{\mathcal{X}}\|_{\mathrm{Muon}}
$, 
ensure that the transform neither collapses nor explodes the trust-region scale. Their dual counterparts control the gradient's nuclear norm: the best first-order decrease under the \(\boldsymbol{Q}\) geometry is at least \(1/\sqrt{K}\) and at most \(\sqrt{K}\) times that under the layer-wise Muon geometry. This is the key stability statement: \emph{no unitary layer transform can distort the Muon geometry by more than a factor of \(\sqrt{K}\), regardless of the data or the choice of \(\boldsymbol{Q}\).}

\begin{corollary}[Smoothness comparison]
	\label{cor:q-muon-smoothness}
	Let $L_{\boldsymbol{Q}}$ and $L_{\mathrm{Muon}}$ be the smallest constants satisfying
	$
	\|\nabla F(\boldsymbol{\mathcal{X}})-\nabla F(\boldsymbol{\mathcal{Y}})\|_{{\boldsymbol{Q}},\mathrm{Tn}}
	\le L_{\boldsymbol{Q}}\|\boldsymbol{\mathcal{X}}-\boldsymbol{\mathcal{Y}}\|_{{\boldsymbol{Q}},\mathrm{Ts}}
	$
	and the analogous inequality under the Muon norm pair.  Then
	$
	\frac{1}{\alpha_{\boldsymbol{Q}}^2}L_{\mathrm{Muon}}
	\le L_{\boldsymbol{Q}}
	\le\beta_{\boldsymbol{Q}}^2L_{\mathrm{Muon}}.
	$
\end{corollary}
In Corollary~\ref{cor:q-muon-smoothness}, the induced smoothness constants differ by at most a factor $K$.  This establishes the maximum scale at which a convergence improvement can arise, while also preventing an unsupported claim of arbitrary acceleration from a mere coordinate change.




\begin{theorem}[LMO gain with TEON]
	\label{thm:teon-q-sandwich}
	For every unitary $\boldsymbol{Q}$ and tensor $\boldsymbol{\mathcal{X}}$, $
	\|\boldsymbol{\mathcal{X}}\|_{\boldsymbol{Q},\mathrm{Ts}}
	\le\|\boldsymbol{\mathcal{X}}\|_{\mathrm{TEON}}
	\le\sqrt K\,\|\boldsymbol{\mathcal{X}}\|_{\boldsymbol{Q},\mathrm{Ts}}.
	$
	Consequently, $
	\|\boldsymbol{\mathcal{G}}\|_{\mathrm{TEON},*}
	\le\|\boldsymbol{\mathcal{G}}\|_{\boldsymbol{Q},\mathrm{Tn}}
	\le\sqrt K\,\|\boldsymbol{\mathcal{G}}\|_{\mathrm{TEON},*}.
	$
	Thus, at a common radius $\rho$, $
	\rho\|\boldsymbol{\mathcal{G}}\|_{\mathrm{TEON},*}
	\le\rho\|\boldsymbol{\mathcal{G}}\|_{\boldsymbol{Q},\mathrm{Tn}}
	\le\sqrt K\,\rho\|\boldsymbol{\mathcal{G}}\|_{\mathrm{TEON},*}.
	$
\end{theorem}
Then, we compare the linearized decrease with TEON. The TEON norm/dual norm, and TEON polar update are stated in Appendix~\ref{app:teon-baseline}.  
Thus, at equal operator radius, the exact TUON first-order oracle is never weaker than the TEON oracle and may be up to $\sqrt K$ stronger.  These proofs can be found in Appendix~\ref{thm:q-muon-comparison-proof-appendix}, \ref{cor:q-muon-smoothness-proof-appendix} and \ref{thm:teon-q-sandwich-proof-appendix}.

\subsection{Convergence}
\label{subsec:convergence-theory}

\begin{theorem}
	\label{thm:correlation-gain}
	Let \(F\) be the training objective, lower bounded by \(F_\star\), \(\eta\) is the step size and \(T\) is the number of iterations. 
	Suppose $F$ is lower bounded by $F_\star$ and is $L_{\boldsymbol{Q}}$-smooth under $\|\cdot\|_{\boldsymbol{Q},\mathrm{Ts}}$.  For
	$
	\boldsymbol{\mathcal{W}}_{t+1}
	=\boldsymbol{\mathcal{W}}_t-\eta\boldsymbol{\mathcal{P}}_{\boldsymbol{Q}}(\nabla F(\boldsymbol{\mathcal{W}}_t)),
	$
	we have
	$
	\min_{0\le t<T}\|\nabla F(\boldsymbol{\mathcal{W}}_t)\|_{\boldsymbol{Q},\mathrm{Tn}}
	\le\frac{F(\boldsymbol{\mathcal{W}}_0)-F_\star}{\eta T}+\frac{L_{\boldsymbol{Q}}\eta}{2}.
	$
	Choosing
	$
	\eta=\sqrt{\frac{2(F(\boldsymbol{\mathcal{W}}_0)-F_\star)}{L_{\boldsymbol{Q}}T}}
	$
	gives
	\begin{equation}
		\min_{0\le t<T}\|\nabla F(\boldsymbol{\mathcal{W}}_t)\|_{\mathrm F}
		\le\min_{0\le t<T}\|\nabla F(\boldsymbol{\mathcal{W}}_t)\|_{\boldsymbol{Q},\mathrm{Tn}}
		\le\sqrt{\frac{2L_{\boldsymbol{Q}}(F(\boldsymbol{\mathcal{W}}_0)-F_\star)}{T}}.
		\label{eq:q-rate-optimized}
	\end{equation}
\end{theorem}

Corollary~\ref{cor:q-muon-smoothness} bounds the smoothness gap between TUON and Muon by at most a factor \(K\). Together with Theorem~\ref{thm:correlation-gain}, this yields a convergence improvement of at most \(\sqrt{K}\) and a tighter \(O(T^{-1/2})\) stationarity certificate. The proof is in Appendix~\ref{covergence-proof-appendix}.

\section{Experiments}
\label{sec:experiments}

We evaluate TUON across three dimensions to thoroughly assess its effectiveness and generality.
(i) \textbf{Architecture and Scalability.} GPT-124M/350M and Qwen-0.6B/4B verify broad applicability and scalability~(\S \ref{Results on LLM pretraining}). 
(ii) \textbf{Ablation.} Ablations isolate the transformation mechanism, cross-layer stacking beyond two layers, module adaptability beyond QKV, and robustness to the polarization method (\S \ref{sec:ablations}).
(iii) \textbf{Cost.} 
Step-time comparison shows that TUON adds only marginal overhead and scales gracefully (\S \ref{sec:stability-cost}).

\paragraph{Baselines.}
We compare TUON against representative optimizers from three levels of structural granularity. 
(i) \textbf{AdamW} \citep{loshchilov2017decoupled} is the de facto vector-aware standard for large-scale pretraining. 
(ii) \textbf{Muon} \citep{jordan2024muon} is a matrix-aware optimizer that has recently demonstrated remarkable breakthroughs in LLM pretraining. 
(iii) \textbf{TEON} \citep{zhang2026teon} is a recent tensor-aware optimizer that has consistently outperformed Muon across model scales and architectures. Tensorion is also a tensor-aware optimizer, but it has been validated mainly on vision tasks, so we use TEON rather than Tensorion as our tensor-aware baseline. 
Together, these baselines span vector-, matrix-, and tensor-level optimization. 
More details can be found in Appendix~\ref{baseline-appendix}. 

\paragraph{Experimental details.} \label{Experimental details}
All baselines follow \textbf{the same configurations} , ensuring a fair comparison. Comprehensive experimental details, such as computational resources, configurations, implementation specifics, are elaborated in  Appendix~\ref{sec:exp-setup}. 
Our implementation strictly follows the PyTorch official API. 
\textit{Fully functional, production-ready code} is publicly released (see Appendix~\ref{Code and Reproducibility} for link and implement details).

\begin{figure}[!htbp]
	\centering
	\begin{subfigure}[b]{0.45\textwidth}
		\centering
		\includegraphics[width=\textwidth]{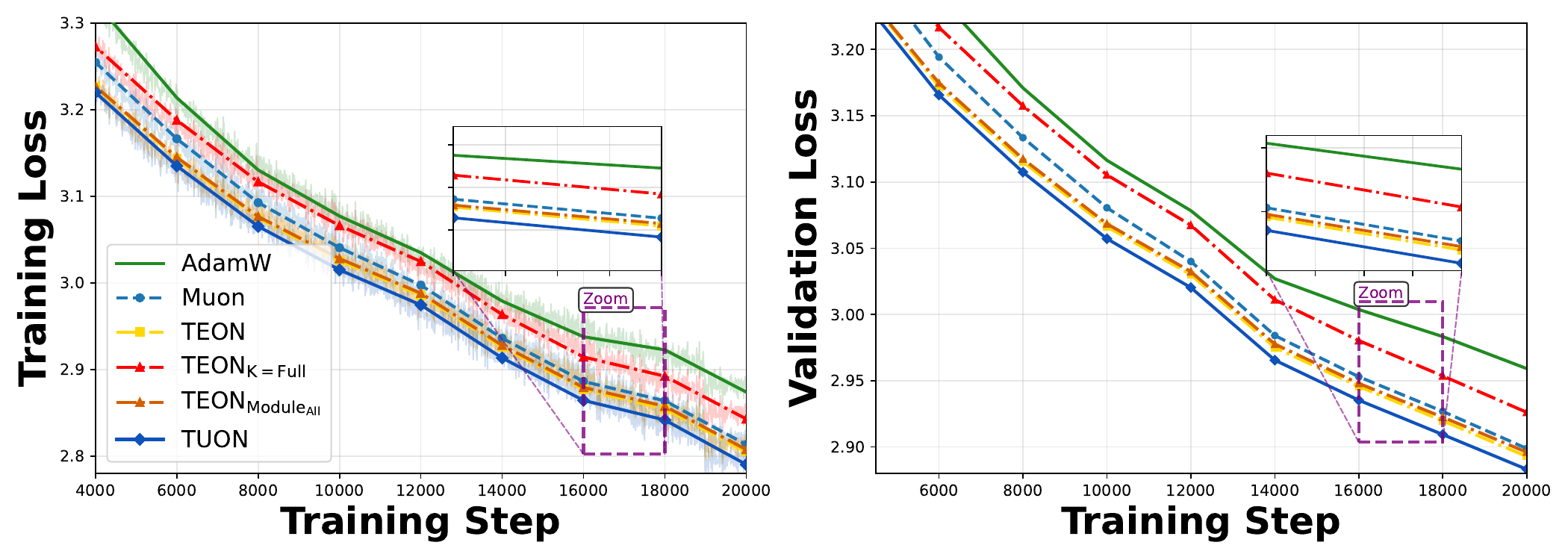}
		\caption{Qwen 0.6B .}
		\label{Performance on Qwen 06B}
	\end{subfigure}
	\begin{subfigure}[b]{0.45\textwidth}
		\centering
		\includegraphics[width=\textwidth]{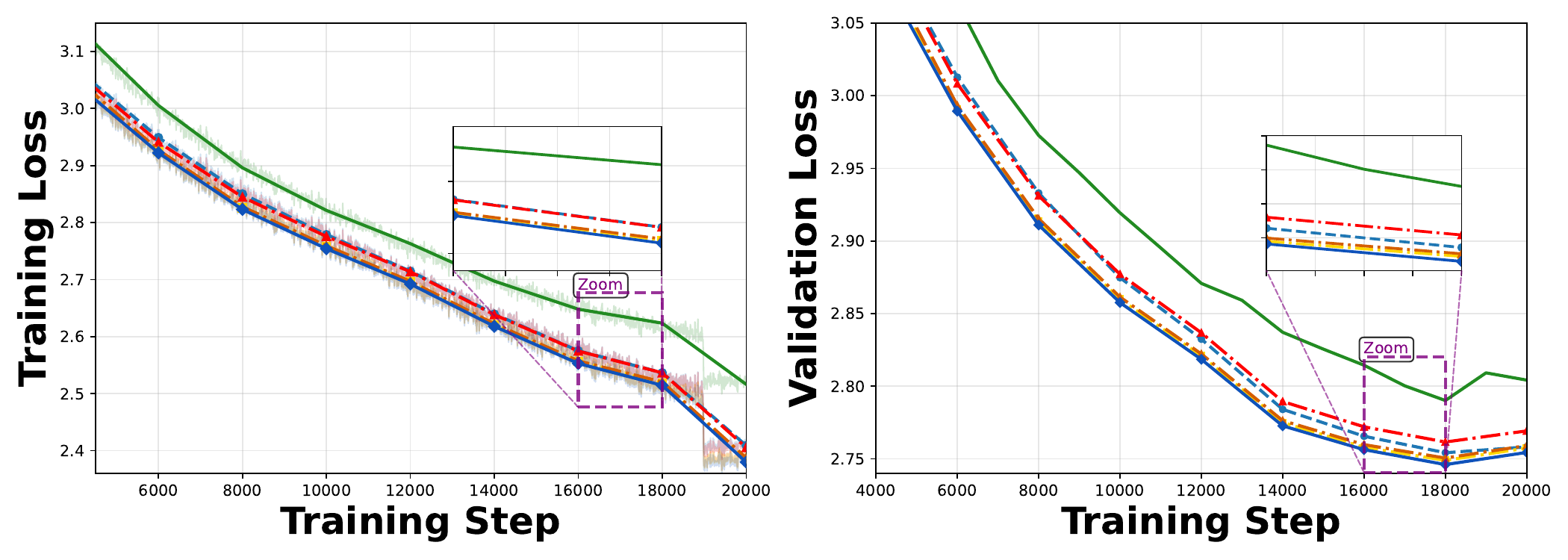}
		\caption{Qwen 4B.}
		\label{Performance on Qwen 4B}
	\end{subfigure}
	\caption{Training curves of Qwen models. }
	\label{fig:training-curves}
\end{figure}

\subsection{Results on LLM pretraining}
\label{Results on LLM pretraining}

\begin{table*}[!th]
	\centering
	\caption{Validation loss on GPT-2 (124M, 350M; 50k iterations) and Qwen3 (0.6B, 4B; 20k iterations). \(\Delta = \text{Val} - \text{Val}_{\text{AdamW}}\); negative \(\Delta\) indicates improvement. Best per column in bold. Lower is better.}
	\label{tab:main-val}
	\small
	\setlength{\tabcolsep}{5pt}
	\begin{tabular}{@{}lcccccccc@{}}
		\toprule
		\multirow{2}{*}{Optimizer}
		& \multicolumn{2}{c}{\textbf{124M\(_{\text{GPT-2}}\)}} & \multicolumn{2}{c}{\textbf{350M\(_{\text{GPT-2}}\)}} & \multicolumn{2}{c}{\textbf{0.6B\(_{\text{Qwen3}}\)}} & \multicolumn{2}{c}{\textbf{4B\(_{\text{Qwen3}}\)}} \\
		\cmidrule(lr){2-3} \cmidrule(lr){4-5} \cmidrule(lr){6-7} \cmidrule(lr){8-9}
		& Val & \(\Delta\) & Val & \(\Delta\) & Val & \(\Delta\) & Val & \(\Delta\) \\
		\midrule
		AdamW & \textbf{2.9230} & -- & 2.7307 & -- & 2.9652 & -- & 2.8108 & -- \\
		Muon & 3.0488 & \(+0.1258\) & 2.7404 & \(+0.0097\) & 2.9056 & \(-0.0596\) & 2.7636 & \(-0.0472\) \\
		TEON & 2.9442 & \(+0.0212\) & 2.6935 & \(-0.0372\) & 2.8978 & \(-0.0674\) & 2.7604 & \(-0.0504\) \\
		\rowcolor{cyan!8}TUON & 2.9357 & \(\mathbf{+0.0127}\) & \textbf{2.6925} & \(\mathbf{-0.0382}\) & \textbf{2.8887} & \(\mathbf{-0.0765}\) & \textbf{2.7598} & \(\mathbf{-0.0510}\) \\
		\bottomrule
	\end{tabular}
\end{table*}

Table~\ref{tab:main-val} reports validation loss on GPT-2 (124M, 350M) and Qwen3 (0.6B, 4B). TUON achieves the lowest loss on three of the four settings and consistently outperforms Muon and TEON. 
Specifically, TUON reduces validation loss over AdamW by \(0.0382\), \(0.0765\), and \(0.0510\) on GPT-2 350M and the two Qwen3 scales, respectively. 
Moreover, TUON outperforms Muon on all four settings, with the largest gap at 124M (\(0.1131\)). This suggests that cross-layer structure is particularly beneficial. 
Furthermore, TUON improves over TEON on every model, consistent with our motivation and theory in \S\ref{subsec:muon-comparison}. 
Fig.~\ref{fig:training-curves} presents the training and validation curves for Qwen3 0.6B and 4B. TUON outperforms AdamW, Muon, TEON, and the TEON variants that stack all layers (\(\text{TEON}_{K=\text{Full}}\)) or extend to QKVO and MLP (\(\text{TEON}_{\text{Module}_{\text{All}}}\)). 
Overall, TUON delivers consistent improvements across two architectures and four scales without per-scale tuning, supporting the claim that a transform-based tensor algebra offers a more effective geometry for cross-layer optimization.

\subsection{Ablation studies}
\label{sec:ablations} 
We conduct ablation studies to isolate the effect of each design choice in our method. Unless otherwise noted, we train all models for 20k iterations for computational reasons.
\paragraph{Choices of \(\boldsymbol{Q}\).} \label{Q-ablation}
The T-SVD-based algebra only requires \(\boldsymbol{Q}\) to be invertible and does not specify which transform suits cross-layer optimization. Our analysis in \S~\ref{sec:method} suggests that \(\boldsymbol{Q}\) should be guided by layer statistics, which we test here by comparing fixed transforms (DCT, Haar, DFT) against data-dependent DecorQ variants (static and residual); details are in Appendix~\ref{Selection-of-Q} and~\ref{app:q-choices}.

In Table~\ref{tab:q-fullL-appendix}, DecorQ achieves the best results, improving over DCT by \(0.0502\), while the other four choices perform comparably. It reflects that layer statistics are empirically superior. Moreover, static DecorQ reaches a competitive \(2.7658\) but remains slightly worse than the adaptive version, indicating that refreshing \(\boldsymbol{Q}\) during training provides a consistent benefit. 
Overall, the ablation confirms that the layer mode should be treated as a structured dimension rather than an arbitrary algebraic index.


\begin{table}[!t]
	\centering
	\caption{Effect of stacking size \(K\). \(\uparrow\)/\(\downarrow\) denote \(K=\text{full}\) is worse/better than \(K=2\). 20k iterations.}
	\label{tab:stacking_k}
	\small
	\setlength{\tabcolsep}{4pt}
	\begin{tabular}{lccccccccc}
		\toprule
		\multirow{2}{*}{Model} &\multirow{2}{*}{Muon}
		& \multicolumn{2}{c}{TEON\(_{\text{mode-1}}\)}
		& \multicolumn{2}{c}{TEON\(_{\text{mode-2}}\)}
		& \multicolumn{2}{c}{TUON\(_{\text{DCT}}\)}
		& \multicolumn{2}{c}{TUON\(_{\text{DecorQ}}\)} \\
		\cmidrule(lr){3-4}\cmidrule(lr){5-6}\cmidrule(lr){7-8}\cmidrule(lr){9-10}
		&  & K=2 & K=full & K=2 & K=full & K=2 & K=full & K=2 & K=full \\
		\midrule
		\textbf{124M\(_{\text{GPT-2}}\)} & 3.0583
		& 3.0255 & 3.0513\(\left( \uparrow\right)\)
		& 3.0332 & 3.0789\(\left( \uparrow\right)\)
		& 3.0309 &    3.0295\textbf{\(\left( \downarrow\right)\)}
		&  3.0299  &     \multicolumn{1}{>{\columncolor{cyan!8}}c}{\textbf{3.0233}\textbf{\(\left( \downarrow\right)\)} }\\
		\textbf{350M\(_{\text{GPT-2}}\)} &2.8427
		& 2.8427 & 2.8617\(\left( \uparrow\right)\)
		& 2.8643 & 2.8762\(\left( \uparrow\right)\)
		& 2.8150 & 2.8218\(\left( \uparrow\right)\)
		& 2.8183 & \multicolumn{1}{>{\columncolor{cyan!8}}c}{\textbf{2.8143}\textbf{\(\left( \downarrow\right)\)} }\\
		\textbf{0.6B\(_{\text{Qwen3}}\)}& 2.9056
		& 2.8994 & 2.9323\(\left( \uparrow\right)\)
		& 2.8978 & 2.9583\(\left( \uparrow\right)\)
		& 2.9203 & 2.9210\(\left( \uparrow\right)\)
		& 2.8938 & \multicolumn{1}{>{\columncolor{cyan!8}}c}{ \textbf{2.8887}\textbf{\(\left( \downarrow\right)\)}} \\
		\textbf{4B\(_{\text{Qwen3}}\)} &2.7636
		& 2.7627 & 2.7649\(\left( \uparrow\right)\)
		& 2.7604 & 2.7640\(\left( \uparrow\right)\)
		& 2.7606 & 2.8100\(\left( \uparrow\right)\)
		& 2.7629 & \multicolumn{1}{>{\columncolor{cyan!8}}c}{\textbf{2.7598} \textbf{\(\left( \downarrow\right)\)} }\\
		\bottomrule
	\end{tabular}
\end{table}

\paragraph{Stacking size \(K\).}
A larger stacking size \(K\) exposes more cross-layer structure, and the theoretical analysis in  \S~\ref{subsec:muon-comparison}
and in~\citep{zhang2026teon} predicts a larger potential gain as \(K\) grows. 
Table~\ref{tab:stacking_k} tests whether this potential is realized in practice. TEON performs best at \(K=2\) and consistently degrades at \(K=\text{full}\) for both mode-1 and mode-2 unfolding, exposing a gap between theory and practice. Although larger \(K\) promises more gain, unfolding methods lose performance as \(K\) grows. Two causes are likely. Unfolding conflates structural information, and aligning top singular vectors across many layers becomes increasingly difficult.

TUON with DecorQ is the only method that improves from \(K=2\) to \(K=\text{full}\) on all four models: \(3.0299 \) to \( 3.0233\) (GPT-2 124M), \(2.8183 \) to \( 2.8143\) (GPT-2 350M), \(2.8938 \) to \( 2.8887\) (Qwen3 0.6B), and \(2.7629 \) to \( 2.7598\) (Qwen3 4B), confirming its ability to exploit more layers. 
We also notice that TUON with DCT  degrades on three of the four models, isolating the effect of \(\boldsymbol{Q}\), i.e., it is not larger \(K\) alone that helps, but the ability to adapt the shared basis to the grouped layers. 
Overall, although the gain remains modest and falls short of the \(\sqrt{K}\) upper bound by theory, TUON with DecorQ is the first that makes stacking more layers beneficial rather than harmful, a regime that prior methods cannot reach.  

\paragraph{Unfolding mode.}
TEON requires an additional hyperparameter, i.e., unfold along mode-1 or mode-2. It reports that mode-1 works better in practice ~\citep{zhang2026teon}. Our experiments reveal that the better mode of TEON is \emph{model-dependent}~(Table \ref{tab:stacking_k}). Specifically, GPT-2 favors mode-1, while Qwen3 favors mode-2. Moreover, choosing the wrong mode costs \(0.0216\) on GPT-2 350M. The gap is non-negligible, meaning TEON requires a per-model mode choice, and neither mode provides a universally correct alignment. \textbf{TUON avoids this hyperparameter by transform-based tensor algebra.} In all main experiments, we report the stronger of the two TEON modes, ensuring that the comparison is against the most favorable TEON configuration. Even under this conservative protocol, TUON outperforms the best TEON on every model in Table~\ref{tab:main-val}.

\paragraph{Layer type.}
Our method applies uniformly to all layer types. We verify this by comparing \(\text{Module}_{\text{QKV}}\) (attention projections only) with \(\text{Module}_{\text{All}}\) (QKVO and MLP). 
Table~\ref{tab:layer-type-appendix} reports the results.  Under mode-1, TEON degrades on all four models, and under mode-2 the degradation is even larger, e.g., from \(3.0332\) to \(3.1027\) on GPT-2 124M. This confirms that the matricization-based tensor algebra struggles to maintain alignment across heterogeneous layer types under either unfolding mode. 
TUON behaves differently, extending from \(\text{Module}_{\text{QKV}}\) to \(\text{Module}_{\text{All}}\) yields comparable or better losses on GPT-2 124M (\(3.0225 \) to \( 3.0233\)), GPT-2 350M (\(2.8291 \) to \( 2.8143\)), Qwen3 0.6B (\(2.8897 \) to \( 2.8887\)), and Qwen3 4B (\(2.7606 \) to \( 2.7598\)). 
Moreover, in the \(\text{Module}_{\text{All}}\) setting, it outperforms the best TEON by \(0.0298\), \(0.0474\), \(0.0131\), and \(0.0041\), respectively. These results show that TUON's benefit extends beyond attention to MLP and output projections, where prior tensor-aware optimizers fail to transfer.

\paragraph{Polarization methods.}
Our optimizer relies on a polar map, approximated in practice by Newton--Schulz~(NS) iterations. We test the robustness of TUON to the polarization method by comparing NS (\(J=5\)) with PolarExpress~\citep{amsel2026polar} on Qwen3-4B. 
In Table~\ref{tab:polarization}, the two methods perform comparably at both stacking sizes. Specifically, at \(K=2\), NS reaches \(2.7629\) and PolarExpress \(2.7638\); at full stacking, PolarExpress is slightly better (\(2.7569\) vs. \(2.7598\)). Both differences are an order of magnitude smaller than the gain from the \(\boldsymbol{Q}\) mechanism (\(0.0502\)), confirming the robustness of TUON.

\subsection{Training step time}
\label{sec:stability-cost}
We compare the per-iteration step time with identical hardware, batch size, sequence length, and parallelism configuration. Table~\ref{tab:step_time} reports the average time per iteration (ms), excluding the first few iterations to remove compilation and initialization overhead. 
Relative to vector-aware AdamW, TUON is slower by \(6.1\%\), \(8.3\%\), \(3.5\%\), and \(12.5\%\) on the four models per iteration, respectively. 
Compared with matrix-aware and tensor-aware optimizer, TUON is on par with Muon and TEON across all four models, within \(0.2\%\)--\(1.1\%\) of TEON and \(1.7\%\)--\(6.1\%\) of Muon. 
These results show that TUON introduces only marginal overhead and scales gracefully to larger models, making the per-iteration cost a reasonable price for the optimization benefits reported above.


\begin{table}[t]
	\centering
	\caption{Time per iteration (ms) on four models. Lower values indicate faster training.}
	\label{tab:step_time}
	\small
	\begin{tabular}{lcccc}
		\toprule
		\multirow{2}{*}{Model} & \multicolumn{4}{c}{Time per Iteration (ms)} \\
		\cmidrule(lr){2-5}
		& \textbf{124M\(_{\text{GPT-2}}\)} & \textbf{350M\(_{\text{GPT-2}}\)} & \textbf{0.6B\(_{\text{Qwen3}}\)} & \textbf{4B\(_{\text{Qwen3}}\)} \\
		\midrule
		AdamW & \textbf{152.74} & \textbf{379.95} & \textbf{3265.5} & \textbf{5466.9} \\
		Muon  & 160.88 & 404.81 & 3324.3 & 5794.7 \\
		TEON  & 161.81 & 408.30 & 3385.7 & 6081.5 \\
		TUON  & 162.12 & 411.64 & 3378.3 & 6150.7 \\
		\bottomrule
	\end{tabular}
\end{table}

\section{Conclusion}
\label{conclusion}

We introduce \textbf{TUON}, a tensor-aware optimizer built on a transform-based tensor algebra grounded in the Tensor-SVD. For third-order tensors in Transformers, TUON applies a unitary transform \(\boldsymbol{Q}\) along the layer mode and polarizes each transformed slice independently before inverting the transform, a mix-then-separate design that exploits cross-layer structure while preserving per-layer diversity and avoiding the information loss of global unfolding. Empirically, TUON improves consistently over Muon and TEON on GPT-124M/350M and Qwen-0.6B/4B, with ablations confirming its transformation mechanism, cross-layer stacking beyond two layers, adaptability beyond QKV, and robustness to polarization, at only marginal overhead. 
For future work, although DecorQ realizes the benefits of larger stacking sizes, these gains remain bounded and below the theoretical maximum. Further extending TUON's stacking capability is therefore a promising direction. Two additional directions are extending TUON to other parameter groups and heterogeneous architectures, and co-designing tensor-efficient kernels that exploit the layer-mode transform and slice-wise polarization to reduce memory traffic and improve hardware efficiency.

\newpage

\subsection*{AI use statement}


In this work, we used generative AI tools solely for suggested-disclosure tasks to aid writing. These tools are used to improve readability, suggesting alternative phrasing, and refining the overall presentation of our work. 
We have not used generative AI tools for any required-disclosure tasks as defined by the ICLR policy, such as generating synthetic data, developing theoretical models or conceptual frameworks or mathematical claims, writing proofs, designing experiments or methodology, implementing methods, or interpreting results. 
We have reviewed all AI-assisted work. All AI-suggested text edits were critically reviewed, revised as needed, and checked for accuracy and consistency by the authors. 
We take responsibility for the final content of this work, including text, claims or artifacts produced with the aid of generative AI.

See the ICLR 2027 AI Policy for Authors for more details. This statement should not be more than 1 page.

\subsection*{Reproducibility statement}


We have taken the following steps to ensure the reproducibility of our work.

Our implementation strictly follows the PyTorch official API, 
\begin{verbatim}
	torch.optim
\end{verbatim}
. The complete production-ready code is publicly available.

\textbf{For all results.}
The complete source code is provided as anonymous material and is also available at \url{https://anonymous.4open.science/r/tuon-802A/README.md}. 

All computational resources, training details, configurations, baselines and TUON implementation specifics, model architectures and hyper-parameters are provided in Appendices \ref{Code and Reproducibility}, \ref{sec:exp-setup}  and \ref{app:hyperparameters}. 

\textbf{For theoretical results.}
Complete proofs of all claims are provided in Appendix~\ref{app:proofs1}. 


We believe the above materials are sufficient for reproducing all main results reported in this paper. Any remaining implementation details not covered here are available upon reasonable request to the authors.

\newpage

\bibliography{iclr2027_conference}
\bibliographystyle{iclr2027_conference}

\newpage

\appendix
\addtocontents{toc}{\protect\setcounter{tocdepth}{-1}}  
\tableofcontents
\addtocontents{toc}{\protect\setcounter{tocdepth}{2}}   

\section{Unified notation}\label{unified-notation}
For one architecture-compatible group of $K$ matrices, we use
\begin{center}
	\begin{tabular}{ll}
		\toprule
		Symbol & Meaning \\
		\midrule
		\(\boldsymbol{W}_t^{(k)},\boldsymbol{G}_t^{(k)},\boldsymbol{M}_t^{(k)}\in\R^{m\times n}\) & parameter, stochastic gradient, and momentum of layer slice \(k\) \\
		\(\boldsymbol{\mathcal{W}}_t,\boldsymbol{\mathcal{G}}_t,\boldsymbol{\mathcal{M}}_t\in\R^{m\times n\times K}\) & their stacked tensors \\
		\(\boldsymbol{Q}\) or \(\boldsymbol{Q}_t\) & unitary layer-mode transform \\
		\(\widehat{\boldsymbol{\mathcal{M}}}_t=\boldsymbol{\mathcal{M}}_t\times_3\boldsymbol{Q}\) & momentum in transformed layer coordinates \\
		\(\mathsf{M}_t\) & mode-3 unfolding of \(\boldsymbol{\mathcal{M}}_t\) \\
		\(\mathsf{C}_t=\mathsf{M}_t\mathsf{M}_t^{\mathsf T}\) & layer Gram matrix \\
		\(\boldsymbol{\mathcal{P}}_{\boldsymbol{Q}}(\boldsymbol{\mathcal{M}})\) or \(\boldsymbol{\mathcal{P}}\) & exact transform--slice-polar--inverse-transform map \\
		\(\operatorname{tpolar}_{\boldsymbol{Q}}(\cdot)\) & generalized tensor polar factor \\
		\(\boldsymbol{\mathcal{U}},\boldsymbol{\mathcal{S}},\boldsymbol{\mathcal{V}}\) & T-SVD factors of \(\boldsymbol{\mathcal{M}}\) \\
		\(\boldsymbol{\mathcal{J}}_{\mathrm{supp}}\) & support indicator tensor \\
		\(\gamma_t\) & shared stack-scale factor \\
		\(\epsilon\) & small constant for numerical stability \\
		\(J\) & number of Newton--Schulz iterations \\
		\(I\) & refresh interval for \(\boldsymbol{Q}_t\) \\
		\(\eta_t,\beta,\lambda\) & learning rate, momentum coefficient, and decoupled weight decay \\
		\(\eta\) & fixed step size used in the theory \\
		\(T\) & number of iterations in the convergence analysis \\
		\(F_\star\) & lower bound of the objective \\
		\(\rho_t\) & trust-region radius in the theory \\
		\(\|\cdot\|_{\boldsymbol{Q},\mathrm{Ts}}\) & \(Q\)-block spectral norm \\
		\(\|\cdot\|_{\boldsymbol{Q},\mathrm{Tn}}\) & \(Q\)-block nuclear norm (dual of \(\|\cdot\|_{\boldsymbol{Q},\mathrm{Ts}}\)) \\
		\(\|\cdot\|_{\mathrm{Muon}},\|\cdot\|_{\mathrm{TEON}}\) & Muon and TEON spectral norms \\
		\(L_Q\) & smoothness constant under the \(Q\)-block norm pair \\
		\(\alpha_Q,\beta_Q\) & transform spread constants \\
		\(\mathscr U(\cdot)\) & mode-1 unfolding operator \\
		\(R_Q=Q^{\mathsf T}\otimes I_n\) & unfolding transformation matrix \\
		\bottomrule
	\end{tabular}
\end{center}
This convention is used in the method, theory, experiments, and appendices.

\section{Preliminaries}
\label{app:preliminaries}

This appendix collects the elementary facts used in Section~\ref{sec:theory}.  They are stated explicitly so that the proofs can be followed using only basic linear algebra and first-order optimization.

\subsection{Tensor notation and unitary mode-$3$ transforms}
\label{app:tensor-transform-basics}

For $\boldsymbol{\mathcal{X}}\in\mathbb{C}^{m\times n\times K}$, its $k$th frontal slice is $\boldsymbol{X}_k=\boldsymbol{\mathcal{X}}(:,:,k)$.  A mode-$3$ transform by $\boldsymbol{Q}\in\mathbb{C}^{K\times K}$ is
\begin{equation}
	\widehat{\boldsymbol{\mathcal{X}}}=\boldsymbol{\mathcal{X}}\times_3\boldsymbol{Q},
	\qquad
	\widehat{\boldsymbol{X}}_j=\sum_{k=1}^K\boldsymbol{Q}_{jk}\boldsymbol{X}_k.
\end{equation}
The Frobenius inner product and norm are
\begin{equation}
	\langle\boldsymbol{\mathcal{X}},\boldsymbol{\mathcal{Y}}\rangle_{\mathrm F}
	:=\sum_{k=1}^K\operatorname{tr}(\boldsymbol{X}_k^{\mathsf H}\boldsymbol{Y}_k),
	\qquad
	\|\boldsymbol{\mathcal{X}}\|_{\mathrm F}^2
	:=\langle\boldsymbol{\mathcal{X}},\boldsymbol{\mathcal{X}}\rangle_{\mathrm F}.
\end{equation}

\begin{lemma}[Parseval identities for a unitary layer transform]
	\label{lem:unitary-parseval}
	If $\boldsymbol{Q}^{\mathsf H}\boldsymbol{Q}=\boldsymbol{I}$, then
	\begin{equation}
		\langle\boldsymbol{\mathcal{X}}\times_3\boldsymbol{Q},\boldsymbol{\mathcal{Y}}\times_3\boldsymbol{Q}\rangle_{\mathrm F}
		=\langle\boldsymbol{\mathcal{X}},\boldsymbol{\mathcal{Y}}\rangle_{\mathrm F},
		\qquad
		\|\boldsymbol{\mathcal{X}}\times_3\boldsymbol{Q}\|_{\mathrm F}=\|\boldsymbol{\mathcal{X}}\|_{\mathrm F}.
		\label{eq:appendix-parseval}
	\end{equation}
\end{lemma}

\begin{proof}
	We prove the inner-product preservation by expanding the Frobenius inner product entry-wise. The norm preservation then follows as a special case.
	
	First, we express the tensor in terms of its tubes. 
	For each position $(a,b)$ in the $m \times n$ matrix, define the tube vector
	\begin{equation}
		\boldsymbol{x}_{ab} := 
		\begin{bmatrix}
			\boldsymbol{X}_1(a,b) \\
			\boldsymbol{X}_2(a,b) \\
			\vdots \\
			\boldsymbol{X}_K(a,b)
		\end{bmatrix}
		\in \mathbb{C}^K,
		\qquad
		\boldsymbol{y}_{ab} := 
		\begin{bmatrix}
			\boldsymbol{Y}_1(a,b) \\
			\boldsymbol{Y}_2(a,b) \\
			\vdots \\
			\boldsymbol{Y}_K(a,b)
		\end{bmatrix}
		\in \mathbb{C}^K.
	\end{equation}
	These vectors collect the values along the third mode (the layer dimension) for a fixed matrix entry $(a,b)$. The Frobenius inner product of $\boldsymbol{\mathcal{X}}$ and $\boldsymbol{\mathcal{Y}}$ can then be written as
	\begin{equation}
		\langle\boldsymbol{\mathcal{X}},\boldsymbol{\mathcal{Y}}\rangle_{\mathrm F}
		=
		\sum_{a=1}^m \sum_{b=1}^n
		\boldsymbol{x}_{ab}^{\mathsf H} \boldsymbol{y}_{ab}.
		\label{eq:proof-parseval-tube-sum}
	\end{equation}
	Indeed, $\boldsymbol{x}_{ab}^{\mathsf H} \boldsymbol{y}_{ab} = \sum_{k=1}^K \overline{\boldsymbol{X}_k(a,b)} \boldsymbol{Y}_k(a,b)$, and summing over all $(a,b)$ recovers the standard Frobenius inner product.
	
	Second, we express the transformed tensor in terms of the same tubes. 
	Applying $\boldsymbol{Q}$ along the third mode transforms each tube vector independently:
	\begin{equation}
		\widehat{\boldsymbol{X}}_j(a,b) = \sum_{k=1}^K \boldsymbol{Q}_{jk} \boldsymbol{X}_k(a,b),
	\end{equation}
	which means that the transformed tube at position $(a,b)$ is
	\begin{equation}
		\widehat{\boldsymbol{x}}_{ab}
		:=
		\begin{bmatrix}
			\widehat{\boldsymbol{X}}_1(a,b) \\
			\widehat{\boldsymbol{X}}_2(a,b) \\
			\vdots \\
			\widehat{\boldsymbol{X}}_K(a,b)
		\end{bmatrix}
		=
		\boldsymbol{Q} \boldsymbol{x}_{ab}.
	\end{equation}
	Similarly, $\widehat{\boldsymbol{y}}_{ab} = \boldsymbol{Q} \boldsymbol{y}_{ab}$.
	
	Then we compute the inner product of the transformed tensors. 
	Using the same tube-sum identity as in Step 1, the inner product of the transformed tensors is
	\begin{equation}
		\langle\boldsymbol{\mathcal{X}} \times_3 \boldsymbol{Q},
		\boldsymbol{\mathcal{Y}} \times_3 \boldsymbol{Q}\rangle_{\mathrm F}
		=
		\sum_{a=1}^m \sum_{b=1}^n
		\widehat{\boldsymbol{x}}_{ab}^{\mathsf H} \widehat{\boldsymbol{y}}_{ab}.
	\end{equation}
	Substituting $\widehat{\boldsymbol{x}}_{ab} = \boldsymbol{Q} \boldsymbol{x}_{ab}$ and $\widehat{\boldsymbol{y}}_{ab} = \boldsymbol{Q} \boldsymbol{y}_{ab}$ gives
	\begin{equation}
		\widehat{\boldsymbol{x}}_{ab}^{\mathsf H} \widehat{\boldsymbol{y}}_{ab}
		=
		(\boldsymbol{Q} \boldsymbol{x}_{ab})^{\mathsf H} (\boldsymbol{Q} \boldsymbol{y}_{ab})
		=
		\boldsymbol{x}_{ab}^{\mathsf H} \boldsymbol{Q}^{\mathsf H} \boldsymbol{Q} \boldsymbol{y}_{ab}.
	\end{equation}
	Since $\boldsymbol{Q}$ is unitary, $\boldsymbol{Q}^{\mathsf H} \boldsymbol{Q} = \boldsymbol{I}$, so
	\begin{equation}
		\widehat{\boldsymbol{x}}_{ab}^{\mathsf H} \widehat{\boldsymbol{y}}_{ab}
		=
		\boldsymbol{x}_{ab}^{\mathsf H} \boldsymbol{y}_{ab}.
	\end{equation}
	
	At last, we sum over all entries. 
	Summing the identity $\widehat{\boldsymbol{x}}_{ab}^{\mathsf H} \widehat{\boldsymbol{y}}_{ab} = \boldsymbol{x}_{ab}^{\mathsf H} \boldsymbol{y}_{ab}$ over all $m \times n$ entries $(a,b)$ yields
	\begin{equation}
		\langle\boldsymbol{\mathcal{X}} \times_3 \boldsymbol{Q},
		\boldsymbol{\mathcal{Y}} \times_3 \boldsymbol{Q}\rangle_{\mathrm F}
		=
		\langle\boldsymbol{\mathcal{X}},\boldsymbol{\mathcal{Y}}\rangle_{\mathrm F}.
	\end{equation}
	This proves the inner-product preservation.
	
	Moreover, we can have the norm preservation as a corollary. 
	Setting $\boldsymbol{\mathcal{Y}} = \boldsymbol{\mathcal{X}}$ in the inner-product preservation gives
	\begin{equation}
		\|\boldsymbol{\mathcal{X}} \times_3 \boldsymbol{Q}\|_{\mathrm F}^2
		=
		\|\boldsymbol{\mathcal{X}}\|_{\mathrm F}^2.
	\end{equation}
	Taking square roots (both sides are nonnegative) yields the norm-preservation identity.
\end{proof}


\subsection{SVD, matrix norms, and the polar factor}
\label{app:matrix-tools}

Let $\boldsymbol{A}\in\mathbb{C}^{m\times n}$ have compact singular value decomposition
\begin{equation}
	\boldsymbol{A}=\boldsymbol{U}\boldsymbol{\Sigma}\boldsymbol{V}^{\mathsf H},
	\qquad
	\boldsymbol{\Sigma}=\operatorname{diag}(\sigma_1,\ldots,\sigma_r),
	\qquad r=\operatorname{rank}(\boldsymbol{A}).
\end{equation}
The spectral, nuclear, and Frobenius norms are
\begin{equation}
	\|\boldsymbol{A}\|_{\mathrm{op}}=\sigma_1,
	\qquad
	\|\boldsymbol{A}\|_*=\sum_{i=1}^r\sigma_i,
	\qquad
	\|\boldsymbol{A}\|_{\mathrm F}^2=\sum_{i=1}^r\sigma_i^2.
\end{equation}
The canonical polar factor of a nonzero matrix $\boldsymbol{A}$ is defined using its compact singular value decomposition $\boldsymbol{A} = \boldsymbol{U} \boldsymbol{\Sigma} \boldsymbol{V}^{\mathsf H}$. Specifically, we set
\begin{equation}
	\operatorname{polar}(\boldsymbol{A}) := \boldsymbol{U} \boldsymbol{V}^{\mathsf H}.
\end{equation}
This factor acts as the unitary part of the polar decomposition, i.e.  it retains the singular vectors of $\boldsymbol{A}$ while replacing all nonzero singular values by one. For a zero matrix, we define $\operatorname{polar}(\boldsymbol{0}) := \boldsymbol{0}$.

\begin{lemma}[Elementary polar identities]
	\label{lem:polar-identities}
	For $\boldsymbol{A}\neq0$ with rank $r$,
	\begin{equation}
		\|\operatorname{polar}(\boldsymbol{A})\|_{\mathrm{op}}=1,
		\qquad
		\|\operatorname{polar}(\boldsymbol{A})\|_{\mathrm F}^2=r,
		\qquad
		\langle \boldsymbol{A},\operatorname{polar}(\boldsymbol{A})\rangle_{\mathrm F}=\|\boldsymbol{A}\|_*.
		\label{eq:polar-identities}
	\end{equation}
\end{lemma}

\begin{proof}
	The nonzero singular values of $\boldsymbol{U}\boldsymbol{V}^{\mathsf H}$ are all one.  Therefore its largest singular value is one and its squared Frobenius norm is the sum of $r$ ones.  For the inner product,
	\begin{align}
		\langle \boldsymbol{A},\boldsymbol{U}\boldsymbol{V}^{\mathsf H}\rangle_{\mathrm F}
		&=\operatorname{tr}\!\left(\boldsymbol{A}^{\mathsf H}\boldsymbol{U}\boldsymbol{V}^{\mathsf H}\right)\\
		&=\operatorname{tr}\!\left(\boldsymbol{V}\boldsymbol{\Sigma}\boldsymbol{U}^{\mathsf H}\boldsymbol{U}\boldsymbol{V}^{\mathsf H}\right)\\
		&=\operatorname{tr}(\boldsymbol{V}\boldsymbol{\Sigma}\boldsymbol{V}^{\mathsf H})
		=\operatorname{tr}(\boldsymbol{\Sigma})
		=\sum_{i=1}^r\sigma_i
		=\|\boldsymbol{A}\|_*.
	\end{align}
	The quantity is real and nonnegative, so taking its real part does not change it.
\end{proof}

\begin{lemma}[Operator--nuclear H\"older inequality]
	\label{lem:matrix-holder}
	For matrices of the same shape,
	\begin{equation}
		|\langle \boldsymbol{A},\boldsymbol{B}\rangle_{\mathrm F}|
		\le\|\boldsymbol{A}\|_*\|\boldsymbol{B}\|_{\mathrm{op}}.
		\label{eq:matrix-holder}
	\end{equation}
	Equality in the maximization over $\|\boldsymbol{B}\|_{\mathrm{op}}\le1$ is attained by $\boldsymbol{B}=\operatorname{polar}(\boldsymbol{A})$.
\end{lemma}

\begin{proof}
	Write $\boldsymbol{A}=\boldsymbol{U}\boldsymbol{\Sigma}\boldsymbol{V}^{\mathsf H}$.  Cyclicity of trace gives
	\begin{equation}
		\langle \boldsymbol{A},\boldsymbol{B}\rangle_{\mathrm F}
		=\operatorname{tr}(\boldsymbol{A}^{\mathsf H}\boldsymbol{B})
		=\operatorname{tr}(\boldsymbol{\Sigma}\boldsymbol{U}^{\mathsf H}\boldsymbol{B}\boldsymbol{V}).
	\end{equation}
	Let $\boldsymbol{C}=\boldsymbol{U}^{\mathsf H}\boldsymbol{B}\boldsymbol{V}$.  Multiplication by matrices with orthonormal columns cannot increase the spectral norm, so $\|\boldsymbol{C}\|_{\mathrm{op}}\le\|\boldsymbol{B}\|_{\mathrm{op}}$.  Since every diagonal entry satisfies $|\boldsymbol{C}_{ii}|\le\|\boldsymbol{C}\|_{\mathrm{op}}$,
	\begin{align}
		|\operatorname{tr}(\boldsymbol{\Sigma}\boldsymbol{C})|
		&=\left|\sum_i\sigma_i\boldsymbol{C}_{ii}\right|
		\le\sum_i\sigma_i|\boldsymbol{C}_{ii}|\\
		&\le\left(\sum_i\sigma_i\right)\|\boldsymbol{C}\|_{\mathrm{op}}
		\le\|\boldsymbol{A}\|_*\|\boldsymbol{B}\|_{\mathrm{op}}.
	\end{align}
	If $\|\boldsymbol{B}\|_{\mathrm{op}}\le1$, Lemma~\ref{lem:polar-identities} shows that $\boldsymbol{B}=\boldsymbol{U}\boldsymbol{V}^{\mathsf H}$ is feasible and gives inner product $\|\boldsymbol{A}\|_*$.  Hence
	\begin{equation}
		\sup_{\|\boldsymbol{B}\|_{\mathrm{op}}\le1}\langle \boldsymbol{A},\boldsymbol{B}\rangle_{\mathrm F}=\|\boldsymbol{A}\|_*.
	\end{equation}
\end{proof}

\begin{lemma}[Rank-based norm bounds]
	\label{lem:rank-norm-bounds}
	If $\boldsymbol{A}\in\mathbb{C}^{m\times n}$ and $r_0:=\min(m,n)$, then
	\begin{equation}
		\|\boldsymbol{A}\|_*\le r_0\|\boldsymbol{A}\|_{\mathrm{op}},
		\qquad
		\|\boldsymbol{A}\|_{\mathrm F}^2\le r_0\|\boldsymbol{A}\|_{\mathrm{op}}^2.
		\label{eq:rank-norm-bounds}
	\end{equation}
\end{lemma}

\begin{proof}
	There are at most $r_0$ nonzero singular values and each is at most $\sigma_1=\|\boldsymbol{A}\|_{\mathrm{op}}$.  Therefore
	\begin{equation}
		\sum_i\sigma_i\le r_0\sigma_1,
		\qquad
		\sum_i\sigma_i^2\le r_0\sigma_1^2.
	\end{equation}
\end{proof}

\subsection{Dual norms, norm ordering, and linear minimization oracles (LMO)}
\label{app:dual-lmo-basics}

We begin by recalling the fundamental duality between measuring the size of a direction and measuring the size of a gradient. Given a norm $\|\cdot\|$ that defines the geometry of our trust region, its dual norm is defined as
\begin{equation}
	\|\boldsymbol{g}\|_* := \sup_{\|\boldsymbol{d}\| \le 1} \langle \boldsymbol{g}, \boldsymbol{d} \rangle.
	\label{eq:dual-norm-def}
\end{equation}
This quantity measures the maximum possible alignment between a fixed vector $\boldsymbol{g}$ (typically a gradient) and any direction $\boldsymbol{d}$ lying within the unit ball of the primal norm. Intuitively, the dual norm tells us how ``large'' the gradient appears when measured against the geometry induced by $\|\cdot\|$.

Given a trust-region radius $\rho > 0$, the linear minimization oracle (LMO) seeks the direction within the scaled ball $\|\boldsymbol{d}\| \le \rho$ that minimizes the first-order approximation of the objective:
\begin{equation}
	\argmin_{\|\boldsymbol{d}\| \le \rho} \langle \boldsymbol{g}, \boldsymbol{d} \rangle.
	\label{eq:lmo-def}
\end{equation}
Because the objective $\langle \boldsymbol{g}, \boldsymbol{d} \rangle$ is linear in $\boldsymbol{d}$, the minimum over a symmetric convex ball is always attained on the boundary, pointing in the direction opposite to $\boldsymbol{g}$. More precisely, by the definition of the dual norm in ~\eqref{eq:dual-norm-def}, the optimal value is
\begin{equation}
	\min_{\|\boldsymbol{d}\| \le \rho} \langle \boldsymbol{g}, \boldsymbol{d} \rangle = -\rho \|\boldsymbol{g}\|_*.
	\label{eq:lmo-optimal-value}
\end{equation}
The direction that achieves this minimum is called the \emph{normalized steepest-descent direction} with respect to the norm $\|\cdot\|$. It points in the direction of the negative dual certificate and has norm exactly $\rho$.

\begin{lemma}[Order reversal for dual norms]
	\label{lem:dual-order-reversal}
	Suppose two norms satisfy
	\begin{equation}
		a\|\boldsymbol{x}\|_A\le\|\boldsymbol{x}\|_B\le b\|\boldsymbol{x}\|_A
		\qquad\text{for all }\boldsymbol{x},
		\label{eq:norm-order}
	\end{equation}
	where $a,b>0$.  Then their duals satisfy
	\begin{equation}
		\frac1b\|\boldsymbol{g}\|_{A,*}\le\|\boldsymbol{g}\|_{B,*}\le\frac1a\|\boldsymbol{g}\|_{A,*}.
		\label{eq:dual-order}
	\end{equation}
\end{lemma}

\begin{proof}
	We prove the two inequalities in ~\eqref{eq:dual-order} separately.
	
	First, for the lower bound: \( \|\boldsymbol{g}\|_{B,*} \ge \frac{1}{b} \|\boldsymbol{g}\|_{A,*} \).
	
	The condition \( \|\boldsymbol{x}\|_B \le b \|\boldsymbol{x}\|_A \) tells us that the \(B\)-ball is larger than the \(A\)-ball scaled by \(b\). More precisely, any vector with \( \|\boldsymbol{x}\|_A \le 1/b \) must satisfy \( \|\boldsymbol{x}\|_B \le 1 \). Hence
	\begin{equation}
		\{\boldsymbol{x} : \|\boldsymbol{x}\|_A \le 1/b \} \;\subseteq\; \{\boldsymbol{x} : \|\boldsymbol{x}\|_B \le 1 \}.
	\end{equation}
	Taking the supremum of the linear functional \( \langle \boldsymbol{g}, \boldsymbol{x} \rangle \) over a larger set can only increase the result. Therefore
	\begin{equation}
		\|\boldsymbol{g}\|_{B,*}
		= \sup_{\|\boldsymbol{x}\|_B \le 1} \langle \boldsymbol{g}, \boldsymbol{x} \rangle
		\ge \sup_{\|\boldsymbol{x}\|_A \le 1/b} \langle \boldsymbol{g}, \boldsymbol{x} \rangle
		= \frac{1}{b} \sup_{\|\boldsymbol{y}\|_A \le 1} \langle \boldsymbol{g}, \boldsymbol{y} \rangle
		= \frac{1}{b} \|\boldsymbol{g}\|_{A,*}.
	\end{equation}
	The third equality uses the change of variables \( \boldsymbol{x} = \boldsymbol{y}/b \).
	
	Second, for the upper bound: \( \|\boldsymbol{g}\|_{B,*} \le \frac{1}{a} \|\boldsymbol{g}\|_{A,*} \).
	
	Similarly, the condition \( a \|\boldsymbol{x}\|_A \le \|\boldsymbol{x}\|_B \) implies
	\begin{equation}
		\{\boldsymbol{x} : \|\boldsymbol{x}\|_B \le 1 \} \;\subseteq\; \{\boldsymbol{x} : \|\boldsymbol{x}\|_A \le 1/a \}.
	\end{equation}
	Taking the supremum of \( \langle \boldsymbol{g}, \boldsymbol{x} \rangle \) over the smaller set (the \(B\)-ball) gives a value no larger than taking it over the larger set (the \(A\)-ball):
	\begin{equation}
		\|\boldsymbol{g}\|_{B,*}
		= \sup_{\|\boldsymbol{x}\|_B \le 1} \langle \boldsymbol{g}, \boldsymbol{x} \rangle
		\le \sup_{\|\boldsymbol{x}\|_A \le 1/a} \langle \boldsymbol{g}, \boldsymbol{x} \rangle
		= \frac{1}{a} \|\boldsymbol{g}\|_{A,*}.
	\end{equation}
	Again, the last equality follows from \( \boldsymbol{x} = \boldsymbol{y}/a \).
\end{proof}

\subsection{The $Q$-transform spectral quantity is a norm}
\label{app:q-block-is-norm}

\begin{lemma}[tensor spectral norm under $\boldsymbol{Q}$-transform axioms]
	\label{lem:q-block-norm}
	If $\boldsymbol{Q}$ is invertible, then
	\begin{equation}
		\|\boldsymbol{\mathcal{X}}\|_{\boldsymbol{Q},\mathrm{Ts}}
		:=\max_j\|(\boldsymbol{\mathcal{X}}\times_3\boldsymbol{Q})_j\|_{\mathrm{op}}
	\end{equation}
	is a norm.
\end{lemma}

\begin{proof}
	We verify the three norm axioms.
	
	\medskip\noindent\textbf{Positive definiteness.}\quad
	Every spectral norm is nonnegative, hence the maximum is nonnegative.  If $\|\boldsymbol{\mathcal{X}}\|_{\boldsymbol{Q},\mathrm{Ts}}=0$, then every transformed slice is zero, so $\boldsymbol{\mathcal{X}}\times_3\boldsymbol{Q}=0$.  Since $\boldsymbol{Q}$ is invertible, applying $\boldsymbol{Q}^{-1}$ gives $\boldsymbol{\mathcal{X}}=0$.  Conversely, $\boldsymbol{\mathcal{X}}=0$ clearly gives norm zero.
	
	\medskip\noindent\textbf{Absolute homogeneity.}\quad
	For any scalar $c$,
	\begin{align}
		\|c\boldsymbol{\mathcal{X}}\|_{\boldsymbol{Q},\mathrm{Ts}}
		&=\max_j\|c\widehat{\boldsymbol{X}}_j\|_{\mathrm{op}}
		=|c|\max_j\|\widehat{\boldsymbol{X}}_j\|_{\mathrm{op}}
		=|c|\|\boldsymbol{\mathcal{X}}\|_{\boldsymbol{Q},\mathrm{Ts}}.
	\end{align}
	
	\medskip\noindent\textbf{Triangle inequality.}\quad
	Linearity of the transform and the matrix operator-norm triangle inequality give
	\begin{align}
		\|\boldsymbol{\mathcal{X}}+\boldsymbol{\mathcal{Y}}\|_{\boldsymbol{Q},\mathrm{Ts}}
		&=\max_j\|\widehat{\boldsymbol{X}}_j+\widehat{\boldsymbol{Y}}_j\|_{\mathrm{op}}\\
		&\le\max_j\bigl(\|\widehat{\boldsymbol{X}}_j\|_{\mathrm{op}}+\|\widehat{\boldsymbol{Y}}_j\|_{\mathrm{op}}\bigr)\\
		&\le\max_j\|\widehat{\boldsymbol{X}}_j\|_{\mathrm{op}}+
		\max_j\|\widehat{\boldsymbol{Y}}_j\|_{\mathrm{op}}.
	\end{align}
\end{proof}

%
%
%

\subsection{T-product and T-SVD tensor algebra}\label{app:qproduct}

We begin by introducing the foundational algebraic framework for third-order tensors under invertible linear transforms. Throughout this section, we let $\boldsymbol{\mathcal{A}} \in \mathbb{R}^{I_1 \times I_2 \times I_3}$ denote a third-order tensor, and we fix an orthogonal linear transform $\boldsymbol{Q} \in \mathbb{R}^{I_3 \times I_3}$ satisfying $\boldsymbol{Q}\boldsymbol{Q}^{\top} = \boldsymbol{Q}^{\top}\boldsymbol{Q} = \boldsymbol{I}$. The application of $\boldsymbol{Q}$ to a tensor is understood mode-wise along the third dimension, i.e., $\boldsymbol{Q}(\boldsymbol{\mathcal{A}}) \coloneqq \boldsymbol{\mathcal{A}} \times_3 \boldsymbol{Q}$.

\begin{definition}[$t$-product under orthogonal transform \citep{kernfeld2015tensor}]
	\label{def:t_product}
	For tensors $\boldsymbol{\mathcal{A}} \in \mathbb{R}^{I_1 \times r \times I_3}$ and $\boldsymbol{\mathcal{B}} \in \mathbb{R}^{r \times I_2 \times I_3}$, their $t$-product, denoted by $\boldsymbol{\mathcal{A}} \ast_{\boldsymbol{Q}} \boldsymbol{\mathcal{B}}$, is the tensor in $\mathbb{R}^{I_1 \times I_2 \times I_3}$ defined as
	\begin{equation}
		\boldsymbol{\mathcal{A}} \ast_{\boldsymbol{Q}} \boldsymbol{\mathcal{B}}
		\;\coloneqq\;
		\boldsymbol{Q}^{-1}\bigl( \boldsymbol{Q}(\boldsymbol{\mathcal{A}}) \odot \boldsymbol{Q}(\boldsymbol{\mathcal{B}}) \bigr),
	\end{equation}
	where $\odot$ denotes the slice-wise matrix product: for each $i_3 = 1,\dots, I_3$,
	\[
	\bigl[\boldsymbol{Q}(\boldsymbol{\mathcal{A}}) \odot \boldsymbol{Q}(\boldsymbol{\mathcal{B}})\bigr]_{:,:,i_3}
	\;=\;
	\bigl[\boldsymbol{Q}(\boldsymbol{\mathcal{A}})\bigr]_{:,:,i_3}
	\cdot
	\bigl[\boldsymbol{Q}(\boldsymbol{\mathcal{B}})\bigr]_{:,:,i_3}.
	\]
\end{definition}

\begin{definition}[$f$-diagonal tensor \citep{kilmer2011factorization}]
	\label{def:fdiag}
	A tensor $\boldsymbol{\mathcal{S}} \in \mathbb{R}^{I_1 \times I_2 \times I_3}$ is said to be \emph{$f$-diagonal} if every frontal slice $\boldsymbol{\mathcal{S}}_{:,:,i_3}$ is a diagonal matrix for all $i_3 \in \{1,\dots, I_3\}$.
\end{definition}

\begin{definition}[Identity tensor under $\boldsymbol{Q}$ \citep{kernfeld2015tensor}]
	\label{def:identity}
	The tensor $\boldsymbol{\mathcal{I}} \in \mathbb{R}^{I_1 \times I_1 \times I_3}$ is called the \emph{identity tensor} (with respect to $\boldsymbol{Q}$) if its transformed version $\boldsymbol{Q}(\boldsymbol{\mathcal{I}})$ has each frontal slice equal to the identity matrix $\boldsymbol{I}_{I_1} \in \mathbb{R}^{I_1 \times I_1}$.
\end{definition}

\begin{definition}[Tensor conjugate transpose \citep{kernfeld2015tensor}]
	\label{def:transpose}
	For a tensor $\boldsymbol{\mathcal{O}} \in \mathbb{R}^{I_1 \times I_2 \times I_3}$, its \emph{conjugate transpose} $\boldsymbol{\mathcal{O}}^{\top} \in \mathbb{R}^{I_2 \times I_1 \times I_3}$ is defined slice-wise in the transformed domain:
	\[
	\bigl[\boldsymbol{Q}(\boldsymbol{\mathcal{O}}^{\top})\bigr]_{:,:,i_3}
	\;=\;
	\bigl( \bigl[\boldsymbol{Q}(\boldsymbol{\mathcal{O}})\bigr]_{:,:,i_3} \bigr)^{\top},
	\qquad i_3 = 1,\dots, I_3.
	\]
\end{definition}

\begin{definition}[Orthogonal tensor \citep{kernfeld2015tensor}]
	\label{def:orthogonal}
	A tensor $\boldsymbol{\mathcal{O}} \in \mathbb{R}^{I_1 \times I_1 \times I_3}$ is called \emph{orthogonal} (under the $t$-product $\ast_{\boldsymbol{Q}}$) if it satisfies
	\[
	\boldsymbol{\mathcal{O}}^{\top} \ast_{\boldsymbol{Q}} \boldsymbol{\mathcal{O}}
	\;=\;
	\boldsymbol{\mathcal{O}} \ast_{\boldsymbol{Q}} \boldsymbol{\mathcal{O}}^{\top}
	\;=\;
	\boldsymbol{\mathcal{I}},
	\]
	where $\boldsymbol{\mathcal{I}}$ is the identity tensor from Definition \ref{def:identity}.
\end{definition}

\begin{theorem}\label{T-SVD}
	(T-SVD) \citep{kernfeld2015tensor} Let  $\boldsymbol{\mathcal{X}}\in \mathbb{R}^{I_1\times I_2 \times I_3}$. Then it can be factorized as 
	$ 	\boldsymbol{\mathcal{X}}=\boldsymbol{\mathcal{U}} \ast_{\boldsymbol{Q}} \boldsymbol{\mathcal{S}} \ast_{\boldsymbol{Q}} \boldsymbol{\mathcal{V}}^{\top}$,
	where $\boldsymbol{\mathcal{U}}  \in \mathbb{R}^{I_1\times I_1 \times I_3 }$, $\boldsymbol{\mathcal{V}} \in \mathbb{R}^{I_2\times I_2 \times I_3 }$  are orthogonal, and $\boldsymbol{\mathcal{S}} \in\mathbb{R}^{I_1\times I_2 \times I_3}$ is a f-diagonal tensor.
\end{theorem}

The T-SVD in Theorem \ref{T-SVD} is illustrated in figure \ref{Illustrations of T-SVD}.

\begin{figure}[!ht]
	\centering
	{
		\includegraphics[width=0.6\textwidth]{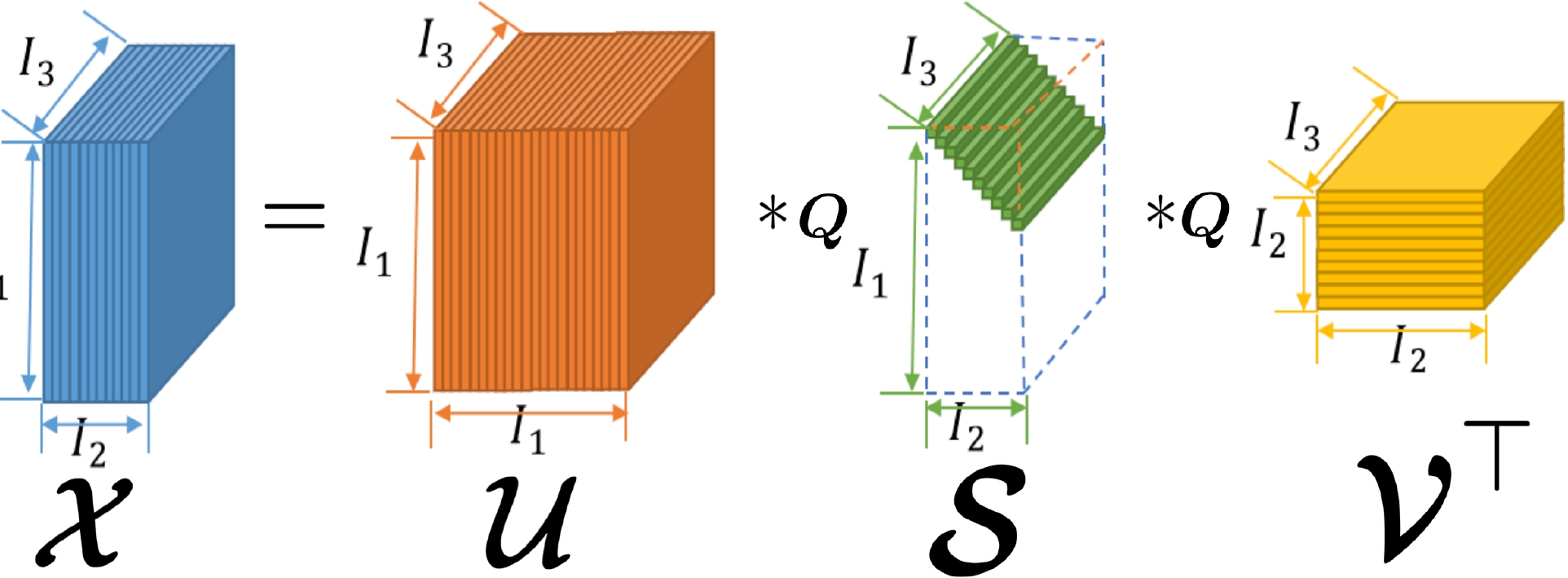}
		\label{}
	}
	\caption{Illustrations of T-SVD}
	\label{Illustrations of T-SVD} 
\end{figure}

The slice decoupling is all that is needed in the tensor-polar proof.

\subsection{Smoothness and the descent lemma}
\label{app:smoothness}

Let $\|\cdot\|$ be a norm and $\|\cdot\|_*$ its dual.  A differentiable function is $L$-smooth \cite{beck2017first} in this norm pair if
\begin{equation}
	\|\nabla F(\boldsymbol{X})-\nabla F(\boldsymbol{Y})\|_*
	\le L\|\boldsymbol{X}-\boldsymbol{Y}\|.
	\label{eq:norm-smoothness-appendix}
\end{equation}

\begin{lemma}[Descent lemma in a general norm]
	\label{lem:general-descent}
	Let $F:\mathbb{R}^{d}\to\mathbb{R}$ be a continuously differentiable function, and let $\|\cdot\|$ be a norm on $\mathbb{R}^{d}$ with dual norm $\|\cdot\|_{*}$. 
	Under ~\eqref{eq:norm-smoothness-appendix}, i.e., its gradient is $L$-Lipschitz under the dual norm,  then, for any $\boldsymbol{X},\boldsymbol{D}\in\mathbb{R}^{d}$,
	\begin{equation}
		F(\boldsymbol{X}+\boldsymbol{D})
		\le F(\boldsymbol{X})+\langle\nabla F(\boldsymbol{X}),\boldsymbol{D}\rangle
		+\frac{L}{2}\|\boldsymbol{D}\|^2.
		\label{eq:general-descent}
	\end{equation}
\end{lemma}

\begin{proof}
	We prove the descent lemma by integrating the smoothness condition along the line segment from $\boldsymbol{X}$ to $\boldsymbol{X}+\boldsymbol{D}$.
	
	First, we give the exact integral representation. 
	Define the scalar function $\phi(s) := F(\boldsymbol{X} + s\boldsymbol{D})$ for $s \in [0,1]$. 
	By the fundamental theorem of calculus,
	\begin{equation}
		F(\boldsymbol{X}+\boldsymbol{D}) - F(\boldsymbol{X})
		= \phi(1) - \phi(0)
		= \int_0^1 \phi'(s) \, ds.
	\end{equation}
	Applying the chain rule gives $\phi'(s) = \langle \nabla F(\boldsymbol{X} + s\boldsymbol{D}), \boldsymbol{D} \rangle$, hence
	\begin{equation}
		F(\boldsymbol{X}+\boldsymbol{D}) - F(\boldsymbol{X})
		= \int_0^1 \langle \nabla F(\boldsymbol{X} + s\boldsymbol{D}), \boldsymbol{D} \rangle \, ds.
		\label{eq:proof-descent-integral}
	\end{equation}
	To isolate the first-order term, we add and subtract $\langle \nabla F(\boldsymbol{X}), \boldsymbol{D} \rangle$ inside the integral:
	\begin{align}
		F(\boldsymbol{X}+\boldsymbol{D}) - F(\boldsymbol{X})
		&= \int_0^1 \langle \nabla F(\boldsymbol{X}), \boldsymbol{D} \rangle \, ds
		+ \int_0^1 \langle \nabla F(\boldsymbol{X} + s\boldsymbol{D}) - \nabla F(\boldsymbol{X}), \boldsymbol{D} \rangle \, ds \\
		&= \langle \nabla F(\boldsymbol{X}), \boldsymbol{D} \rangle
		+ \int_0^1 \langle \nabla F(\boldsymbol{X} + s\boldsymbol{D}) - \nabla F(\boldsymbol{X}), \boldsymbol{D} \rangle \, ds.
		\label{eq:proof-descent-decomposed}
	\end{align}
	The first term is exactly the linear approximation at $\boldsymbol{X}$; the integral captures the accumulated deviation of the gradient along the path.
	
	Then we bound the integral using smoothness. 
	For each $s \in [0,1]$, the dual H\"older inequality gives
	\begin{equation}
		\left| \langle \nabla F(\boldsymbol{X} + s\boldsymbol{D}) - \nabla F(\boldsymbol{X}), \boldsymbol{D} \rangle \right|
		\le \| \nabla F(\boldsymbol{X} + s\boldsymbol{D}) - \nabla F(\boldsymbol{X}) \|_* \, \| \boldsymbol{D} \|.
	\end{equation}
	By the $L$-smoothness assumption in ~\eqref{eq:norm-smoothness-appendix}, the gradient deviation is bounded by
	\begin{equation}
		\| \nabla F(\boldsymbol{X} + s\boldsymbol{D}) - \nabla F(\boldsymbol{X}) \|_*
		\le L \| s\boldsymbol{D} \|
		= Ls \| \boldsymbol{D} \|.
	\end{equation}
	Combining these two bounds yields
	\begin{equation}
		\left| \langle \nabla F(\boldsymbol{X} + s\boldsymbol{D}) - \nabla F(\boldsymbol{X}), \boldsymbol{D} \rangle \right|
		\le Ls \| \boldsymbol{D} \|^2.
		\label{eq:proof-descent-integrand-bound}
	\end{equation}

	Applying the bound in \eqref{eq:proof-descent-integrand-bound} to the integral in ~\eqref{eq:proof-descent-decomposed}, we obtain
	\begin{align}
		F(\boldsymbol{X}+\boldsymbol{D}) - F(\boldsymbol{X})
		&\le \langle \nabla F(\boldsymbol{X}), \boldsymbol{D} \rangle
		+ \int_0^1 Ls \| \boldsymbol{D} \|^2 \, ds \\
		&= \langle \nabla F(\boldsymbol{X}), \boldsymbol{D} \rangle
		+ \frac{L}{2} \| \boldsymbol{D} \|^2.
	\end{align}
	This proves the claimed descent inequality.
\end{proof}

\section{Related work}\label{related-work-appendix}
Optimizers can be categorized by the structural granularity they exploit. 
\paragraph{Vector-aware optimizers.}
Vector-aware optimizers treat parameters as vectors, updating each coordinate independently based on per-coordinate gradient statistics.   Adagrad~\citep{duchi2011adaptive} introduced per-parameter scaling but suffers from aggressive learning-rate decay due to accumulating squared gradients. RMSprop~\citep{hinton2012neural} addressed this by employing exponential moving averages, providing more stable updates. Adam~\citep{kingma2015adam}, combining momentum with adaptive scaling, has become the default optimizer in most deep learning pipelines, with AdamW~\citep{loshchilov2017decoupled} further improving generalization by decoupling weight decay from gradient updates.

Recent variants enhance efficiency and convergence through various strategies. Adan~\citep{xie2024adan} and Win~\citep{zhou2024win,zhou2023win} strengthen Adam with Nesterov-style acceleration for faster convergence. Lion~\citep{chen2023symbolic} removes second-moment tracking entirely, trading some adaptability for significant memory savings. Sophia~\citep{liu2023sophia} leverages approximate second-order information (Hessian diagonal) to accelerate pre-training. Adam-mini~\citep{zhang2024adam} reduces memory footprint via block-wise learning rates, partitioning parameters into blocks that share adaptive scales. Despite their practical success, all these methods fundamentally treat parameters as flattened vectors, updating each coordinate independently and disregarding the structural correlations inherent in matrix-valued or tensor-valued parameters such as weight matrices and convolutional kernels.

\paragraph{Matrix-aware optimizers.}
Recognizing the limitations of vector-wise updates, a growing body of work exploits matrix structure to improve optimization. K-FAC~\citep{martens2015optimizing} and Shampoo~\citep{gupta2018shampoo} use Kronecker-factored approximations of the Fisher or curvature matrix, enabling efficient second-order updates that respect matrix geometry. Adafactor~\citep{shazeer2018adafactor} reduces memory through factored second-moment estimates, while LAMB~\citep{you2019large} introduces layer-wise normalization for stable large-batch training. More recently, GaLore~\citep{zhao2024galore} projects gradients into low-rank subspaces to substantially reduce optimizer memory usage during pre-training. SOAP~\citep{vyas2025soap} interprets Shampoo in the eigenbasis of its preconditioner and performs Adam-style updates in that basis, achieving improved performance over Adam at the cost of additional computational overhead. AdaDiag~\citep{nguyen2025improving} leverages SVD to transform gradients into a near-diagonal preconditioning space, accelerating convergence through better-conditioned updates.

Among matrix-aware methods, Muon~\citep{jordan2024muon} takes a distinctly different approach. Instead of constructing a preconditioner, Muon applies the polar factor—the unitary part of the polar decomposition—to each layer-wise momentum matrix via efficient Newton--Schulz iterations. This operation replaces all nonzero singular values with one, effectively regularizing the spectral distribution of update matrices and preventing rank collapse. The result is improved isotropy and training stability, particularly in large-scale settings where gradient rank collapse is a known issue. Subsequent work has scaled Muon to larger architectures~\citep{liu2025muon}. Other recent advances further explore structural efficiency and scalability~\citep{wang20244,chen2024fira,pethick2025training,liu2025cosmos,an2025asgo,xie2025structured,yu2024zeroth}. Collectively, these matrix-aware methods demonstrate that structural awareness—particularly spectral operations at the matrix level—can substantially enhance training efficiency, naturally motivating further exploration into higher-order tensor structures.

\paragraph{Tensor-aware optimizers.}
Recent advances in optimization for deep neural networks have sought to exploit the structured nature of parameters beyond the vectorized view. 
Tensorion~\cite{bogachev2026tensorion} extends Muon to higher-order parameter tensors, such as convolutional kernels, by replacing the matrix spectral norm with a designed tensor norm that balances tightness to the tensor spectral norm. Concretely, Tensorion defines a relaxed dual norm as the maximum of nuclear norms over a set of tensor unfoldings, and its primal norm is obtained via duality. The optimizer then selects a matricization mode that yields the most balanced unfolding, thereby maximizing a high-probability upper bound on the nuclear norm. Then it applies Newton--Schulz orthogonalization to the resulting matrix before folding it back. 

Tensorion is primarily designed for convolutional neural networks and other vision-oriented architectures where parameters naturally appear as 4D tensors. Empirically, it demonstrates improved convergence and test accuracy over both SGD-based and Adam-based baselines on CIFAR and Tiny ImageNet datasets, as well as consistent gains when applied to the convolutional patch-embedding module in Vision Transformers. However, its reliance on a fixed unfolding per layer, while computationally efficient, may not fully adapt to dynamically changing gradient structures during training, and its validation is largely confined to computer vision tasks, with limited exploration in large-scale language modeling or sequential domains.

In contrast, TEON~\cite{zhang2026teon} generalizes Muon across layers rather than within a single tensor. Motivated by the observation that layer-wise orthogonalization neglects cross-layer gradient correlations, TEON stacks gradient matrices of the Query, Key, and Value projections of Transformer blocks into a third-order tensor and applies mode-1 matricization followed by matrix orthogonalization. This choice is grounded in a theoretical analysis under the non-Euclidean trust-region framework, which shows that TEON can achieve up to a \(\sqrt{K}\)-fold improvement over Muon in convergence bounds when the stacked gradients exhibit aligned top right singular vectors. TEON is thus tailored for large language model pre-training, where parameter blocks are predominantly matrices but exhibit strong inter-layer and inter-component correlations. Experimental results on GPT-style and LLaMA-style models, ranging from 60M to 1B parameters, demonstrate consistent perplexity improvements over Muon and AdamW under various approximate SVD schemes. Nevertheless, \(\sqrt{K}\)-fold convergence improvement is theoretically appealing but practically limited, as its benefits are realized only at 
$K=2$. Larger stacks fail to maintain the required singular-vector alignment and yield diminishing performance.
More fundamentally, TEON applies a single global orthogonalization to the entire concatenated unfolding. The flat shape of a single matrix unfolded by all stacked layers may cause information loss. It inevitably suppresses weaker but potentially informative layer-specific signals when \(K\) grows or when layer-wise directions are not perfectly aligned.

The limitations of Tensorion and TEON reveal a common structural tension: Tensorion operates \emph{intra-tensor} by selecting a single unfolding of a native higher-order parameter, while TEON operates \emph{inter-tensor} by globally matricizing stacked layers, but both rely on a single, fixed matricization that collapses cross-mode or cross-layer information into a monolithic matrix before orthogonalization. This raises a natural question: can we design a tensor-aware optimizer that captures cross-layer structure without sacrificing per-layer expressivity, while avoiding the NP-hardness of the full tensor spectral norm?

We answer this affirmatively with \textbf{TUON}, a novel optimizer built upon transform-based tensor algebra. Instead of globally matricizing stacked layers, TUON applies a unitary transform \(\boldsymbol{Q}\) along the layer mode, effectively rotating cross-layer information into a transform domain where each frontal slice becomes an independent \(m \times n\) matrix representing a transformed cross-layer mode. Critically, TUON then applies the Muon polar map \emph{independently to each transformed slice} before inverting the transform. This "mix-then-separate" strategy captures cross-layer correlations through the unitary mixing while preserving per-layer diversity through slice-wise polar maps, preventing the dominance of a single global direction that plagues TEON's global unfolding.

TUON is algebraically distinct from both predecessors. Unlike TEON, whose global polar map is invariant to unitary layer transforms, TUON's slice-wise polar map does \emph{not} commute with the layer transform, making the choice of \(\boldsymbol{Q}\) genuinely consequential and enabling the optimizer to adapt to the data-dependent structure of layer-wise gradients. Unlike Tensorion, TUON operates in an exact LMO sense under a well-defined \(\boldsymbol{Q}\) spectral norm, with dual norm equal to the sum of transformed-slice nuclear norms, eliminating the need for unfolding-mode search or spectral-norm relaxation. Theoretically, TUON enjoys a geometry-matched convergence guarantee and achieves up to a \(\sqrt{K}\)-fold improvement over TEON in both radius-matched and energy-matched oracle alignments, with the gain governed by a measurable alignment factor. Practically, the method adds only a lightweight unitary transform and retains the same Newton--Schulz iterations as Muon, making it scalable and easy to integrate into existing training pipelines. By combining the theoretical rigor of exact LMO geometry with the practical efficiency of transform-based tensor algebra, TUON offers a principled and generalizable solution to tensor-aware optimization that bridges the gap between intra-tensor structure and inter-layer correlations.

\paragraph{Tensor algebra.}
In the past years, two classical frameworks have dominated tensor decomposition related tensor algebra. 
CP decomposition factorizes a tensor as a sum of rank-one tensors; while conceptually straightforward, it is computationally demanding and its rank is NP-hard to compute \citep{kolda2009tensor}.
Tucker decomposition, also known as Higher-Order Singular Value Decomposition (HOSVD), generalizes matrix SVD by decomposing a tensor into a core tensor multiplied by factor matrices along each mode \citep{tucker1966some, de2000multilinear}. 
These methods are grounded in a \textbf{matricization-based tensor algebra}, where tensors are first unfolded into matrices, and all subsequent algebraic operations are carried out at the matrix level. 
Although HOSVD successfully produces orthogonal bases per mode, the resulting core tensor is generally not f-diagonal (definitions can be found in Appendix~\ref{app:qproduct}), losing the neat spectral interpretation of diagonal singular values in the matrix case.
To remedy this, \citet{kilmer2011factorization} introduced the t-product, a novel multiplication based on circular convolution and the Fast Fourier Transform (FFT).
The t-product induces the tensor SVD (T-SVD):
\[
\boldsymbol{\mathcal{X}} = \boldsymbol{\mathcal{U}} \ast_{\boldsymbol{Q}} \boldsymbol{\mathcal{S}} \ast_{\boldsymbol{Q}} \boldsymbol{\mathcal{V}}^{\top},
\]
where \(\boldsymbol{\mathcal{U}}\) and \(\boldsymbol{\mathcal{V}}\) are orthogonal, and \(\boldsymbol{\mathcal{S}}\) is f-diagonal, thereby mirroring the matrix SVD structure.
Computation proceeds by applying FFT along the third dimension, performing matrix SVD on each frontal slice, and inverting the transform. 
A major generalization is made by \citet{kernfeld2015tensor}, who show that the t-product only depends on the chosen invertible linear transform and need not be limited to the Fourier transform.
This leads to tensor-tensor products with arbitrary invertible transforms, such as the Discrete Cosine Transform (DCT), wavelets, or even data-dependent transforms, greatly expanding flexibility and applicability \citep{kernfeld2015tensor, lu2019low}.
This yields a family of products and factorizations with algebraic structures \citep{kernfeld2015tensor}. 
These methods rely on the t-product, a tensor algebra defined natively on tensors rather than through unfolding. Matrix operations appear only as slice-wise operations in the transform domain. We refer to this paradigm as \textbf{transform-based tensor algebra}.

Subsequent works systematically investigate  relations between T-SVD and HOSVD and established algebraic properties such as orthogonal invariance.
Meanwhile, the tensor nuclear norm (TNN) and the tubal rank were defined based on T-SVD, providing powerful convex optimization tools for low-rank tensor recovery, including tensor completion and robust principal component analysis \citep{lu2018exact, lu2019low, zheng2020tensor}.
Algorithmically, the original T-SVD requires computing many matrix SVDs. \citet{lu2018exact} reduce the cost by leveraging conjugate symmetry in the Fourier domain, requiring only about half the SVD computations.
More recently, randomized T-SVD algorithms have been proposed for large-scale data via tensor random projections to obtain fast low-rank approximations.
These tools have been successfully applied to tensor completion, robust principal component analysis, and multilinear systems \citep{lu2019low, zheng2020tensor, lu2018exact}.
Recent work extends these ideas to higher-order tensors via separable transforms along multiple modes \citep{zheng2020tensor}, and to data-dependent transforms that adapt the basis to the data at hand \citep{kong2021tensor, zheng2026refta}. 
General tensor algebra has also been considered in low-rank optimization~\citep{yang2026space,yang2026optimization,yang2025variational}. 
We specialize these tools to a full-parameter optimizer and develop comparative, convergence, and inexact-oracle analyses. 

The definitions and theorems related to t-product and T-SVD can be found in the Preliminaries \ref{app:qproduct} in Appendix.

\section{Muon and TEON Geometries}
\label{app:baselines}

This appendix fixes the baseline notation used in the main comparisons.  The discussion concerns ideal exact polar maps; finite-iteration implementations may approximate them.

\subsection{Muon}
\label{app:muon-baseline}

For a stack $\boldsymbol{\mathcal{X}} \in \mathbb{C}^{m \times n \times K}$ with frontal slices $\boldsymbol{X}_k \in \mathbb{C}^{m \times n}$ for $k=1,\ldots,K$, we define the Muon norm pair as follows:
\begin{equation}
	\|\boldsymbol{\mathcal{X}}\|_{\mathrm{Muon}}
	:=\max_{1 \le k \le K} \|\boldsymbol{X}_k\|_{\mathrm{op}},
	\qquad
	\|\boldsymbol{\mathcal{G}}\|_{\mathrm{Muon},*}
	:=\sum_{k=1}^K \|\boldsymbol{G}_k\|_*.
	\label{eq:appendix-muon-norms}
\end{equation}
Here, $\boldsymbol{\mathcal{X}}$ denotes a generic tensor direction (used as the argument of the primal norm), while $\boldsymbol{\mathcal{G}}$ denotes a gradient tensor (used as the argument of the dual norm). The constraint $\|\boldsymbol{\mathcal{D}}\|_{\mathrm{Muon}} \le \rho$ means that each layer's update direction $\boldsymbol{D}_k$ has spectral norm at most $\rho$, independently of all other layers.

\begin{proposition}[Ideal Muon as an exact LMO]
	\label{prop:muon-lmo-appendix}
	Suppose each gradient slice has a compact singular value decomposition
	\begin{equation}
		\boldsymbol{G}_k = \boldsymbol{U}_k \boldsymbol{\Sigma}_k \boldsymbol{V}_k^{\mathsf H},
		\qquad k = 1,\ldots,K,
	\end{equation}
	where $\boldsymbol{U}_k \in \mathbb{C}^{m \times r_k}$ and $\boldsymbol{V}_k \in \mathbb{C}^{n \times r_k}$ have orthonormal columns, and $\boldsymbol{\Sigma}_k = \operatorname{diag}(\sigma_{k,1},\ldots,\sigma_{k,r_k}) \in \mathbb{R}^{r_k \times r_k}$ is diagonal with positive singular values. Then an optimal solution of the linear minimization problem
	\begin{equation}
		\min_{\|\boldsymbol{\mathcal{D}}\|_{\mathrm{Muon}} \le \rho}
		\langle \boldsymbol{\mathcal{G}}, \boldsymbol{\mathcal{D}} \rangle_{\mathrm F}
	\end{equation}
	is given by the slices
	\begin{equation}
		\boldsymbol{D}_k^* = -\rho \, \boldsymbol{U}_k \boldsymbol{V}_k^{\mathsf H},
		\qquad k = 1,\ldots,K,
	\end{equation}
	and the corresponding optimal value is
	\begin{equation}
		\min_{\|\boldsymbol{\mathcal{D}}\|_{\mathrm{Muon}} \le \rho}
		\langle \boldsymbol{\mathcal{G}}, \boldsymbol{\mathcal{D}} \rangle_{\mathrm F}
		= -\rho \sum_{k=1}^K \|\boldsymbol{G}_k\|_*.
	\end{equation}
\end{proposition}

\begin{proof}
	The trust-region constraint $\|\boldsymbol{\mathcal{D}}\|_{\mathrm{Muon}} \le \rho$ is, by definition, equivalent to the $K$ independent constraints
	\begin{equation}
		\|\boldsymbol{D}_k\|_{\mathrm{op}} \le \rho
		\qquad \text{for all } k = 1,\ldots,K.
	\end{equation}
	The objective also decouples:
	\begin{equation}
		\langle \boldsymbol{\mathcal{G}}, \boldsymbol{\mathcal{D}} \rangle_{\mathrm F}
		= \sum_{k=1}^K \langle \boldsymbol{G}_k, \boldsymbol{D}_k \rangle_{\mathrm F}.
	\end{equation}
	Thus the overall optimization separates into $K$ independent matrix subproblems, each of the form
	\begin{equation}
		\min_{\|\boldsymbol{D}_k\|_{\mathrm{op}} \le \rho}
		\langle \boldsymbol{G}_k, \boldsymbol{D}_k \rangle_{\mathrm F}.
	\end{equation}
	By Lemma~\ref{lem:matrix-holder}, the optimal solution to each subproblem is
	\begin{equation}
		\boldsymbol{D}_k^* = -\rho \, \operatorname{polar}(\boldsymbol{G}_k)
		= -\rho \, \boldsymbol{U}_k \boldsymbol{V}_k^{\mathsf H},
	\end{equation}
	with optimal value $-\rho \|\boldsymbol{G}_k\|_*$. Summing the $K$ independent optimal values proves the claim.
\end{proof}

\subsection{TEON}
\label{app:teon-baseline}

Without loss of generality, for analysis, let $\mathscr U(\boldsymbol{\mathcal{X}})=[\boldsymbol{X}_1,\ldots,\boldsymbol{X}_K]$.  Define
\begin{equation}
	\|\boldsymbol{\mathcal{X}}\|_{\mathrm{TEON}}
	:=\|\mathscr U(\boldsymbol{\mathcal{X}})\|_{\mathrm{op}},
	\qquad
	\|\boldsymbol{\mathcal{G}}\|_{\mathrm{TEON},*}
	:=\|\mathscr U(\boldsymbol{\mathcal{G}})\|_*.
	\label{eq:appendix-teon-norms}
\end{equation}
Unlike Muon, TEON couples all slices through one concatenated matrix constraint.

\begin{proposition}[Ideal TEON as an exact LMO]
	\label{prop:teon-lmo-appendix}
	Let $\mathscr U(\boldsymbol{\mathcal{G}})=\boldsymbol{U}\boldsymbol{\Sigma}\boldsymbol{V}^{\mathsf H}$.  Then
	\begin{equation}
		\boldsymbol{\mathcal{D}}_{\mathrm{TEON}}^*
		=-\rho\mathscr U^{-1}(\boldsymbol{U}\boldsymbol{V}^{\mathsf H})
	\end{equation}
	solves
	\begin{equation}
		\min_{\|\boldsymbol{\mathcal{D}}\|_{\mathrm{TEON}}\le\rho}
		\langle\boldsymbol{\mathcal{G}},\boldsymbol{\mathcal{D}}\rangle_{\mathrm F},
	\end{equation}
	with optimal value $-\rho\|\mathscr U(\boldsymbol{\mathcal{G}})\|_*$.
\end{proposition}

\begin{proof}
	Matricization preserves the Frobenius inner product, so the tensor problem is exactly
	\begin{equation}
		\min_{\|\boldsymbol{D}_{(1)}\|_{\mathrm{op}}\le\rho}
		\langle\mathscr U(\boldsymbol{\mathcal{G}}),\boldsymbol{D}_{(1)}\rangle_{\mathrm F}.
	\end{equation}
	Lemma~\ref{lem:matrix-holder} gives $\boldsymbol{D}_{(1)}^*=-\rho\boldsymbol{U}\boldsymbol{V}^{\mathsf H}$.  Folding this matrix back to a tensor gives the stated direction.
\end{proof}

\subsection{How a layer transform acts on the TEON unfolding}
\label{app:unfold-transform-proof}

We establish a key algebraic identity that underpins the comparison between TEON and our proposed method. 

Let $\boldsymbol{\mathcal{X}} \in \mathbb{C}^{m \times n \times K}$ be a third-order tensor with frontal slices $\boldsymbol{X}_1, \ldots, \boldsymbol{X}_K$, and let $\mathscr U(\boldsymbol{\mathcal{X}}) := [\boldsymbol{X}_1, \ldots, \boldsymbol{X}_K] \in \mathbb{C}^{m \times nK}$ denote its horizontal unfolding. Under the mode-3 transform convention, applying a unitary matrix $\boldsymbol{Q} \in \mathbb{C}^{K \times K}$ along the layer dimension and then unfolding is equivalent to first unfolding and then multiplying on the right by a Kronecker product:
\begin{equation}
	\mathscr U(\boldsymbol{\mathcal{X}} \times_3 \boldsymbol{Q})
	= \mathscr U(\boldsymbol{\mathcal{X}}) (\boldsymbol{Q}^{\mathsf T} \otimes \boldsymbol{I}_n).
	\label{eq:unfold-transform-identity}
\end{equation}

The Kronecker product $\boldsymbol{R}_{\boldsymbol{Q}} := \boldsymbol{Q}^{\mathsf T} \otimes \boldsymbol{I}_n$ is unitary whenever $\boldsymbol{Q}$ is:
\begin{equation}
	\boldsymbol{R}_{\boldsymbol{Q}}^{\mathsf H} \boldsymbol{R}_{\boldsymbol{Q}}
	= (\overline{\boldsymbol{Q}} \otimes \boldsymbol{I}_n)(\boldsymbol{Q}^{\mathsf T} \otimes \boldsymbol{I}_n)
	= (\overline{\boldsymbol{Q}} \boldsymbol{Q}^{\mathsf T}) \otimes \boldsymbol{I}_n
	= \boldsymbol{I}_{nK}.
\end{equation}

This identity reveals a crucial distinction between TEON and TUON. The TEON norm is defined as the operator norm of the horizontal unfolding:
\begin{equation}
	\|\boldsymbol{\mathcal{X}}\|_{\mathrm{TEON}}
	:= \|\mathscr U(\boldsymbol{\mathcal{X}})\|_{\mathrm{op}}.
\end{equation}
Since $\boldsymbol{R}_{\boldsymbol{Q}}$ is unitary, right-multiplication by $\boldsymbol{R}_{\boldsymbol{Q}}$ leaves the operator norm invariant:
\begin{equation}
	\|\mathscr U(\boldsymbol{\mathcal{X}} \times_3 \boldsymbol{Q})\|_{\mathrm{op}}
	= \|\mathscr U(\boldsymbol{\mathcal{X}}) \boldsymbol{R}_{\boldsymbol{Q}}\|_{\mathrm{op}}
	= \|\mathscr U(\boldsymbol{\mathcal{X}})\|_{\mathrm{op}}.
\end{equation}
Consequently, the TEON norm is invariant to the choice of the unitary layer transform $\boldsymbol{Q}$: the transform is entirely ``absorbed'' by the unitary invariance of the spectral norm.

This stands in sharp contrast to TUON, where the layer transform is applied \emph{before} the nonlinear slice-wise polar map. Because the polar map does not commute with unitary transformations (as shown in Theorem~\ref{thm:nonredundancy}), the TUON geometry depends nontrivially on $\boldsymbol{Q}$. The identity in ~\eqref{eq:unfold-transform-identity} is therefore instrumental in the TEON comparison (Section~\ref{subsec:muon-comparison}) and the transform-cancellation proof (Theorem~\ref{thm:nonredundancy}), where it formalizes why TEON cannot exploit cross-layer structure through a unitary layer basis.

\section{Selection of $Q$}
\label{Selection-of-Q}

\subsection{Momentum tensor and layer transform}
For one compatible group, stack the gradients and momenta as
$
\cG_t(:,:,k)=\boldsymbol{G}_t^{(k)}$ and $
\cM_t=\beta\cM_{t-1}+(1-\beta)\cG_t$
where $\cM_t\in\R^{m\times n\times K}$. Given a unitary $\boldsymbol{Q}\in \mathbb{C}^{K\times K}$, $\boldsymbol{Q}^{\mathsf H} \boldsymbol{Q}=I$, our tensor-product-transform treats every tube as a length-$K$ column,
\begin{equation}
	\widehat{\cM}_t=\cM_t\times_3\boldsymbol{Q},
	\quad
	\text{where the $j$-th slice of $\widehat{\cM}_t$}\quad \widehat{\boldsymbol{M}}_{t,j}=\widehat{\boldsymbol{\mathcal{M}}}_t(:,:,j)=\sum_{k=1}^K \boldsymbol{Q}_{jk}\boldsymbol{M}_t^{(k)}.
	\label{eq:method-transform}
\end{equation}

Before we study the choice of $\boldsymbol{Q}$, we   establish a rigorous equivalence between tensor slice representation and matrix multiplication.  

We use the sans-serif bold symbol $\mathbf{\mathsf{M}}$ exclusively for the mode-3 unfolding matrix, reserving the italic bold $\boldsymbol{M}$ for individual slice matrices (e.g., $\boldsymbol{M}_{t}^{(k)}$).

First, 	we establish a rigorous equivalence between the slice-wise linear combination
\begin{equation}
	\widehat{\boldsymbol{M}}_{t,j} = \sum_{k=1}^{K} Q_{jk} \, \boldsymbol{M}_{t}^{(k)}, \qquad j = 1, \dots, K,
	\label{eq:slice-combination}
\end{equation}
and the mode-3 matrix multiplication
\begin{equation}
	\widehat{\mathbf{\mathsf{M}}}_t = \mathbf{Q} \, \mathbf{\mathsf{M}}_t,
	\label{eq:matrix-multiplication}
\end{equation}
where $\boldsymbol{M}_{t}^{(k)} \in \mathbb{R}^{m \times n}$ are the original slices, $Q_{jk} \in \mathbb{C}$ are the entries of $\mathbf{Q} \in \mathbb{C}^{K \times K}$, and $\widehat{\boldsymbol{M}}_{t,j} \in \mathbb{R}^{m \times n}$ are the resulting slices.

It can be proved as follows.  
Define the \textbf{vectorization} operator $\operatorname{vec}: \mathbb{R}^{m \times n} \to \mathbb{R}^{mn}$ that stacks the columns of a matrix into a single column vector. We construct the mode-3 unfolding matrix $\mathbf{\mathsf{M}}_t \in \mathbb{R}^{K \times mn}$ by stacking the transposed vectorizations of each slice as rows:
\begin{equation}
	\mathbf{\mathsf{M}}_t =
	\begin{bmatrix}
		\operatorname{vec}(\boldsymbol{M}_{t}^{(1)})^{\mathsf T} \\
		\operatorname{vec}(\boldsymbol{M}_{t}^{(2)})^{\mathsf T} \\
		\vdots \\
		\operatorname{vec}(\boldsymbol{M}_{t}^{(K)})^{\mathsf T}
	\end{bmatrix}.
	\label{eq:mode3-unfolding}
\end{equation}
Equivalently, the $k$-th row of $\mathbf{\mathsf{M}}_t$ satisfies
\begin{equation}
	(\mathbf{\mathsf{M}}_t)_{k, :} = \operatorname{vec}(\boldsymbol{M}_{t}^{(k)})^{\mathsf T}.
\end{equation}

\textbf{Then we consider the matrix multiplication. }
Let $\mathbf{Q} \in \mathbb{C}^{K \times K}$. Compute the matrix product
\begin{equation}
	\widehat{\mathbf{\mathsf{M}}}_t = \mathbf{Q} \, \mathbf{\mathsf{M}}_t \in \mathbb{C}^{K \times mn}.
	\label{eq:Q-left-multiply}
\end{equation}
By the definition of matrix multiplication, the $j$-th row of $\widehat{\mathbf{\mathsf{M}}}_t$ is
\begin{equation}
	(\widehat{\mathbf{\mathsf{M}}}_t)_{j, :} = \sum_{k=1}^{K} Q_{jk} \, (\mathbf{\mathsf{M}}_t)_{k, :}.
	\label{eq:row-combination}
\end{equation}

\textbf{Next, we apply the inverse vectorization operator} $\operatorname{vec}^{-1}: \mathbb{R}^{mn} \to \mathbb{R}^{m \times n}$ to each row of $\widehat{\mathbf{\mathsf{M}}}_t$. By linearity of $\operatorname{vec}^{-1}$,
\begin{equation}
	\operatorname{vec}^{-1}\!\left( (\widehat{\mathbf{\mathsf{M}}}_t)_{j, :}^{\mathsf T} \right)
	= \operatorname{vec}^{-1}\!\left( \sum_{k=1}^{K} Q_{jk} \, (\mathbf{\mathsf{M}}_t)_{k, :}^{\mathsf T} \right)
	= \sum_{k=1}^{K} Q_{jk} \, \operatorname{vec}^{-1}\!\left( (\mathbf{\mathsf{M}}_t)_{k, :}^{\mathsf T} \right).
\end{equation}
By construction of $\mathbf{\mathsf{M}}_t$ in ~\eqref{eq:mode3-unfolding},
\begin{equation}
	\operatorname{vec}^{-1}\!\left( (\mathbf{\mathsf{M}}_t)_{k, :}^{\mathsf T} \right) = \boldsymbol{M}_{t}^{(k)}.
\end{equation}
Therefore, defining
\begin{equation}
	\widehat{\boldsymbol{M}}_{t,j} := \operatorname{vec}^{-1}\!\left( (\widehat{\mathbf{\mathsf{M}}}_t)_{j, :}^{\mathsf T} \right),
\end{equation}
we obtain
\begin{equation}
	\widehat{\boldsymbol{M}}_{t,j} = \sum_{k=1}^{K} Q_{jk} \, \boldsymbol{M}_{t}^{(k)}.
	\label{eq:recovered-slice}
\end{equation}

As a conclusion, ~\eqref{eq:recovered-slice} is identical to ~\eqref{eq:slice-combination}. Hence,
\begin{equation}
	\widehat{\boldsymbol{M}}_{t,j} = \sum_{k=1}^{K} Q_{jk} \boldsymbol{M}_{t}^{(k)}
	\;\Longleftrightarrow\;
	\widehat{\mathbf{\mathsf{M}}}_t = \mathbf{Q} \mathbf{\mathsf{M}}_t
\end{equation}
where the equivalence holds up to the bijective mapping between slices and rows established by $\operatorname{vec}$ and $\operatorname{vec}^{-1}$. Consequently, the slice-wise linear combination is precisely the row-wise linear combination effected by left-multiplication by $\mathbf{Q}$ on the mode-3 unfolding.

The equivalence above justifies interpreting $\times_3$ as left-multiplication on the mode-3 unfolding. 
Under above formulations, we give several concrete choices of $\boldsymbol{Q}$ and analyze the default DecorQ method we choose .
\paragraph{DCT.}
For the orthonormal DCT-II,
\[
\boldsymbol{Q}_{jk}=\alpha_j\cos\!\left[\frac{\pi}{K}\left(k+\frac12\right)j\right],
\quad
\alpha_0=K^{-1/2},\quad \alpha_j=(2/K)^{1/2}\;(j>0),
\]
with zero-based indices.  It is real and expresses smooth and oscillatory depth modes without imposing a periodic endpoint.

The main practical transform is as follow. For $K=2$,
\begin{equation}
	\widehat{\boldsymbol{M}}_{t,1}=\frac{{\boldsymbol{M}}_t^{(1)}+{\boldsymbol{M}}_t^{(2)}}{\sqrt2},
	\qquad
	\widehat{\boldsymbol{M}}_{t,2}=\frac{{\boldsymbol{M}}_t^{(1)}-{\boldsymbol{M}}_t^{(2)}}{\sqrt2}.
	\label{eq:two-point-dct}
\end{equation}
The first slice is a shared adjacent-layer mode and the second is a residual mode.

\paragraph{Haar transform.}
The Haar transform is the simplest orthogonal wavelet transform. For a vector of length $N = 2^n$, it recursively decomposes adjacent pairs of entries into their sum and difference, producing a multi-scale representation: the leading coefficients encode global averages, while the remaining coefficients encode local differences at increasingly finer scales.

Formally, the Haar matrix $H_N$ is defined recursively as
\begin{equation}
	H_1 = [1], \qquad
	H_{2n} = \frac{1}{\sqrt{2}}
	\begin{bmatrix}
		H_n \otimes [1,\ 1] \\
		I_n \otimes [1,\ -1]
	\end{bmatrix},
\end{equation}
where $\otimes$ denotes the Kronecker product and $I_n$ is the $n \times n$
identity matrix. For example,
\begin{equation}
	H_2 = \frac{1}{\sqrt{2}}
	\begin{bmatrix}
		1 & 1 \\
		1 & -1
	\end{bmatrix}, \qquad
	H_4 = \frac{1}{2}
	\begin{bmatrix}
		1 & 1 & 1 & 1 \\
		1 & 1 & -1 & -1 \\
		\sqrt{2} & -\sqrt{2} & 0 & 0 \\
		0 & 0 & \sqrt{2} & -\sqrt{2}
	\end{bmatrix}.
\end{equation}

The Haar matrix is orthogonal, i.e., $H_N H_N^\top = H_N^\top H_N = I$, and
can be applied in $O(N)$ time via a butterfly scheme. Unlike the DCT, whose
basis functions are global cosine waves, each Haar basis function is supported on a small number of adjacent coordinates. It is therefore well suited to capturing local, piecewise, or multi-scale correlation structures, while requiring only additions, subtractions, and powers of $\sqrt{2}$, with no trigonometric evaluations.

\paragraph{DFT/FFT.}
The unitary DFT uses $\boldsymbol{Q}_{jk}=K^{-1/2}e^{-2\pi i jk/K}$.  It is implemented with an orthonormal FFT.  Frontal slices and NS products are complex-valued, with Hermitian rather than ordinary transposes.

\paragraph{DecorQ.} 
Let \(\mathsf{M}_t \in \mathbb{R}^{K \times mn}\) be the mode-3 unfolding of the momentum tensor \(\boldsymbol{\mathcal{M}}_t\), whose \(k\)-th row is \(\operatorname{vec}(\boldsymbol{M}_t^{(k)})^{\mathsf T}\). We set \(\boldsymbol{Q} = U_3^{\mathsf T}\), where \(U_3\) is the left singular matrix from the SVD \(\mathsf{M}_t = U_3 \Sigma_3 V_3^{\mathsf T}\). This choice is not arbitrary: it diagonalizes the layer-wise Gram matrix, inducing a geometrically decoupled coordinate system in which subsequent operations act independently on each slice.

To see this, consider the Gram matrix of the original slices:
\begin{equation}
	\mathsf{C}_t := \mathsf{M}_t \mathsf{M}_t^{\mathsf T} \in \mathbb{R}^{K \times K}, \qquad
	(\mathsf{C}_t)_{ij} = \langle \boldsymbol{M}_t^{(i)}, \boldsymbol{M}_t^{(j)} \rangle_{\mathrm F},
\end{equation}
whose $(i,j)$-th entry measures the Frobenius inner product between the $i$-th and $j$-th momentum slices. This matrix encodes the \textbf{pairwise coupling} among layers, that is large off-diagonal entries indicate strongly correlated update directions, while near-zero entries indicate decoupled evolution. The goal of $\boldsymbol{Q}$ is to  \textbf{diagonalize} this coupling structure via a unitary rotation, so that each transformed slice evolves independently in its own orthogonal subspace.

Under the transform \(\widehat{\mathsf{M}}_t = \boldsymbol{Q} \mathsf{M}_t\), the Gram matrix of the transformed slices becomes
\begin{equation}
	\widehat{\mathsf{C}}_t
	= \widehat{\mathsf{M}}_t \widehat{\mathsf{M}}_t^{\mathsf T}
	= \boldsymbol{Q} \mathsf{M}_t \mathsf{M}_t^{\mathsf T} \boldsymbol{Q}^{\mathsf T}
	= U_3^{\mathsf T} \bigl( U_3 \Sigma_3^2 U_3^{\mathsf T} \bigr) U_3
	= \Sigma_3^2,
\end{equation}
which is diagonal. Hence, for any \(i \neq j\),
\begin{equation}
	\langle \widehat{\boldsymbol{M}}_{t,i}, \widehat{\boldsymbol{M}}_{t,j} \rangle_{\mathrm F}
	= (\widehat{\mathsf{C}}_t)_{ij} = 0.
\end{equation}
The transformed slices \(\widehat{\boldsymbol{M}}_{t,1}, \dots, \widehat{\boldsymbol{M}}_{t,K}\) are therefore mutually orthogonal in the Frobenius inner product sense.

This orthogonality has a fundamental algorithmic consequence. The polarization step computes \(\widehat{\boldsymbol{P}}_{t,j} = \operatorname{polar}(\widehat{\boldsymbol{M}}_{t,j})\), which preserves the column and row subspaces of each slice while discarding its magnitude. Since the slices are mutually orthogonal, their polar factors remain confined to their respective orthogonal subspaces:
\begin{equation}
	\langle \widehat{\boldsymbol{P}}_{t,i}, \widehat{\boldsymbol{P}}_{t,j} \rangle_{\mathrm F} = 0, \quad i \neq j.
\end{equation}
Because the inverse transform \(\boldsymbol{\mathcal{P}}_t = \widehat{\boldsymbol{\mathcal{P}}}_t \times_3 \boldsymbol{Q}^{\mathsf H}\) is a unitary rotation (and thus preserves inner products), this orthogonality carries over to the original tensor domain:
\begin{equation}
	\langle \boldsymbol{\mathcal{P}}_t(:,:,i), \boldsymbol{\mathcal{P}}_t(:,:,j) \rangle_{\mathrm F} = 0, \quad i \neq j.
\end{equation}

Consequently, the update direction for slice \(i\) resides entirely in a subspace orthogonal to that of slice \(j\). Any modification to slice \(i\)'s update—whether scaling, clipping, or otherwise adjusting its magnitude—produces no projection onto slice \(j\)'s subspace. In other words, adjusting slice \(i\) does not affect slice \(j\). This geometric decoupling effectively transforms the original coupled multi-layer optimization problem into \(K\) independent single-channel subproblems, each evolving in its own orthogonal subspace without mutual interference.

\section{Polarization via Newton--Schulz iteration}\label{Newton--Schulz-polynomial}

We now describe the practical polarization step that replaces the expensive singular value decomposition. Instead of computing the exact polar factor via SVD, we use a fast iterative approximation based on the Newton--Schulz polynomial.

For each transformed slice $\widehat{\boldsymbol{M}}_{t,j} \in \mathbb{R}^{m \times n}$, we first ensure that the matrix is in a shape suitable for efficient iteration. If $m > n$, we transpose the slice so that the number of rows does not exceed the number of columns; this keeps the Gram matrix $\boldsymbol{X} \boldsymbol{X}^{\mathsf T}$ as small as possible and reduces computational cost. After optionally transposing, we initialize
\begin{equation}
	\boldsymbol{X}_{0,j} = \frac{\widehat{\boldsymbol{M}}_{t,j}}{\| \widehat{\boldsymbol{M}}_{t,j} \|_{\mathrm F} + \epsilon},
\end{equation}
where $\epsilon > 0$ is a small constant to prevent division by zero. This initialization scales the slice to unit Frobenius norm, which is a stable starting point for the iteration.

We then apply $J$ steps of the quintic Newton--Schulz iteration. At each step $\ell = 0, 1, \dots, J-1$, we compute
\begin{equation}
	\boldsymbol{A}_{\ell,j} = \boldsymbol{X}_{\ell,j} \boldsymbol{X}_{\ell,j}^{\mathsf H},
\end{equation}
and update
\begin{equation}
	\boldsymbol{X}_{\ell+1,j}
	= a \boldsymbol{X}_{\ell,j}
	+ \left( b \boldsymbol{A}_{\ell,j} + c \boldsymbol{A}_{\ell,j}^2 \right) \boldsymbol{X}_{\ell,j},
	\label{eq:ns}
\end{equation}
where $(a, b, c) = (3.4445, -4.7750, 2.0315)$ are the precomputed coefficients from the PolarExpress scheme \citep{amsel2026polar}. These coefficients are designed to flatten the singular-value spectrum rapidly: the iteration drives all nonzero singular values toward one within a small number of steps. The output $\boldsymbol{X}_{J,j}$ is an approximation to the exact polar factor $\operatorname{polar}(\widehat{\boldsymbol{M}}_{t,j})$, though it is not exact due to the finite iteration count and floating-point rounding \citep{kim2026newton}. In practice, $J = 4$ or $5$ suffices for most applications.

\section{Stack scaling and weight update}

After obtaining the approximate polar factors $\boldsymbol{\mathcal{P}}_t$ in the original tensor domain via inverse transformation, we apply a single shared scaling factor to the entire stack. This step controls the overall step magnitude while preserving the direction of the update.

Let $\operatorname{RMS}(\boldsymbol{\mathcal{P}}_t)$ denote the root-mean-square value of all entries in the tensor $\boldsymbol{\mathcal{P}}_t$:
\begin{equation}
	\operatorname{RMS}(\boldsymbol{\mathcal{P}}_t)
	= \sqrt{ \frac{1}{mnK} \sum_{i=1}^m \sum_{j=1}^n \sum_{k=1}^K |\boldsymbol{\mathcal{P}}_t(i,j,k)|^2 }.
\end{equation}
We define the scaling factor
\begin{equation}
	\gamma_t = \frac{s}{\operatorname{RMS}(\boldsymbol{\mathcal{P}}_t) + \epsilon},
	\label{eq:scale}
\end{equation}
where $s > 0$ is a target RMS value (typically set to a constant such as $1 / \sqrt{\max(m,n)}$ as suggested by the rank--energy identity in Proposition~\ref{prop:rank-energy}), and $\epsilon$ is a small constant to avoid division by zero. The scaled update is then
\begin{equation}
	\widetilde{\boldsymbol{\mathcal{P}}}_t = \gamma_t \, \boldsymbol{\mathcal{P}}_t.
	\label{eq:scaled-update}
\end{equation}
Because $\gamma_t$ is shared across all layers, this scaling changes the effective trust-region radius but does not alter the relative direction of the update across slices. This is a crucial design choice: it preserves the geometric structure of the polar direction while allowing the overall step size to be tuned by a single global parameter.

Finally, we apply decoupled weight decay and the scaled update to obtain the new parameters:
\begin{equation}
	\boldsymbol{\mathcal{W}}_{t+1}
	= (1 - \eta_t \lambda) \, \boldsymbol{\mathcal{W}}_t
	- \eta_t \, \widetilde{\boldsymbol{\mathcal{P}}}_t,
	\label{eq:weight-update}
\end{equation}
where $\eta_t$ is the learning rate and $\lambda$ is the weight-decay coefficient. This update rule combines the momentum-based polar direction with standard $L_2$ regularization.

We note that one could also apply per-slice scaling factors, for example to preserve the Frobenius energy of each individual slice. Such variants can be useful in certain engineering contexts, but they correspond to a different weighted geometry and are not part of the main algorithm presented here. We therefore treat them as separate ablations and focus on the shared-scaling scheme throughout this work.

\subsection{Computational complexity analysis}\label{complexity-appendix}
We analyze the per-iteration computational cost of the proposed optimizer for a single tensor group consisting of $K$ matrices of size $m\times n$.  	
The main computational burden lies in the slice-wise Newton--Schulz polar iterations. For each of the $K$ slices, each iteration involves matrix multiplications of an $m\times n$ matrix with its transpose, incurring $O(mn\cdot\min(m,n))$ operations. With a fixed number of iterations $J$ (typically $J=5$ in our implementation), the total cost of polarization is
$O(JKmn\cdot\min(m,n))$, which is $O(Kmn\cdot\min(m,n))$ since $J$ is a constant.
In addition to polarization, the tensor transform $\cM_t\times_3 \boldsymbol{Q}$ and its inverse each require $O(K^2 mn)$ operations. When $\boldsymbol{Q}$ is updated adaptively from the momentum tensor, an additional SVD on the $K\times mn$ unfolding matrix is performed every $I$ (e.g., 100) steps, costing $O(K^2 mn)$ per update (assuming $K\ll mn$), which amortizes to $O(K^2 mn/I)$ per iteration. Thus the total complexity per iteration is approximately $K\,mn\cdot\min(m,n) + K^2 mn$. Since $K\ll m$ and $K\ll  n$, the polarization step therefore dominates the overall cost, and the total complexity per iteration is approximately $
O\!\left(K\,mn\cdot\min(m,n) \right)$.

Muon incurs $O(Kmn\cdot\min(m,n))$ cost, while TEON requires $O(m(nK)\cdot \min(m,nK))$. Our method has a cost of $O\!\left(Kmn\cdot\min(m,n)\right)$. In the typical regime where $K\ll \min(m,n)$, these complexities are  comparable. Moreover, this overhead remains negligible relative to the overall pretraining cost. Thus, TUON achieves its cross-layer structural benefits at a modest computational premium, scaling gracefully with the group size $K$. 
Table~\ref{tab:complexity} compares the resulting complexities.

\section{Detailed proofs for Section~\ref{sec:theory}}
\label{app:proofs}

\subsection{Proof of Proposition~\ref{prop:q-dual}: duality of the $Q$-block norms}\label{app:proofs1}

Before solving a norm-constrained steepest-descent problem, we must know the dual norm. It measures the largest possible first-order alignment with a direction in the unit trust region.  Proposition~\ref{prop:q-dual} identifies the measure naturally associated with TUON.

\textbf{Proposition~\ref{prop:q-dual}}[Dual norm] \textit{
	For unitary ${\boldsymbol{Q}}$, $\|\cdot\|_{{\boldsymbol{Q}},\mathrm{Tn}}$ is the dual norm of $\|\cdot\|_{{\boldsymbol{Q}},\mathrm{Ts}}$:
	\begin{equation}
		\sup_{\|\boldsymbol{\mathcal{D}}\|_{{\boldsymbol{Q}},\mathrm{Ts}}\le1}
		\langle\boldsymbol{\mathcal{G}},\boldsymbol{\mathcal{D}}\rangle_{\mathrm F}
		=\|\boldsymbol{\mathcal{G}}\|_{{\boldsymbol{Q}},\mathrm{Tn}}.
		\label{eq:q-duality}
	\end{equation}
}

\begin{proof}
	We prove the duality by establishing both the upper bound and its attainability. The argument proceeds in three steps.
	
	First, we move the inner product to the transform domain. 
	By the Parseval identity for unitary mode-$\boldsymbol{Q}$ transforms (Lemma~\ref{lem:unitary-parseval}), the Frobenius inner product is preserved under the layer-mode transformation:
	\begin{equation}
		\langle\boldsymbol{\mathcal{G}},\boldsymbol{\mathcal{D}}\rangle_{\mathrm F}
		=\langle\widehat{\boldsymbol{\mathcal{G}}},\widehat{\boldsymbol{\mathcal{D}}}\rangle_{\mathrm F}
		=\sum_{j=1}^K\langle\widehat{\boldsymbol{G}}_j,\widehat{\boldsymbol{D}}_j\rangle_{\mathrm F},
		\label{eq:proof-dual-step1}
	\end{equation}
	where $\widehat{\boldsymbol{\mathcal{G}}} = \boldsymbol{\mathcal{G}} \times_3 \boldsymbol{Q}$ and $\widehat{\boldsymbol{\mathcal{D}}} = \boldsymbol{\mathcal{D}} \times_3 \boldsymbol{Q}$ denote the transformed tensors, and $\widehat{\boldsymbol{G}}_j, \widehat{\boldsymbol{D}}_j \in \mathbb{C}^{m \times n}$ are their respective $j$-th frontal slices. Thus the tensor inner product decouples into a sum of $K$ independent matrix inner products.
	
	Second, we have the universal upper bound.  
	For any feasible direction satisfying $\|\boldsymbol{\mathcal{D}}\|_{\boldsymbol{Q},\mathrm{Ts}}\le1$, the definition of the tensor spectral norm under $\boldsymbol{Q}$-transform implies
	\begin{equation}
		\|\widehat{\boldsymbol{D}}_j\|_{\mathrm{op}}\le1
		\qquad\text{for all }j=1,\dots,K.
	\end{equation}
	Applying the matrix H\"older inequality (Lemma~\ref{lem:matrix-holder}) slice-wise yields
	\begin{equation}
		\langle\widehat{\boldsymbol{G}}_j,\widehat{\boldsymbol{D}}_j\rangle_{\mathrm F}
		\le\|\widehat{\boldsymbol{G}}_j\|_*\|\widehat{\boldsymbol{D}}_j\|_{\mathrm{op}}
		\le\|\widehat{\boldsymbol{G}}_j\|_*.
	\end{equation}
	Summing this inequality over all slices gives
	\begin{equation}
		\langle\boldsymbol{\mathcal{G}},\boldsymbol{\mathcal{D}}\rangle_{\mathrm F}
		\le\sum_{j=1}^K\|\widehat{\boldsymbol{G}}_j\|_*
		=\|\boldsymbol{\mathcal{G}}\|_{\boldsymbol{Q},\mathrm{Tn}}.
		\label{eq:proof-dual-upper}
	\end{equation}
	Since this bound holds for every feasible $\boldsymbol{\mathcal{D}}$, we conclude that the supremum is at most $\|\boldsymbol{\mathcal{G}}\|_{\boldsymbol{Q},\mathrm{Tn}}$.
	
	At last, we construct the equality case to show that the upper bound is tight. For each nonzero transformed slice, take its compact singular value decomposition
	\begin{equation}
		\widehat{\boldsymbol{G}}_j = \boldsymbol{U}_j \boldsymbol{\Sigma}_j \boldsymbol{V}_j^{\mathsf H},
	\end{equation}
	and define its canonical polar factor
	\begin{equation}
		\widehat{\boldsymbol{D}}_j^* = \boldsymbol{U}_j \boldsymbol{V}_j^{\mathsf H} = \operatorname{polar}(\widehat{\boldsymbol{G}}_j).
	\end{equation}
	For zero slices, set $\widehat{\boldsymbol{D}}_j^* = \boldsymbol{0}$. Lemma~\ref{lem:polar-identities} guarantees that each polar factor satisfies $\|\widehat{\boldsymbol{D}}_j^*\|_{\mathrm{op}}\le1$, so the assembled tensor $\widehat{\boldsymbol{\mathcal{D}}}^*$ is feasible under the tensor spectral norm under $\boldsymbol{Q}$-transform. Moreover, the polar alignment identity gives
	\begin{equation}
		\langle\widehat{\boldsymbol{G}}_j,\widehat{\boldsymbol{D}}_j^*\rangle_{\mathrm F}
		= \|\widehat{\boldsymbol{G}}_j\|_*.
	\end{equation}
	Summing over $j$ and applying the inverse unitary transform (which preserves the inner product) yields equality in ~\eqref{eq:proof-dual-upper}. Therefore the upper bound is attained, and the supremum is exactly $\|\boldsymbol{\mathcal{G}}\|_{\boldsymbol{Q},\mathrm{Tn}}$.
\end{proof}

By applying the transform $\boldsymbol{Q}$, the original tensor norm decouples into independent matrix norms across the $K$ slices. Therefore, the natural measure of the gradient's descent potential is no longer its Frobenius norm, but the sum of nuclear norms of the transformed slices. This decoupling serves as the basic tool behind the exact oracle, the comparison with Muon and TEON, and the convergence.

\subsection{Proof of Theorem~\ref{thm:q-lmo}: exact normalized steepest descent}\label{app:proofs2}

\textbf{Theorem}~\ref{thm:q-lmo} [Exact LMO and normalized steepest descent]
\textit{
	For every radius $\rho > 0$,  
		\begin{equation}
				-\rho \boldsymbol{\mathcal{P}}_{\boldsymbol{Q}}(\boldsymbol{\mathcal{M}})
				\in
				\argmin_{\| \boldsymbol{\mathcal{D}} \|_{\boldsymbol{Q},\mathrm{Ts}} \le \rho}
				\langle \boldsymbol{\mathcal{M}}, \boldsymbol{\mathcal{D}} \rangle_{\mathrm F},
				\label{eq:q-lmo}
			\end{equation}
	and the optimal value is $
	-\rho \| \boldsymbol{\mathcal{M}} \|_{\boldsymbol{Q},\mathrm{Tn}}
	= -\rho \sum_{j=1}^K \| \widehat{\boldsymbol{M}}_j \|_*.
	$
}

\begin{proof}
	For any \(\boldsymbol{\mathcal{D}}\) with \(\|\boldsymbol{\mathcal{D}}\|_{\boldsymbol{Q},\mathrm{Ts}} \le \rho\), the dual H\"older inequality from Proposition~\ref{prop:q-dual} gives
	\begin{equation}
		|\langle \boldsymbol{\mathcal{M}}, \boldsymbol{\mathcal{D}} \rangle_{\mathrm F}|
		\le \|\boldsymbol{\mathcal{M}}\|_{\boldsymbol{Q},\mathrm{Tn}} \|\boldsymbol{\mathcal{D}}\|_{\boldsymbol{Q},\mathrm{Ts}}
		\le \rho \|\boldsymbol{\mathcal{M}}\|_{\boldsymbol{Q},\mathrm{Tn}}.
	\end{equation}
	Hence
	\begin{equation}
		\langle \boldsymbol{\mathcal{M}}, \boldsymbol{\mathcal{D}} \rangle_{\mathrm F}
		\ge -\rho \|\boldsymbol{\mathcal{M}}\|_{\boldsymbol{Q},\mathrm{Tn}}.
		\label{eq:lmo-universal-lower}
	\end{equation}
	No feasible direction can have a smaller linearized objective value.
	
	In transformed coordinates, set
	\begin{equation}
		\widehat{\boldsymbol{D}}_j^*
		:= -\rho \operatorname{polar}(\widehat{\boldsymbol{M}}_j).
	\end{equation}
	Each nonzero polar factor has spectral norm one, so
	\begin{equation}
		\|\boldsymbol{\mathcal{D}}^*\|_{\boldsymbol{Q},\mathrm{Ts}}
		= \max_j \|\widehat{\boldsymbol{D}}_j^*\|_{\mathrm{op}}
		\le \rho.
	\end{equation}
	Inverse transforming the slices gives exactly
	\begin{equation}
		\boldsymbol{\mathcal{D}}^*
		= -\rho \boldsymbol{\mathcal{P}}_{\boldsymbol{Q}}(\boldsymbol{\mathcal{M}}).
	\end{equation}
	
	Using Parseval and Lemma~\ref{lem:polar-identities},
	\begin{align}
		\langle \boldsymbol{\mathcal{M}}, \boldsymbol{\mathcal{D}}^* \rangle_{\mathrm F}
		&= -\rho \sum_{j=1}^K
		\langle \widehat{\boldsymbol{M}}_j, \operatorname{polar}(\widehat{\boldsymbol{M}}_j) \rangle_{\mathrm F} \\
		&= -\rho \sum_{j=1}^K \|\widehat{\boldsymbol{M}}_j\|_* \\
		&= -\rho \|\boldsymbol{\mathcal{M}}\|_{\boldsymbol{Q},\mathrm{Tn}}.
	\end{align}
	This equals the universal lower bound in~\eqref{eq:lmo-universal-lower}, so the proposed direction is globally optimal.
\end{proof}

\subsection{Proof of Proposition~\ref{prop:q-tpolar}: tensor-polar equivalence}\label{prop:q-tpolar-proof-appendix}

The  update of our method also admits a natural interpretation in the language of tensor algebra. 
Let
\begin{equation}
	\boldsymbol{\mathcal{M}}
	= \boldsymbol{\mathcal{U}} *_{\boldsymbol{Q}} \boldsymbol{\mathcal{S}} *_{\boldsymbol{Q}} \boldsymbol{\mathcal{V}}^{\top}
	\label{eq:q-tsvd}
\end{equation}
be the  T-SVD~\citep{kernfeld2015tensor,kilmer2011factorization,zheng2026refta} of $\boldsymbol{\mathcal{M}}$ (see Appendix~\ref{app:qproduct} for definitions), where $\boldsymbol{\mathcal{S}}$ is an $\boldsymbol{f}$-diagonal tensor whose diagonal entries in the transform domain are the nonzero singular values of the corresponding transformed slices. 
Define $\boldsymbol{\mathcal{J}}_{\mathrm{supp}}$ to be the $\boldsymbol{f}$-diagonal tensor obtained from $\boldsymbol{\mathcal{S}}$ by replacing every nonzero diagonal entry with $1$, while leaving zeros unchanged. In other words, $\boldsymbol{\mathcal{J}}_{\mathrm{supp}}$ is the indicator of the support of the singular values in the transform domain.
The polar factor of a matrix $\boldsymbol{A} = \boldsymbol{U} \boldsymbol{\Sigma} \boldsymbol{V}^{\top}$ is $\boldsymbol{U} \boldsymbol{V}^{\top}$, which replaces the singular values $\boldsymbol{\Sigma}$ by the identity on the support. Analogously, the generalized tensor polar factor of $\boldsymbol{\mathcal{M}}$ is obtained by replacing the singular-value tensor $\boldsymbol{\mathcal{S}}$ with its support indicator $\boldsymbol{\mathcal{J}}_{\mathrm{supp}}$:
\begin{equation}
\operatorname{t-polar}_{\boldsymbol{Q}}(\boldsymbol{\mathcal{M}})
:= \boldsymbol{\mathcal{U}} *_{\boldsymbol{Q}} \boldsymbol{\mathcal{J}}_{\mathrm{supp}} *_{\boldsymbol{Q}} \boldsymbol{\mathcal{V}}^{\top}.
\end{equation}
The following proposition establishes that this tensor-algebra object is exactly the optimizer map from Theorem~\ref{thm:q-lmo}.

\textbf{Proposition~\ref{prop:q-tpolar}}[$\boldsymbol{Q}$-tensor polar equivalence]
\textit{
	The optimizer map equals the partial-isometry factor of the generalized T-SVD:
	\begin{equation}
		\boldsymbol{\mathcal{P}}_{\boldsymbol{Q}}(\boldsymbol{\mathcal{M}})
		= \operatorname{tpolar}_{\boldsymbol{Q}}(\boldsymbol{\mathcal{M}})
		= \boldsymbol{\mathcal{U}} *_{\boldsymbol{Q}} \boldsymbol{\mathcal{J}}_{\mathrm{supp}} *_{\boldsymbol{Q}} \boldsymbol{\mathcal{V}}^{\top}.
		\label{eq:q-tpolar}
	\end{equation}
	For $K = 1$, this reduces to matrix Muon.  For $\boldsymbol{Q} = \boldsymbol{I}$, it reduces to independent layer-wise polar maps.
}

\begin{proof}
	We prove the equivalence by working in the transform domain, where the Tensor SVD decouples into independent matrix SVDs.
	
	First, we make slice-wise decomposition in the transform domain. 
	By the definition of the generalized T-SVD in Appendix~\ref{app:qproduct}, ~\eqref{eq:q-tsvd} is equivalent to a collection of $K$ independent matrix singular value decompositions applied to each transformed slice:
	\begin{equation}
		\widehat{\boldsymbol{M}}_j
		= \widehat{\boldsymbol{U}}_j \widehat{\boldsymbol{S}}_j \widehat{\boldsymbol{V}}_j^{\mathsf H},
		\qquad j = 1, \dots, K,
	\end{equation}
	where $\widehat{\boldsymbol{S}}_j = \operatorname{diag}(\sigma_{j,1}, \ldots, \sigma_{j,r_j}) \in \mathbb{R}^{r_j \times r_j}$ is a diagonal matrix containing the nonzero singular values of the $j$-th transformed slice, and $r_j = \operatorname{rank}(\widehat{\boldsymbol{M}}_j)$.
	
	Then we replace singular values by their support indicators. 
	For each $j$, define $\widehat{\boldsymbol{J}}_{\mathrm{supp},j} \in \mathbb{R}^{r_j \times r_j}$ to be the diagonal matrix
	\begin{equation}
		\widehat{\boldsymbol{J}}_{\mathrm{supp},j}
		:= \operatorname{diag}(\mathbf{1}_{\{\sigma_{j,1} > 0\}}, \ldots, \mathbf{1}_{\{\sigma_{j,r_j} > 0\}})
		= \boldsymbol{I}_{r_j},
	\end{equation}
	which is simply the identity matrix of size $r_j$, since all diagonal entries of $\widehat{\boldsymbol{S}}_j$ are positive by construction. In the full $m \times n$ matrix, this corresponds to a diagonal matrix with ones on the support of the nonzero singular values and zeros elsewhere. Then
	\begin{equation}
		\widehat{\boldsymbol{U}}_j \widehat{\boldsymbol{J}}_{\mathrm{supp},j} \widehat{\boldsymbol{V}}_j^{\mathsf H}
		= \widehat{\boldsymbol{U}}_j \widehat{\boldsymbol{V}}_j^{\mathsf H}
		= \operatorname{polar}(\widehat{\boldsymbol{M}}_j).
	\end{equation}
	This is precisely the canonical polar factor of the $j$-th transformed slice: it retains the left and right singular vectors while discarding the singular values.
	
	Thus, by reassembling the tensor, the transformed frontal slices of the tensor
	\begin{equation}
		\boldsymbol{\mathcal{U}} *_{\boldsymbol{Q}} \boldsymbol{\mathcal{J}}_{\mathrm{supp}} *_{\boldsymbol{Q}} \boldsymbol{\mathcal{V}}^{\top}
	\end{equation}
	are exactly the slice-wise polar factors constructed in Step 2. Applying the inverse transform $\boldsymbol{Q}^{\mathsf H}$ along the third mode reconstructs the original tensor domain. By the definition of $\boldsymbol{\mathcal{P}}_{\boldsymbol{Q}}(\boldsymbol{\mathcal{M}})$ in Section~\ref{subsec:exact-geometry}, this reconstructed tensor is precisely $\boldsymbol{\mathcal{P}}_{\boldsymbol{Q}}(\boldsymbol{\mathcal{M}})$. Therefore,
	\begin{equation}
		\boldsymbol{\mathcal{P}}_{\boldsymbol{Q}}(\boldsymbol{\mathcal{M}})
		= \boldsymbol{\mathcal{U}} *_{\boldsymbol{Q}} \boldsymbol{\mathcal{J}}_{\mathrm{supp}} *_{\boldsymbol{Q}} \boldsymbol{\mathcal{V}}^{\top},
	\end{equation}
	which proves the proposition.
\end{proof}

\subsection{Proof of Theorem~\ref{thm:nonredundancy}: cancellation and non-cancellation}\label{app:proofs4}

\textbf{Theorem~\ref{thm:nonredundancy}}[unfolding cancellation versus slice-wise non-cancellation]
\textit{
	Let $\boldsymbol{M}_{(1)} = \mathscr U(\boldsymbol{\mathcal{M}})$ and $\boldsymbol{R}_{\boldsymbol{Q}} = \boldsymbol{Q}^{\mathsf T} \otimes \boldsymbol{I}_n$.  For the canonical matrix polar map,
	\begin{equation}
		\operatorname{polar}(\boldsymbol{M}_{(1)} \boldsymbol{R}_{\boldsymbol{Q}}) \boldsymbol{R}_{\boldsymbol{Q}}^{\mathsf H}
		= \operatorname{polar}(\boldsymbol{M}_{(1)}).
		\label{eq:global-cancellation}
	\end{equation}
	Hence a unitary layer transform is algebraically redundant around a single global unfolding polar.  In contrast,
	\begin{equation}
		\bigl( \operatorname{polar}_{\mathrm{slice}}(\boldsymbol{\mathcal{M}} \times_3 \boldsymbol{Q}) \bigr) \times_3 \boldsymbol{Q}^{\mathsf H}
		\neq
		\operatorname{polar}_{\mathrm{slice}}(\boldsymbol{\mathcal{M}})
		\label{eq:slice-noncancellation}
	\end{equation}
	in general.
}

\begin{proof}
	We prove the two claims separately.
	First, we will show that unfolding polar cancels the transform. 
	Let
	\begin{equation}
		\boldsymbol{M}_{(1)} = \boldsymbol{U} \boldsymbol{\Sigma} \boldsymbol{V}^{\mathsf H}
	\end{equation}
	be a compact singular value decomposition of the unfolded matrix.  Because $\boldsymbol{R}_{\boldsymbol{Q}}$ is unitary, we have
	\begin{equation}
		\boldsymbol{M}_{(1)} \boldsymbol{R}_{\boldsymbol{Q}}
		= \boldsymbol{U} \boldsymbol{\Sigma} \boldsymbol{V}^{\mathsf H} \boldsymbol{R}_{\boldsymbol{Q}}
		= \boldsymbol{U} \boldsymbol{\Sigma} (\boldsymbol{R}_{\boldsymbol{Q}}^{\mathsf H} \boldsymbol{V})^{\mathsf H}.
	\end{equation}
	The columns of $\boldsymbol{R}_{\boldsymbol{Q}}^{\mathsf H} \boldsymbol{V}$ remain orthonormal, so this is another compact SVD with the same left singular vectors $\boldsymbol{U}$ and the same singular values $\boldsymbol{\Sigma}$.  Therefore,
	\begin{equation}
		\operatorname{polar}(\boldsymbol{M}_{(1)} \boldsymbol{R}_{\boldsymbol{Q}})
		= \boldsymbol{U} (\boldsymbol{R}_{\boldsymbol{Q}}^{\mathsf H} \boldsymbol{V})^{\mathsf H}
		= \boldsymbol{U} \boldsymbol{V}^{\mathsf H} \boldsymbol{R}_{\boldsymbol{Q}}
		= \operatorname{polar}(\boldsymbol{M}_{(1)}) \boldsymbol{R}_{\boldsymbol{Q}}.
	\end{equation}
	Multiplying both sides on the right by $\boldsymbol{R}_{\boldsymbol{Q}}^{\mathsf H}$ proves ~\eqref{eq:global-cancellation}.
	
	Second, we give a concrete slice-wise counterexample. 
	It suffices to construct this single instance where equality fails. Consider two scalar slices
	\begin{equation}
		\boldsymbol{M}_1 = 2, \qquad \boldsymbol{M}_2 = 1,
	\end{equation}
	and the two-point orthonormal DCT/Hadamard matrix
	\begin{equation}
		\boldsymbol{Q} = \frac{1}{\sqrt{2}}
		\begin{bmatrix} 1 & 1 \\ 1 & -1 \end{bmatrix}.
	\end{equation}
	
	If we apply scalar polar maps directly to the original slices, we obtain $(1, 1)^{\mathsf T}$, since the polar factor of any positive scalar is $1$.
	
	In the transformed path, we first mix the slices:
	\begin{equation}
		\boldsymbol{Q} \begin{bmatrix} 2 \\ 1 \end{bmatrix}
		= \frac{1}{\sqrt{2}} \begin{bmatrix} 3 \\ 1 \end{bmatrix}.
	\end{equation}
	Both entries are positive, so their scalar polar factors are again $(1, 1)^{\mathsf T}$. Applying the inverse transform gives
	\begin{equation}
		\boldsymbol{Q}^{\mathsf T} \begin{bmatrix} 1 \\ 1 \end{bmatrix}
		= \begin{bmatrix} \sqrt{2} \\ 0 \end{bmatrix},
	\end{equation}
	which differs from the original $(1, 1)^{\mathsf T}$. Therefore, slice-wise polarization and layer mixing do not commute in general.
\end{proof}

\subsection{Proof of Proposition~\ref{prop:rank-energy}: exact update energy}\label{app:proofs5}

\textbf{Proposition~\ref{prop:rank-energy}}[Rank--energy identity]
\textit{
	Let $r_j := \operatorname{rank}(\widehat{\boldsymbol{M}}_j)$.  Then
	\begin{equation}
		\| \boldsymbol{\mathcal{P}}_{\boldsymbol{Q}}(\boldsymbol{\mathcal{M}}) \|_{\mathrm F}^2
		= \sum_{j=1}^K r_j,
		\qquad
		\operatorname{RMS}(\boldsymbol{\mathcal{P}}_{\boldsymbol{Q}}(\boldsymbol{\mathcal{M}}))
		= \sqrt{ \frac{\sum_j r_j}{mnK} }.
	\end{equation}
	If every transformed slice is full rank, then
	\begin{equation}
		\operatorname{RMS}(\boldsymbol{\mathcal{P}}_{\boldsymbol{Q}})
		= \frac{1}{\sqrt{\max(m,n)}},
	\end{equation}
	independently of $K$.
}

\begin{proof}
	We prove the rank--energy identity by analyzing the Frobenius norm of the polar factor in the transform domain and then transferring the result back to the original domain via unitary invariance.
	
	For each transformed slice $\widehat{\boldsymbol{M}}_j \in \mathbb{C}^{m \times n}$, Lemma~\ref{lem:polar-identities} states that the canonical polar factor $\operatorname{polar}(\widehat{\boldsymbol{M}}_j)$ has all nonzero singular values equal to one. Consequently, its squared Frobenius norm equals the number of nonzero singular values, which is precisely its rank:
	\begin{equation}
		\| \operatorname{polar}(\widehat{\boldsymbol{M}}_j) \|_{\mathrm F}^2
		= \operatorname{rank}(\widehat{\boldsymbol{M}}_j)
		=: r_j.
		\label{eq:proof-rank-energy-slice}
	\end{equation}
	This holds for each slice independently.
	
	The transformed polar tensor $\widehat{\boldsymbol{\mathcal{P}}} := \boldsymbol{\mathcal{P}}_{\boldsymbol{Q}}(\boldsymbol{\mathcal{M}}) \times_3 \boldsymbol{Q}$ has frontal slices $\widehat{\boldsymbol{P}}_j = \operatorname{polar}(\widehat{\boldsymbol{M}}_j)$. Since the Frobenius norm of a tensor is the square root of the sum of the squared Frobenius norms of its frontal slices, we have
	\begin{equation}
		\| \widehat{\boldsymbol{\mathcal{P}}} \|_{\mathrm F}^2
		= \sum_{j=1}^K \| \widehat{\boldsymbol{P}}_j \|_{\mathrm F}^2
		= \sum_{j=1}^K r_j.
		\label{eq:proof-rank-energy-transformed}
	\end{equation}
	
	By definition, $\boldsymbol{\mathcal{P}}_{\boldsymbol{Q}}(\boldsymbol{\mathcal{M}}) = \widehat{\boldsymbol{\mathcal{P}}} \times_3 \boldsymbol{Q}^{\mathsf H}$. Since $\boldsymbol{Q}$ is unitary, Lemma~\ref{lem:unitary-parseval} guarantees that the mode-3 unitary transform preserves the Frobenius norm:
	\begin{equation}
		\| \boldsymbol{\mathcal{P}}_{\boldsymbol{Q}}(\boldsymbol{\mathcal{M}}) \|_{\mathrm F}
		= \| \widehat{\boldsymbol{\mathcal{P}}} \|_{\mathrm F}.
	\end{equation}
	Combining this with ~\eqref{eq:proof-rank-energy-transformed} gives
	\begin{equation}
		\| \boldsymbol{\mathcal{P}}_{\boldsymbol{Q}}(\boldsymbol{\mathcal{M}}) \|_{\mathrm F}^2
		= \sum_{j=1}^K r_j.
		\label{eq:proof-rank-energy-final}
	\end{equation}
	This establishes the first part of the proposition.
	
	The root-mean-square (RMS) value of a tensor is defined as the Frobenius norm divided by the square root of the total number of entries. Since $\boldsymbol{\mathcal{P}}_{\boldsymbol{Q}}(\boldsymbol{\mathcal{M}}) \in \mathbb{C}^{m \times n \times K}$ has $mnK$ entries, we have
	\begin{equation}
		\operatorname{RMS}\bigl(\boldsymbol{\mathcal{P}}_{\boldsymbol{Q}}(\boldsymbol{\mathcal{M}})\bigr)^2
		= \frac{1}{mnK} \| \boldsymbol{\mathcal{P}}_{\boldsymbol{Q}}(\boldsymbol{\mathcal{M}}) \|_{\mathrm F}^2
		= \frac{1}{mnK} \sum_{j=1}^K r_j.
	\end{equation}
	Taking square roots yields the RMS formula in the proposition.
	
	If every transformed slice is full rank, then each slice has rank $r_j = r := \min(m,n)$, since a matrix in $\mathbb{C}^{m \times n}$ can have rank at most $\min(m,n)$. Substituting $r_j = r$ into the RMS formula gives
	\begin{equation}
		\operatorname{RMS}\bigl(\boldsymbol{\mathcal{P}}_{\boldsymbol{Q}}(\boldsymbol{\mathcal{M}})\bigr)^2
		= \frac{K r}{mnK}
		= \frac{1}{\max(m,n)}.
	\end{equation}
	Here we used the identity $mn / r = \max(m,n)$, which follows from $r = \min(m,n)$. Therefore,
	\begin{equation}
		\operatorname{RMS}\bigl(\boldsymbol{\mathcal{P}}_{\boldsymbol{Q}}(\boldsymbol{\mathcal{M}})\bigr)
		= \frac{1}{\sqrt{\max(m,n)}},
	\end{equation}
	which is independent of $K$. This completes the proof.
\end{proof}

\subsection{Non-unitary transforms}

Then we show the necessity of the invertibility of transform $\boldsymbol{Q}$.  Non-unitary coordinates do not preserve the Frobenius pairing, so it cannot simply use the proposed pipeline.

\begin{theorem}
	\label{thm:inverse-adjoint}
	For a non-unitary transform $\boldsymbol{T}$. 
	Let $\mathcal L_{\boldsymbol{T}}(\boldsymbol{\mathcal{D}}):=\boldsymbol{\mathcal{D}}\times_3\boldsymbol{T}$ and
	$
	\|\boldsymbol{\mathcal{D}}\|_{\boldsymbol{T},\mathrm{Ts}}
	:=\max_j\|\bigl(\mathcal L_{\boldsymbol{T}}\boldsymbol{\mathcal{D}}\bigr)_j\|_{\mathrm{op}}.
	$
	An exact radius-$\rho$ LMO is 
	$ 
	\boldsymbol{\mathcal{D}}^*
	=-\rho\,\mathcal L_{\boldsymbol{T}}^{-1}\!\left[
	\operatorname{polar}_{\mathrm{slice}}\!\left(\mathcal L_{\boldsymbol{T}}^{-*}\boldsymbol{\mathcal{G}}\right)
	\right],
	$ where $
	\mathcal L_{\boldsymbol{T}}^{-*}\boldsymbol{\mathcal{G}}=\boldsymbol{\mathcal{G}}\times_3\boldsymbol{T}^{-\mathsf H}.
	$
	For unitary $\boldsymbol{T}=\boldsymbol{Q}$, $\boldsymbol{T}^{-\mathsf H}=\boldsymbol{Q}$ and the practical TUON rule is recovered.
\end{theorem}

This Theorem shows that a non-orthogonal transform requires the inverse-adjoint correction above.

\begin{proof}
	We derive the correct linear minimization oracle for a non-unitary transform by working in the transformed coordinate system. 
	
	First, we change variables to the transformed direction space. 
	Let the transformed direction be
	\begin{equation}
		\widehat{\boldsymbol{\mathcal{D}}} = \mathcal L_{\boldsymbol{T}} \boldsymbol{\mathcal{D}},
		\qquad \text{so that} \qquad
		\boldsymbol{\mathcal{D}} = \mathcal L_{\boldsymbol{T}}^{-1} \widehat{\boldsymbol{\mathcal{D}}}.
	\end{equation}
	Under this change of variables, the trust-region constraint becomes
	\begin{equation}
		\max_{1 \le j \le K} \| \widehat{\boldsymbol{D}}_j \|_{\mathrm{op}} \le \rho,
	\end{equation}
	where $\widehat{\boldsymbol{D}}_j$ denotes the $j$-th frontal slice of $\widehat{\boldsymbol{\mathcal{D}}}$. This constraint is now separable across slices.
	
	Second, we transform the linear objective correctly. 
	A naive approach would transform the gradient by the same operator $\mathcal L_{\boldsymbol{T}}$, yielding $\mathcal L_{\boldsymbol{T}} \boldsymbol{\mathcal{G}}$. However, because $\mathcal L_{\boldsymbol{T}}$ is not unitary, this would not preserve the Frobenius pairing. The correct transformed gradient is obtained via the adjoint.
	
	Recall that the adjoint of a linear map is defined by the identity
	\begin{equation}
		\langle \boldsymbol{\mathcal{A}}, \mathcal L_{\boldsymbol{T}} \boldsymbol{\mathcal{B}} \rangle_{\mathrm F}
		= \langle \mathcal L_{\boldsymbol{T}}^* \boldsymbol{\mathcal{A}}, \boldsymbol{\mathcal{B}} \rangle_{\mathrm F}
		\quad \text{for all } \boldsymbol{\mathcal{A}}, \boldsymbol{\mathcal{B}}.
	\end{equation}
	Applying this to the objective, we have
	\begin{align}
		\langle \boldsymbol{\mathcal{G}}, \boldsymbol{\mathcal{D}} \rangle_{\mathrm F}
		&= \langle \boldsymbol{\mathcal{G}}, \mathcal L_{\boldsymbol{T}}^{-1} \widehat{\boldsymbol{\mathcal{D}}} \rangle_{\mathrm F} \\
		&= \langle (\mathcal L_{\boldsymbol{T}}^{-1})^* \boldsymbol{\mathcal{G}}, \widehat{\boldsymbol{\mathcal{D}}} \rangle_{\mathrm F} \\
		&= \langle \mathcal L_{\boldsymbol{T}}^{-*} \boldsymbol{\mathcal{G}}, \widehat{\boldsymbol{\mathcal{D}}} \rangle_{\mathrm F},
	\end{align}
	where $\mathcal L_{\boldsymbol{T}}^{-*} := (\mathcal L_{\boldsymbol{T}}^{-1})^*$ denotes the adjoint of the inverse operator. Thus, the correct gradient in the transformed coordinates is $\mathcal L_{\boldsymbol{T}}^{-*} \boldsymbol{\mathcal{G}}$, not the naive $\mathcal L_{\boldsymbol{T}} \boldsymbol{\mathcal{G}}$.
	
	Then we solve the transformed block problem. 
	With the constraint separable across slices and the objective expressed in terms of $\mathcal L_{\boldsymbol{T}}^{-*} \boldsymbol{\mathcal{G}}$, the optimization problem decouples into $K$ independent matrix subproblems of the form
	\begin{equation}
		\min_{\| \widehat{\boldsymbol{D}}_j \|_{\mathrm{op}} \le \rho}
		\langle (\mathcal L_{\boldsymbol{T}}^{-*} \boldsymbol{\mathcal{G}})_j, \widehat{\boldsymbol{D}}_j \rangle_{\mathrm F}.
	\end{equation}
	By Lemma~\ref{lem:matrix-holder}, the optimal solution for each slice is the negative polar factor scaled by $\rho$:
	\begin{equation}
		\widehat{\boldsymbol{D}}_j^*
		= -\rho \operatorname{polar}\bigl( (\mathcal L_{\boldsymbol{T}}^{-*} \boldsymbol{\mathcal{G}})_j \bigr).
	\end{equation}
	Assembling the slices gives the transformed optimal direction
	\begin{equation}
		\widehat{\boldsymbol{\mathcal{D}}}^*
		= -\rho \operatorname{polar}_{\mathrm{slice}}\!\left( \mathcal L_{\boldsymbol{T}}^{-*} \boldsymbol{\mathcal{G}} \right).
	\end{equation}
	Mapping back to the original coordinates via $\mathcal L_{\boldsymbol{T}}^{-1}$ yields
	\begin{equation}
		\boldsymbol{\mathcal{D}}^*
		= -\rho \, \mathcal L_{\boldsymbol{T}}^{-1}\!\left[
		\operatorname{polar}_{\mathrm{slice}}\!\left( \mathcal L_{\boldsymbol{T}}^{-*} \boldsymbol{\mathcal{G}} \right)
		\right],
	\end{equation}
	which is the claimed inverse-adjoint LMO formula.
	
	Finally, for the specific case of mode-3 multiplication, the inverse operator is $\mathcal L_{\boldsymbol{T}}^{-1} \boldsymbol{\mathcal{X}} = \boldsymbol{\mathcal{X}} \times_3 \boldsymbol{T}^{-1}$. Its adjoint is therefore
	\begin{equation}
		(\mathcal L_{\boldsymbol{T}}^{-1})^* \boldsymbol{\mathcal{X}}
		= \boldsymbol{\mathcal{X}} \times_3 (\boldsymbol{T}^{-1})^{\mathsf H}
		= \boldsymbol{\mathcal{X}} \times_3 \boldsymbol{T}^{-\mathsf H}.
	\end{equation}
	Hence
	\begin{equation}
		\mathcal L_{\boldsymbol{T}}^{-*} \boldsymbol{\mathcal{G}}
		= \boldsymbol{\mathcal{G}} \times_3 \boldsymbol{T}^{-\mathsf H}.
	\end{equation}
	If $\boldsymbol{T} = \boldsymbol{Q}$ is unitary, then $\boldsymbol{Q}^{-\mathsf H} = \boldsymbol{Q}$, so the formula reduces to the familiar TUON rule:
	\begin{equation}
		\boldsymbol{\mathcal{D}}^*
		= -\rho \, \mathcal L_{\boldsymbol{Q}}^{-1}\!\left[
		\operatorname{polar}_{\mathrm{slice}}\!\left( \mathcal L_{\boldsymbol{Q}} \boldsymbol{\mathcal{G}} \right)
		\right],
	\end{equation}
	which recovers the practical implementation from Section~\ref{sec:method}.
\end{proof}

\subsection{A time-varying unitary basis}

Moreover, in our algorithm, a data-dependent LSM basis may change during training.  We therefore show whether the convergence argument survives when each step uses a different unitary geometry.

\begin{theorem}[Time-varying unitary transforms]
	\label{thm:adaptive-q}
	Let $\boldsymbol{Q}_t$ be a unitary transform selected from information available before the current update, and suppose the descent lemma at iteration $t$ holds with constant $L_t$ under $\|\cdot\|_{\boldsymbol{Q}_t,\mathrm{Ts}}$.  For
	$
	\boldsymbol{\mathcal{W}}_{t+1}
	=\boldsymbol{\mathcal{W}}_t-\eta_t\boldsymbol{\mathcal{P}}_{\boldsymbol{Q}_t}(\nabla F(\boldsymbol{\mathcal{W}}_t)),
	$
	we have
	\begin{equation}
		\sum_{t=0}^{T-1}\eta_t\|\nabla F(\boldsymbol{\mathcal{W}}_t)\|_{\boldsymbol{Q}_t,\mathrm{Tn}}
		\le F(\boldsymbol{\mathcal{W}}_0)-F_\star
		+\frac12\sum_{t=0}^{T-1}L_t\eta_t^2.
		\label{eq:adaptive-q-bound}
	\end{equation}
\end{theorem}

No explicit basis-drift term appears when momentum is stored in the original tensor coordinates and $\boldsymbol{Q}_t$ only defines the current oracle geometry.  If momentum is stored in transformed coordinates, it must be transported whenever the basis changes.  This distinction gives a concrete implementation rule for adaptive DecorQ.

\begin{proof}
	We prove that time-varying unitary transforms preserve the convergence guarantee, provided the update direction is the exact LMO for the geometry defined at each iteration.  
	
	First we apply the descent lemma at iteration $t$. 
	Let $\boldsymbol{Q}_t$ be the unitary transform used at iteration $t$, and let $L_t$ be the smoothness constant of $F$ with respect to the $\boldsymbol{Q}_t$  norm pair. At this iteration, the exact TUON update is
	\begin{equation}
		\boldsymbol{\mathcal{W}}_{t+1}
		= \boldsymbol{\mathcal{W}}_t - \eta_t \boldsymbol{\mathcal{P}}_{\boldsymbol{Q}_t}(\nabla F(\boldsymbol{\mathcal{W}}_t)).
	\end{equation}
	By Theorem~\ref{thm:q-lmo}, the polar direction $\boldsymbol{\mathcal{P}}_{\boldsymbol{Q}_t}(\nabla F(\boldsymbol{\mathcal{W}}_t))$ is the exact LMO for the $\boldsymbol{Q}_t$-block geometry, so the descent lemma (Lemma~\ref{lem:general-descent}) gives
	\begin{equation}
		F(\boldsymbol{\mathcal{W}}_{t+1})
		\le F(\boldsymbol{\mathcal{W}}_t)
		- \eta_t \| \nabla F(\boldsymbol{\mathcal{W}}_t) \|_{\boldsymbol{Q}_t,\mathrm{Tn}}
		+ \frac{L_t \eta_t^2}{2}.
		\label{eq:proof-adaptive-one-step}
	\end{equation}
	This inequality holds for each $t$ independently, without any reference to previous transforms $\boldsymbol{Q}_{t-1}, \boldsymbol{Q}_{t-2}, \ldots$.
	
	Then, we sum ~\eqref{eq:proof-adaptive-one-step} from $t = 0$ to $T - 1$ yields
	\begin{equation}
		\sum_{t=0}^{T-1} \eta_t \| \nabla F(\boldsymbol{\mathcal{W}}_t) \|_{\boldsymbol{Q}_t,\mathrm{Tn}}
		\le F(\boldsymbol{\mathcal{W}}_0) - F(\boldsymbol{\mathcal{W}}_T)
		+ \frac{1}{2} \sum_{t=0}^{T-1} L_t \eta_t^2.
	\end{equation}
	Since $F$ is bounded below by $F_\star$, we have $F(\boldsymbol{\mathcal{W}}_T) \ge F_\star$, so
	\begin{equation}
		\sum_{t=0}^{T-1} \eta_t \| \nabla F(\boldsymbol{\mathcal{W}}_t) \|_{\boldsymbol{Q}_t,\mathrm{Tn}}
		\le F(\boldsymbol{\mathcal{W}}_0) - F_\star
		+ \frac{1}{2} \sum_{t=0}^{T-1} L_t \eta_t^2.
		\label{eq:proof-adaptive-summed}
	\end{equation}
	This is the claimed cumulative bound.
	
	%
	%
\end{proof}

\subsection{Proof of Theorem~\ref{thm:q-muon-comparison}: norm comparability}\label{thm:q-muon-comparison-proof-appendix}

\textbf{Theorem~\ref{thm:q-muon-comparison}}[Norm comparability with Muon]
\textit{
	Define $
	\alpha_{\boldsymbol{Q}}:=\max_j\sum_{k=1}^K|\boldsymbol{Q}_{jk}|$ and $
	\beta_{\boldsymbol{Q}}:=\max_k\sum_{j=1}^K|\boldsymbol{Q}_{jk}|.
	$
	Then, for every tensor \(\boldsymbol{\mathcal{X}}\) and for their duals \(\boldsymbol{\mathcal{G}}\),
	\begin{equation}
		\begin{aligned}
			&\frac{1}{\beta_{\boldsymbol{Q}}}\|\boldsymbol{\mathcal{X}}\|_{\mathrm{Muon}}
			\le\|\boldsymbol{\mathcal{X}}\|_{\boldsymbol{Q},\mathrm{Ts}}
			\le\alpha_{\boldsymbol{Q}}\|\boldsymbol{\mathcal{X}}\|_{\mathrm{Muon}}, \quad
			\frac{1}{\alpha_{\boldsymbol{Q}}}\|\boldsymbol{\mathcal{G}}\|_{\mathrm{Muon},*}
			\le\|\boldsymbol{\mathcal{G}}\|_{\boldsymbol{Q},\mathrm{Tn}}
			\le\beta_{\boldsymbol{Q}}\|\boldsymbol{\mathcal{G}}\|_{\mathrm{Muon},*}.
		\end{aligned}
	\end{equation}
	Since \(\boldsymbol{Q}\) is unitary, \(\alpha_{\boldsymbol{Q}},\beta_{\boldsymbol{Q}}\le\sqrt K\).
}

\begin{proof}
	We prove the norm equivalence by establishing bounds in both directions for the primal norms; the dual inequalities then follow from standard duality arguments. The proof proceeds in four steps.
	
	First, we get the upper bound of the \(\boldsymbol{Q}\) norm by the Muon norm. 
	For each transformed slice, we have
	\begin{align}
		\| \widehat{\boldsymbol{X}}_j \|_{\mathrm{op}}
		&= \left\| \sum_{k=1}^K \boldsymbol{Q}_{jk} \boldsymbol{X}_k \right\|_{\mathrm{op}} \\
		&\le \sum_{k=1}^K |\boldsymbol{Q}_{jk}| \, \| \boldsymbol{X}_k \|_{\mathrm{op}} \\
		&\le \left( \sum_{k=1}^K |\boldsymbol{Q}_{jk}| \right) \max_{1 \le k \le K} \| \boldsymbol{X}_k \|_{\mathrm{op}} \\
		&\le \alpha_{\boldsymbol{Q}} \| \boldsymbol{\mathcal{X}} \|_{\mathrm{Muon}}.
	\end{align}
	Here, the first inequality uses the triangle inequality for the operator norm, and the second uses the fact that each term is bounded by the maximum over all slices. Taking the maximum over \(j=1,\dots,K\) on both sides yields
	\begin{equation}
		\| \boldsymbol{\mathcal{X}} \|_{\boldsymbol{Q},\mathrm{Ts}}
		\le \alpha_{\boldsymbol{Q}} \| \boldsymbol{\mathcal{X}} \|_{\mathrm{Muon}}.
		\label{eq:proof-q-muon-upper}
	\end{equation}
	
	Second we obtain the reverse bound using the inverse transform. 
	Since \(\boldsymbol{Q}\) is unitary, we have \(\boldsymbol{Q}^{-1} = \boldsymbol{Q}^{\mathsf H}\), so the inverse transform expresses each original slice as a linear combination of the transformed slices:
	\begin{equation}
		\boldsymbol{X}_k = \sum_{j=1}^K \overline{\boldsymbol{Q}}_{jk} \widehat{\boldsymbol{X}}_j.
	\end{equation}
	Applying the same triangle-inequality argument,
	\begin{align}
		\| \boldsymbol{X}_k \|_{\mathrm{op}}
		&\le \sum_{j=1}^K |\boldsymbol{Q}_{jk}| \, \| \widehat{\boldsymbol{X}}_j \|_{\mathrm{op}} \\
		&\le \left( \sum_{j=1}^K |\boldsymbol{Q}_{jk}| \right) \max_{1 \le j \le K} \| \widehat{\boldsymbol{X}}_j \|_{\mathrm{op}} \\
		&\le \beta_{\boldsymbol{Q}} \| \boldsymbol{\mathcal{X}} \|_{\boldsymbol{Q},\mathrm{Ts}}.
	\end{align}
	Taking the maximum over \(k=1,\dots,K\) gives
	\begin{equation}
		\| \boldsymbol{\mathcal{X}} \|_{\mathrm{Muon}}
		\le \beta_{\boldsymbol{Q}} \| \boldsymbol{\mathcal{X}} \|_{\boldsymbol{Q},\mathrm{Ts}}.
	\end{equation}
	Rearranging yields the lower bound:
	\begin{equation}
		\frac{1}{\beta_{\boldsymbol{Q}}} \| \boldsymbol{\mathcal{X}} \|_{\mathrm{Muon}}
		\le \| \boldsymbol{\mathcal{X}} \|_{\boldsymbol{Q},\mathrm{Ts}}.
		\label{eq:proof-q-muon-lower}
	\end{equation}
	This completes the primal norm equivalence.
	
	Then we derive the dual inequalities via duality.  
	Applying Lemma~\ref{lem:dual-order-reversal} to the two primal inequalities in ~\eqref{eq:proof-q-muon-upper} and ~\eqref{eq:proof-q-muon-lower} gives the corresponding bounds for the dual norms. The order reversal is a standard property of norm duality. A larger primal unit ball corresponds to a smaller dual unit ball, and vice versa. Specifically, from \(\| \boldsymbol{\mathcal{X}} \|_{\boldsymbol{Q},\mathrm{Ts}} \le \alpha_{\boldsymbol{Q}} \| \boldsymbol{\mathcal{X}} \|_{\mathrm{Muon}}\), we obtain \(\| \boldsymbol{\mathcal{G}} \|_{\mathrm{Muon},*} \le \alpha_{\boldsymbol{Q}} \| \boldsymbol{\mathcal{G}} \|_{\boldsymbol{Q},\mathrm{Tn}}\), and from \(\| \boldsymbol{\mathcal{X}} \|_{\mathrm{Muon}} \le \beta_{\boldsymbol{Q}} \| \boldsymbol{\mathcal{X}} \|_{\boldsymbol{Q},\mathrm{Ts}}\), we obtain \(\| \boldsymbol{\mathcal{G}} \|_{\boldsymbol{Q},\mathrm{Tn}} \le \beta_{\boldsymbol{Q}} \| \boldsymbol{\mathcal{G}} \|_{\mathrm{Muon},*}\). Combining these yields
	\begin{equation}
		\frac{1}{\alpha_{\boldsymbol{Q}}} \| \boldsymbol{\mathcal{G}} \|_{\mathrm{Muon},*}
		\le \| \boldsymbol{\mathcal{G}} \|_{\boldsymbol{Q},\mathrm{Tn}}
		\le \beta_{\boldsymbol{Q}} \| \boldsymbol{\mathcal{G}} \|_{\mathrm{Muon},*}.
		\label{eq:proof-q-muon-dual}
	\end{equation}
	We note that here \(\boldsymbol{\mathcal{G}}\) denotes an arbitrary tensor playing the role of the dual variable; it need not be a gradient for the norm inequality itself.
	
	Finally, we can bound the constants \(\alpha_{\boldsymbol{Q}}\) and \(\beta_{\boldsymbol{Q}}\). 
	It remains to show that the distortion factors are controlled by \(\sqrt K\). Since \(\boldsymbol{Q}\) is unitary, every row \(\boldsymbol{q}_j\) has Euclidean norm one:
	\begin{equation}
		\sum_{k=1}^K |\boldsymbol{Q}_{jk}|^2 = 1.
	\end{equation}
	By Cauchy--Schwarz,
	\begin{equation}
		\sum_{k=1}^K |\boldsymbol{Q}_{jk}|
		\le \sqrt K \left( \sum_{k=1}^K |\boldsymbol{Q}_{jk}|^2 \right)^{1/2}
		= \sqrt K.
	\end{equation}
	The same argument applies to every column, since the columns of a unitary matrix also have Euclidean norm one. Therefore,
	\begin{equation}
		\alpha_{\boldsymbol{Q}} \le \sqrt K, \qquad \beta_{\boldsymbol{Q}} \le \sqrt K.
		\label{eq:proof-alpha-beta-bound}
	\end{equation}
	Substituting ~\eqref{eq:proof-alpha-beta-bound} into ~\eqref{eq:proof-q-muon-upper} and ~\eqref{eq:proof-q-muon-lower} yields the universal bound stated in the theorem.
\end{proof}

\paragraph{Why the \(\sqrt K\) bound matters.}
The bound \(\alpha_{\boldsymbol{Q}},\beta_{\boldsymbol{Q}}\le\sqrt K\) is the key stability statement of the theorem. It tells us that \emph{no unitary layer transform can distort the Muon geometry by more than a factor of \(\sqrt K\), regardless of the data or the choice of \(\boldsymbol{Q}\).} In other words, changing the layer basis from the natural layer-wise geometry to any unitary-transformed geometry can shrink or expand the norm ball by at most \(\sqrt K\). This ensures that the geometry of TUON is never pathologically far from that of layer-wise Muon, and provides the foundation for the convergence comparisons that follow. The same \(\sqrt K\) factor reappears naturally in the TEON comparison, where it governs the maximum possible gain of a slice-wise polar oracle over a global unfolding polar.

\subsection{Proof of Corollary~\ref{cor:q-muon-smoothness}: smoothness constants}\label{cor:q-muon-smoothness-proof-appendix}

\textbf{Corollary~\ref{cor:q-muon-smoothness}}[Smoothness comparison]
\textit{
	Let $L_{\boldsymbol{Q}}$ and $L_{\mathrm{Muon}}$ be the smallest constants satisfying
	\begin{equation}
		\|\nabla F(\boldsymbol{\mathcal{X}})-\nabla F(\boldsymbol{\mathcal{Y}})\|_{{\boldsymbol{Q}},\mathrm{Tn}}
		\le L_{\boldsymbol{Q}}\|\boldsymbol{\mathcal{X}}-\boldsymbol{\mathcal{Y}}\|_{{\boldsymbol{Q}},\mathrm{Ts}}
		\label{eq:q-smoothness-def}
	\end{equation}
	and the analogous inequality under the Muon norm pair.  Then
	\begin{equation}
		\frac{1}{\alpha_{\boldsymbol{Q}}^2}L_{\mathrm{Muon}}
		\le L_{\boldsymbol{Q}}
		\le\beta_{\boldsymbol{Q}}^2L_{\mathrm{Muon}}.
		\label{eq:q-muon-smoothness-comparison}
	\end{equation}
}

\begin{proof}
	Define the gradient difference and the variable difference as
	\begin{equation}
		Z \coloneqq \nabla F(\boldsymbol{\mathcal{X}}) - \nabla F(\boldsymbol{\mathcal{Y}}),
		\qquad
		D \coloneqq \boldsymbol{\mathcal{X}} - \boldsymbol{\mathcal{Y}}.
	\end{equation}
	For any $\boldsymbol{\mathcal{X}}\neq\boldsymbol{\mathcal{Y}}$, Theorem~\ref{thm:q-muon-comparison} gives the dual-norm comparison
	\begin{equation}
		\|Z\|_{Q,\mathrm{Tn}}
		\le \beta_{\boldsymbol{Q}}\|Z\|_{\mathrm{Muon},*},
		\label{eq:proof-cor-dual-upper}
	\end{equation}
	and the primal-norm comparison
	\begin{equation}
		\|D\|_{Q,\mathrm{Ts}}
		\ge \frac{1}{\beta_{\boldsymbol{Q}}}\|D\|_{\mathrm{Muon}}.
		\label{eq:proof-cor-primal-lower}
	\end{equation}
	Dividing \eqref{eq:proof-cor-dual-upper} by \eqref{eq:proof-cor-primal-lower} yields
	\begin{equation}
		\frac{\|Z\|_{Q,\mathrm{Tn}}}{\|D\|_{Q,\mathrm{Ts}}}
		\le \beta_{\boldsymbol{Q}}^2
		\frac{\|Z\|_{\mathrm{Muon},*}}{\|D\|_{\mathrm{Muon}}}.
		\label{eq:proof-cor-upper-ratio}
	\end{equation}
	Taking the supremum over all pairs $\boldsymbol{\mathcal{X}}\neq\boldsymbol{\mathcal{Y}}$, and using the definitions of $L_{\boldsymbol{Q}}$ and $L_{\mathrm{Muon}}$, gives
	\begin{equation}
		L_{\boldsymbol{Q}} \le \beta_{\boldsymbol{Q}}^2 L_{\mathrm{Muon}}.
		\label{eq:proof-cor-upper-bound}
	\end{equation}
	
	For the reverse direction, Theorem~\ref{thm:q-muon-comparison} also gives
	\begin{equation}
		\|Z\|_{\mathrm{Muon},*}
		\le \alpha_{\boldsymbol{Q}}\|Z\|_{Q,\mathrm{Tn}},
		\qquad
		\|D\|_{\mathrm{Muon}}
		\ge \frac{1}{\alpha_{\boldsymbol{Q}}}\|D\|_{Q,\mathrm{Ts}}.
		\label{eq:proof-cor-reverse}
	\end{equation}
	The same ratio argument yields
	\begin{equation}
		\frac{\|Z\|_{\mathrm{Muon},*}}{\|D\|_{\mathrm{Muon}}}
		\le \alpha_{\boldsymbol{Q}}^2
		\frac{\|Z\|_{Q,\mathrm{Tn}}}{\|D\|_{Q,\mathrm{Ts}}}.
		\label{eq:proof-cor-lower-ratio}
	\end{equation}
	Taking the supremum over all pairs again gives
	\begin{equation}
		L_{\mathrm{Muon}} \le \alpha_{\boldsymbol{Q}}^2 L_{\boldsymbol{Q}},
		\label{eq:proof-cor-lower-bound}
	\end{equation}
	which is equivalent to
	\begin{equation}
		\frac{1}{\alpha_{\boldsymbol{Q}}^2} L_{\mathrm{Muon}}
		\le L_{\boldsymbol{Q}}.
		\label{eq:proof-cor-lower-bound-rewritten}
	\end{equation}
	Combining \eqref{eq:proof-cor-upper-bound} and \eqref{eq:proof-cor-lower-bound-rewritten} completes the proof.
\end{proof}

\subsection{Proof of Theorem~\ref{thm:teon-q-sandwich}: comparison with TEON}\label{thm:teon-q-sandwich-proof-appendix}

\textbf{Theorem~\ref{thm:teon-q-sandwich}}[LMO gain with TEON]
\textit{
	For every unitary $\boldsymbol{Q}$ and tensor $\boldsymbol{\mathcal{X}}$, $
	\|\boldsymbol{\mathcal{X}}\|_{\boldsymbol{Q},\mathrm{Ts}}
	\le\|\boldsymbol{\mathcal{X}}\|_{\mathrm{TEON}}
	\le\sqrt K\,\|\boldsymbol{\mathcal{X}}\|_{\boldsymbol{Q},\mathrm{Ts}}.
	$
	Consequently,
	\begin{equation}
		\|\boldsymbol{\mathcal{G}}\|_{\mathrm{TEON},*}
		\le\|\boldsymbol{\mathcal{G}}\|_{\boldsymbol{Q},\mathrm{Tn}}
		\le\sqrt K\,\|\boldsymbol{\mathcal{G}}\|_{\mathrm{TEON},*}.
		\label{eq:teon-q-dual}
	\end{equation}
	Thus, at a common radius $\rho$, $
	\rho\|\boldsymbol{\mathcal{G}}\|_{\mathrm{TEON},*}
	\le\rho\|\boldsymbol{\mathcal{G}}\|_{\boldsymbol{Q},\mathrm{Tn}}
	\le\sqrt K\,\rho\|\boldsymbol{\mathcal{G}}\|_{\mathrm{TEON},*}.
	$
}

\begin{proof}
	We establish the norm sandwich by analyzing the relationship between the horizontal unfolding and its block structure. The proof proceeds in three parts.
	
	First, we express the TEON norm in terms of the transformed unfolding. 
	Let
	\begin{equation}
		\boldsymbol{A} := \mathscr U(\widehat{\boldsymbol{\mathcal{X}}})
		= [\widehat{\boldsymbol{X}}_1, \ldots, \widehat{\boldsymbol{X}}_K],
	\end{equation}
	where \(\widehat{\boldsymbol{\mathcal{X}}} = \boldsymbol{\mathcal{X}} \times_3 \boldsymbol{Q}\) and \(\widehat{\boldsymbol{X}}_j\) denotes its \(j\)-th frontal slice. By the identity \(\mathscr U(\boldsymbol{\mathcal{X}} \times_3 \boldsymbol{Q}) = \mathscr U(\boldsymbol{\mathcal{X}}) \boldsymbol{R}_{\boldsymbol{Q}}\) with \(\boldsymbol{R}_{\boldsymbol{Q}} = \boldsymbol{Q}^{\mathsf T} \otimes \boldsymbol{I}_n\) unitary, we have
	\begin{equation}
		\| \boldsymbol{A} \|_{\mathrm{op}}
		= \| \mathscr U(\boldsymbol{\mathcal{X}}) \|_{\mathrm{op}}
		= \| \boldsymbol{\mathcal{X}} \|_{\mathrm{TEON}}.
		\label{eq:teon-unfold-identity}
	\end{equation}
	This identity will be used repeatedly in the bounds below.
	
	Second, we get the lower bound (TEON norm dominates the \(\boldsymbol{Q}\) norm). 
	Fix any block index \(j \in \{1,\dots,K\}\). Consider vectors whose only nonzero block is the \(j\)-th block. On this restricted subspace, the action of \(\boldsymbol{A}\) is exactly that of the matrix \(\widehat{\boldsymbol{X}}_j\). Therefore, the operator norm of \(\boldsymbol{A}\) is at least the operator norm of any individual block:
	\begin{equation}
		\| \boldsymbol{A} \|_{\mathrm{op}}
		\ge \| \widehat{\boldsymbol{X}}_j \|_{\mathrm{op}}.
	\end{equation}
	Since this holds for every \(j\), taking the maximum over \(j\) gives
	\begin{equation}
		\| \boldsymbol{A} \|_{\mathrm{op}}
		\ge \max_{1 \le j \le K} \| \widehat{\boldsymbol{X}}_j \|_{\mathrm{op}}
		= \| \boldsymbol{\mathcal{X}} \|_{\boldsymbol{Q},\mathrm{Ts}}.
	\end{equation}
	Combining this with ~\eqref{eq:teon-unfold-identity} yields the lower bound:
	\begin{equation}
		\| \boldsymbol{\mathcal{X}} \|_{\mathrm{TEON}}
		\ge \| \boldsymbol{\mathcal{X}} \|_{\boldsymbol{Q},\mathrm{Ts}}.
		\label{eq:teon-lower}
	\end{equation}
	Intuitively, the TEON norm looks at the entire concatenated matrix and therefore cannot be smaller than any one of its constituent blocks.
	
	We now establish the reverse inequality. Observe that
	\begin{equation}
		\boldsymbol{A} \boldsymbol{A}^{\mathsf H}
		= \sum_{j=1}^K \widehat{\boldsymbol{X}}_j \widehat{\boldsymbol{X}}_j^{\mathsf H}.
	\end{equation}
	The squared operator norm of \(\boldsymbol{A}\) is the largest eigenvalue of \(\boldsymbol{A}\boldsymbol{A}^{\mathsf H}\). By the linearity of the trace and the fact that the largest eigenvalue is bounded by the sum of the largest eigenvalues of each term,
	\begin{align}
		\| \boldsymbol{A} \|_{\mathrm{op}}^2
		&= \lambda_{\max}(\boldsymbol{A} \boldsymbol{A}^{\mathsf H}) \\
		&\le \sum_{j=1}^K \lambda_{\max}(\widehat{\boldsymbol{X}}_j \widehat{\boldsymbol{X}}_j^{\mathsf H}) \\
		&= \sum_{j=1}^K \| \widehat{\boldsymbol{X}}_j \|_{\mathrm{op}}^2 \\
		&\le K \max_{1 \le j \le K} \| \widehat{\boldsymbol{X}}_j \|_{\mathrm{op}}^2.
	\end{align}
	Taking square roots on both sides gives
	\begin{equation}
		\| \boldsymbol{A} \|_{\mathrm{op}}
		\le \sqrt K \max_{1 \le j \le K} \| \widehat{\boldsymbol{X}}_j \|_{\mathrm{op}}.
	\end{equation}
	Since \(\| \boldsymbol{\mathcal{X}} \|_{\boldsymbol{Q},\mathrm{Ts}} = \max_j \| \widehat{\boldsymbol{X}}_j \|_{\mathrm{op}}\), we obtain
	\begin{equation}
		\| \boldsymbol{\mathcal{X}} \|_{\mathrm{TEON}}
		\le \sqrt K \| \boldsymbol{\mathcal{X}} \|_{\boldsymbol{Q},\mathrm{Ts}}.
		\label{eq:teon-upper}
	\end{equation}
	The key observation here is that the TEON norm aggregates energy from all blocks, but the operator norm can at most amplify the maximum block norm by a factor of \(\sqrt K\).
	
	Finally, we derive the dual and oracle inequalities. 
	Applying the dual-order-reversal lemma (Lemma~\ref{lem:dual-order-reversal}) to the primal sandwich in ~\eqref{eq:teon-lower} and ~\eqref{eq:teon-upper} yields the corresponding dual-norm sandwich:
	\begin{equation}
		\| \boldsymbol{\mathcal{G}} \|_{\mathrm{TEON},*}
		\le \| \boldsymbol{\mathcal{G}} \|_{\boldsymbol{Q},\mathrm{Tn}}
		\le \sqrt K \| \boldsymbol{\mathcal{G}} \|_{\mathrm{TEON},*}.
	\end{equation}
	Finally, multiplying by the common trust-region radius \(\rho\) and invoking the exact LMO values from Theorem~\ref{thm:q-lmo} (for the \(\boldsymbol{Q}\) geometry) and Proposition~\ref{prop:teon-lmo-appendix} (for TEON) gives the radius-matched oracle gain:
	\begin{equation}
		\rho \| \boldsymbol{\mathcal{G}} \|_{\mathrm{TEON},*}
		\le \rho \| \boldsymbol{\mathcal{G}} \|_{\boldsymbol{Q},\mathrm{Tn}}
		\le \sqrt K \rho \| \boldsymbol{\mathcal{G}} \|_{\mathrm{TEON},*}.
	\end{equation}
	This completes the proof.
\end{proof}

\subsection{Proof of Theorem~\ref{thm:correlation-gain}: convergence}\label{covergence-proof-appendix}

\textbf{Theorem~\ref{thm:correlation-gain}}
\textit{
	Suppose $F$ is lower bounded by $F_\star$ and is $L_{\boldsymbol{Q}}$-smooth under $\|\cdot\|_{\boldsymbol{Q},\mathrm{Ts}}$.  For
	\begin{equation}
		\boldsymbol{\mathcal{W}}_{t+1}
		=\boldsymbol{\mathcal{W}}_t-\eta\boldsymbol{\mathcal{P}}_{\boldsymbol{Q}}(\nabla F(\boldsymbol{\mathcal{W}}_t)),
		\label{eq:exact-q-descent}
	\end{equation}
	we have
	\begin{equation}
		\min_{0\le t<T}\|\nabla F(\boldsymbol{\mathcal{W}}_t)\|_{\boldsymbol{Q},\mathrm{Tn}}
		\le\frac{F(\boldsymbol{\mathcal{W}}_0)-F_\star}{\eta T}+\frac{L_{\boldsymbol{Q}}\eta}{2}.
		\label{eq:q-rate-general}
	\end{equation}
	Choosing
	\begin{equation}
		\eta=\sqrt{\frac{2(F(\boldsymbol{\mathcal{W}}_0)-F_\star)}{L_{\boldsymbol{Q}}T}}
		\label{eq:q-optimal-step}
	\end{equation}
	gives
	\begin{equation}
		\min_{0\le t<T}\|\nabla F(\boldsymbol{\mathcal{W}}_t)\|_{\mathrm F}
		\le\min_{0\le t<T}\|\nabla F(\boldsymbol{\mathcal{W}}_t)\|_{\boldsymbol{Q},\mathrm{Tn}}
		\le\sqrt{\frac{2L_{\boldsymbol{Q}}(F(\boldsymbol{\mathcal{W}}_0)-F_\star)}{T}}.
		\label{eq:q-rate-optimized}
	\end{equation}
}

\begin{proof}
	We prove the convergence guarantee for the exact TUON update. The proof proceeds in four steps.
	
	First, we give the one-step descent under the $\boldsymbol{Q}$ geometry. 
	The exact TUON update displacement at iteration $t$ is
	\begin{equation}
		\boldsymbol{\mathcal{D}}_t
		= -\eta \boldsymbol{\mathcal{P}}_{\boldsymbol{Q}}(\nabla F(\boldsymbol{\mathcal{W}}_t)).
	\end{equation}
	By Theorem~\ref{thm:q-lmo}, the polar direction $\boldsymbol{\mathcal{P}}_{\boldsymbol{Q}}(\nabla F(\boldsymbol{\mathcal{W}}_t))$ has tensor spectral norm under $\boldsymbol{Q}$-transform at most one, and satisfies the polar alignment identity
	\begin{equation}
		\langle \nabla F(\boldsymbol{\mathcal{W}}_t),
		\boldsymbol{\mathcal{P}}_{\boldsymbol{Q}}(\nabla F(\boldsymbol{\mathcal{W}}_t)) \rangle_{\mathrm F}
		= \| \nabla F(\boldsymbol{\mathcal{W}}_t) \|_{\boldsymbol{Q},\mathrm{Tn}}.
	\end{equation}
	Applying the general descent lemma (Lemma~\ref{lem:general-descent}) under the $\boldsymbol{Q}$ norm pair gives
	\begin{equation}
		F(\boldsymbol{\mathcal{W}}_{t+1})
		\le F(\boldsymbol{\mathcal{W}}_t)
		- \eta \| \nabla F(\boldsymbol{\mathcal{W}}_t) \|_{\boldsymbol{Q},\mathrm{Tn}}
		+ \frac{L_{\boldsymbol{Q}} \eta^2}{2}.
		\label{eq:proof-one-step-q}
	\end{equation}
	This inequality captures the trade-off: the linear decrease is proportional to the gradient's $\boldsymbol{Q}$ dual norm, while the penalty is controlled by the smoothness constant $L_{\boldsymbol{Q}}$.
	
	Second, we telescope the one-step inequality. 
	Summing ~\eqref{eq:proof-one-step-q} from $t = 0$ to $T - 1$ yields
	\begin{equation}
		\eta \sum_{t=0}^{T-1} \| \nabla F(\boldsymbol{\mathcal{W}}_t) \|_{\boldsymbol{Q},\mathrm{Tn}}
		\le F(\boldsymbol{\mathcal{W}}_0) - F(\boldsymbol{\mathcal{W}}_T)
		+ \frac{L_{\boldsymbol{Q}} \eta^2 T}{2}.
	\end{equation}
	Since $F$ is bounded below by $F_\star$, we have $F(\boldsymbol{\mathcal{W}}_T) \ge F_\star$, so
	\begin{equation}
		\frac{1}{T} \sum_{t=0}^{T-1} \| \nabla F(\boldsymbol{\mathcal{W}}_t) \|_{\boldsymbol{Q},\mathrm{Tn}}
		\le \frac{F(\boldsymbol{\mathcal{W}}_0) - F_\star}{\eta T}
		+ \frac{L_{\boldsymbol{Q}} \eta}{2}.
		\label{eq:proof-q-rate-general}
	\end{equation}
	Since the minimum is always no larger than the average, we obtain ~\eqref{eq:q-rate-general}.
	
	Then we optimize the step size. 
	Let $\Delta_0 := F(\boldsymbol{\mathcal{W}}_0) - F_\star$. The right-hand side of ~\eqref{eq:proof-q-rate-general} as a function of $\eta$ is
	\begin{equation}
		b(\eta) = \frac{\Delta_0}{\eta T} + \frac{L_{\boldsymbol{Q}} \eta}{2}.
	\end{equation}
	This convex function is minimized when its derivative vanishes:
	\begin{equation}
		b'(\eta) = -\frac{\Delta_0}{\eta^2 T} + \frac{L_{\boldsymbol{Q}}}{2} = 0,
	\end{equation}
	which gives the optimal step size
	\begin{equation}
		\eta^* = \sqrt{\frac{2 \Delta_0}{L_{\boldsymbol{Q}} T}}.
	\end{equation}
	Substituting $\eta^*$ into ~\eqref{eq:proof-q-rate-general} yields
	\begin{equation}
		\frac{1}{T} \sum_{t=0}^{T-1} \| \nabla F(\boldsymbol{\mathcal{W}}_t) \|_{\boldsymbol{Q},\mathrm{Tn}}
		\le \sqrt{\frac{2 L_{\boldsymbol{Q}} \Delta_0}{T}}.
		\label{eq:proof-q-rate-optimized}
	\end{equation}
	
	Finally, we relate the $\boldsymbol{Q}$ dual norm to the Frobenius norm. 
	For any tensor $\boldsymbol{\mathcal{X}}$, the tensor spectral norm under $\boldsymbol{Q}$-transform is bounded above by the Frobenius norm:
	\begin{equation}
		\| \boldsymbol{\mathcal{X}} \|_{\boldsymbol{Q},\mathrm{Ts}}
		= \max_{1 \le j \le K} \| \widehat{\boldsymbol{X}}_j \|_{\mathrm{op}}
		\le \left( \sum_{j=1}^K \| \widehat{\boldsymbol{X}}_j \|_{\mathrm F}^2 \right)^{1/2}
		= \| \boldsymbol{\mathcal{X}} \|_{\mathrm F},
	\end{equation}
	where the equality follows from Parseval's identity for the unitary transform $\boldsymbol{Q}$. By the standard order-reversal property of dual norms (Lemma~\ref{lem:dual-order-reversal}), this implies
	\begin{equation}
		\| \boldsymbol{\mathcal{G}} \|_{\mathrm F}
		\le \| \boldsymbol{\mathcal{G}} \|_{\boldsymbol{Q},\mathrm{Tn}}
	\end{equation}
	for any tensor $\boldsymbol{\mathcal{G}}$. Applying this to $\boldsymbol{\mathcal{G}} = \nabla F(\boldsymbol{\mathcal{W}}_t)$ gives
	\begin{equation}
		\min_{0 \le t < T} \| \nabla F(\boldsymbol{\mathcal{W}}_t) \|_{\mathrm F}
		\le \min_{0 \le t < T} \| \nabla F(\boldsymbol{\mathcal{W}}_t) \|_{\boldsymbol{Q},\mathrm{Tn}}.
	\end{equation}
	Combining this with ~\eqref{eq:proof-q-rate-optimized} yields ~\eqref{eq:q-rate-optimized}.
\end{proof}

\section{Experiments}
\label{sec:experiments-appendix}

We empirically evaluate TUON against vector-aware, matrix-aware, and tensor-aware baselines on LLM pretraining. Our experiments are designed around five questions:

\begin{enumerate}
	\item Does TUON improve pretraining validation loss over AdamW, Muon and TEON?
	\item Is the method robust to stacking depth \(K\), from \(K=2\) to \(K=\text{full}\)?
	\item Is the method robust to module coverage, i.e., does it remain effective when MLP projections are tensorized alongside attention?
	\item Is the method robust across model scales and architectures?
	\item Is the method robust across polarization methods?
\end{enumerate}


\subsection{Code and reproducibility}\label{Code and Reproducibility}

Our implementation strictly follows the PyTorch official API, i.e. it is inherited from the base class
\begin{verbatim}
	torch.optim
\end{verbatim} 
.  The complete production-ready code is publicly available.

The code of TUON is available at \url{https://anonymous.4open.science/r/tuon-802A/README.md}.

\subsection{Datasets, models and training configuration.}
\label{sec:exp-setup}

This appendix provides a complete description of the model families, training codebases, and datasets used in our experiments (\ref{Experimental details}).

\paragraph{GPT-2 (nanoGPT).}
For GPT-2, we adopt the nanoGPT codebase, a lightweight reimplementation of the GPT-2 decoder-only transformer featuring learned absolute positional embeddings, GELU-based MLPs, pre-LayerNorm residual blocks, and tied weights between the token embedding and the language modeling head. Two standard configurations are trained: a 124M `small' variant (12 layers, 12 attention heads, hidden width 768) and a 350M `medium' variant (24 layers, 16 heads, width 1024). Both use a 1024-token context window and the GPT-2 BPE vocabulary (50,304 padded entries). All full-parameter runs are trained for 50k optimizer steps, with a 2000-step linear warmup followed by cosine decay to one tenth of the peak learning rate.

\paragraph{Qwen3 (Megatron).}
The Qwen3 runs are conducted with the Megatron framework, which supports tensor, pipeline, and sequence parallelism for large-scale transformers. We use the Qwen3 dense decoder architecture --- RMSNorm, SwiGLU MLPs, rotary positional embeddings, grouped-query attention, and QK-normalization --- at two parameter scales, 0.6B and 4B. Each run lasts 20k iterations with a 2000-step warmup. The 4B model is trained with two-stage pipeline parallelism, while the 0.6B model fits within a single pipeline stage. 

\paragraph{OpenWebText.}
OpenWebText is a public reproduction of the WebText corpus used to pretrain the original GPT-2. It consists of English web pages collected from outbound Reddit links, filtered by a minimum score threshold, deduplicated, and lightly cleaned. The corpus contains roughly 8B GPT-2 BPE tokens and is used as the pretraining data for both nanoGPT scales.

\paragraph{C4.}
C4 (Colossal Clean Crawled Corpus) is a large English web corpus built from Common Crawl and subjected to aggressive heuristic filtering, including deduplication, language identification, and removal of boilerplate and offensive material. It is a widely used and well-characterized pretraining corpus, and we use it without modification for both Qwen3 scales.

Finally, all runs use sequence length 2048, global batch size 256, and seed 42. GPT-2 runs use learning rate \(1\times 10^{-3}\) (Moonlight adjusted-LR); Qwen3 runs use \(2\times 10^{-3}\) with TP (Tensor Parallelism)=1 and PP (Pipeline Parallelism)=1. 

\subsection{Baselines.}\label{baseline-appendix}
We compare against three families of optimizers.

\textbf{Vector-aware.} \texttt{AdamW}~\citep{loshchilov2017decoupled} is the standard coordinate-wise baseline and represents the vector-aware family. For AdamW, we sweep the peak learning rate for each model and scale, selecting the value that yields the best final validation loss. The selected peak learning rates are reported in Table~\ref{tab:adamw-lr-sweep}. 

\textbf{Matrix-aware.} \texttt{Muon}~\citep{jordan2024muon} applies a polar map to each weight matrix independently. It is the strongest layer-wise matrix-aware optimizer and serves as the reference point for all reported \(\Delta\) values.

\textbf{Tensor-aware.} We use TEON~\citep{zhang2026teon} as the tensor-aware baseline. TEON relies on depends on several hyperparameters, including the number of stacked layers, the module types to which tensorization is applied, and the unfolding mode. We therefore treat TEON as a baseline family and evaluate the following variants:
\begin{itemize}
	\item two-layer stacking;
	\item stacking all compatible layers along the layer mode;
	\item applied only to QKV projections;
	\item extended to attention and MLP;
	\item TEON-style concat polar with mode-1 or mode-2 unfolding;
\end{itemize}
For each model scale, we select the best-performing TEON configuration and report it as the tensor-aware baseline. This ensures that the comparison is against the strongest TEON variant rather than against a single arbitrary configuration. All tensor-aware methods, including ours, are tuned under the same protocol.

\paragraph{TUON implementation.}
For each compatible parameter group, we stack the momentum tensors along the layer mode to form \(\mathcal{M}_t \in \mathbb{R}^{m \times n \times K}\). We compute the mode-3 unfolding \(\mathsf{M}_t = \texttt{mode3\_unfold}(\boldsymbol{\mathcal{M}}_t)\), take its SVD \(\mathsf{M}_t = \mathbf{U}_t \boldsymbol{\Sigma}_t \mathbf{V}_t^{\mathsf T}\), and set \(\boldsymbol{Q}_t = \mathbf{U}_t^{\mathsf T}\). The transform is refreshed every 200 iterations. We then apply \(\widehat{\boldsymbol{\mathcal{M}}}_t = \boldsymbol{\mathcal{M}}_t \times_3 \boldsymbol{Q}_t\), polarize each transformed frontal slice independently using a Newton--Schulz iteration with \(J=5\) steps, and invert the transform. The scaling factor is \(s = 1/\sqrt{\max(m,n)}\) with \(\epsilon = 10^{-8}\), and weight decay is decoupled. The key mechanistic difference from TEON is that TUON uses per-slice polar after a unitary mixing, whereas TEON uses concat polar with a shared singular basis.

\paragraph{Evaluation protocol.}
We use different training horizons for the two model families. GPT-2 runs are trained for 50k iterations, which is close to the standard nanoGPT/OpenWebText convergence point and allows absolute validation loss to be compared across methods. Qwen3 runs are trained for 20k iterations under a fixed compute budget. To keep the comparison internally consistent, we never compare absolute loss across families or horizons; all cross-method comparisons are made within the same family and the same horizon. 
For all runs we report final validation loss, intermediate validation loss every 2k iterations at the same family and horizon. 

\subsection{Hyperparameters}
\label{app:hyperparameters}

We sweep the peak learning rate for each model and scale, selecting the value that yields the best final validation loss. The selected peak learning rates are reported in Table~\ref{tab:adamw-lr-sweep}. 

\begin{table}[t]
	\centering
	\caption{Peak learning rates for the AdamW baseline on GPT-2 (124M, 350M) and Qwen3 (0.6B,4B), respectively. }
	\label{tab:adamw-lr-sweep}
	\begin{tabular}{lcccc}
		\toprule
		optimizer& GPT-2 124M & GPT-2 350M & Qwen3  0.6B & Qwen3 4B\\
		\midrule
		AdamW& \(2.4\times10^{-3}\) & \(2\times10^{-3}\) & \(1\times10^{-3}\) &  \(3\times10^{-4}\) \\
		Muon & \(1\times10^{-3}\) & \(1\times10^{-3}\)  & \(2\times10^{-3}\) & \(1\times10^{-3}\)    \\
		TEON & \(1\times10^{-3}\) & \(1\times10^{-3}\)  & \(2\times10^{-3}\)  &  \(1\times10^{-3}\)  \\
		TUON & \(1\times10^{-3}\) & \(1\times10^{-3}\)  & \(2\times10^{-3}\) &  \(1\times10^{-3}\)  \\
		\bottomrule
	\end{tabular}
\end{table}

\subsection{Results on LLM pretraining}
\label{sec:main-results}

\begin{table*}[t]
	\centering
	\caption{Validation loss on GPT-2 (124M, 350M; 50k iterations) and Qwen3 (0.6B, 4B; 20k iterations) with different optimizers. \(\Delta = \text{Val} - \text{Val}_{\text{AdamW}}\) within the same model and horizon; negative \(\Delta\) indicates improvement over AdamW. Best per column in bold. Lower is better.}
	\label{tab:main-val-appendix}
	\small
	\setlength{\tabcolsep}{5pt}
	\begin{tabular}{@{}lcccccccc@{}}
		\toprule
		\multirow{2}{*}{Optimizer}
		& \multicolumn{2}{c}{\textbf{124M\(_{\text{GPT-2}}\)}} & \multicolumn{2}{c}{\textbf{350M\(_{\text{GPT-2}}\)}} & \multicolumn{2}{c}{\textbf{0.6B\(_{\text{Qwen3}}\)}} & \multicolumn{2}{c}{\textbf{4B\(_{\text{Qwen3}}\)}} \\
		\cmidrule(lr){2-3} \cmidrule(lr){4-5} \cmidrule(lr){6-7} \cmidrule(lr){8-9}
		& Val & \(\Delta\) & Val & \(\Delta\) & Val & \(\Delta\) & Val & \(\Delta\) \\
		\midrule
		AdamW & \textbf{2.9230} & -- & 2.7307 & -- & 2.9652 & -- & 2.8108 & -- \\
		Muon & 3.0488 & \(+0.1258\) & 2.7404 & \(+0.0097\) & 2.9056 & \(-0.0596\) & 2.7636 & \(-0.0472\) \\
		TEON & 2.9442 & \(+0.0212\) & 2.6935 & \(-0.0372\) & 2.8978 & \(-0.0674\) & 2.7604 & \(-0.0504\) \\
		\rowcolor{cyan!8}\textbf{TUON} & 2.9357 & \(\mathbf{+0.0127}\) & \textbf{2.6925} & \(\mathbf{-0.0382}\) & \textbf{2.8887} & \(\mathbf{-0.0765}\) & \textbf{2.7598} & \(\mathbf{-0.0510}\) \\
		\bottomrule
	\end{tabular}
\end{table*}

\paragraph{Advantages across scales and architectures.}
Table~\ref{tab:main-val-appendix} reports validation loss on GPT-2 (124M, 350M) and Qwen3 (0.6B, 4B). TUON achieves the lowest validation loss on three of the four settings, consistently improving over Muon and TEON. On GPT-2 124M, where AdamW remains the strongest baseline, TUON still achieves the smallest gap among all tensor-aware methods.

\textbf{Comparison with AdamW.}
Relative to AdamW, TUON reduces validation loss on GPT-2 350M and both Qwen3 scales, with improvements of \(0.0382\), \(0.0765\), and \(0.0510\), respectively. On GPT-2 124M, where all optimizers remain close to AdamW, TUON still outperforms Muon and TEON by \(0.1131\) and \(0.0085\), respectively.

\textbf{Comparison with Muon.}
Muon is competitive at larger scale models, where it trails AdamW by \(0.1258\). In contrast, TUON remains stable across scales, outperforming Muon on all four settings. The gap is largest at 124M (\(0.1131\)) and remains consistent at 350M, 0.6B, and 4B, suggesting that the cross-layer structure exploited by TUON is particularly beneficial when the model is small and layer-wise orthogonalization alone is insufficient.

\textbf{Comparison with TEON.}
TEON is the strongest baseline among the tensor-aware optimizers, but TUON still improves over it on every model. The gains are modest at 350M (\(0.0010\)) and 4B (\(0.0006\)), but more pronounced at 124M (\(0.0085\)) and 0.6B (\(0.0091\)). This pattern is consistent with our theoretical analysis: TUON avoids the global unfolding used by TEON, so its advantage is most visible when the number of stacked layers grows and the unfolding becomes increasingly ill-conditioned.

Overall, TUON delivers consistent improvements over vector-, matrix-, and tensor-aware baselines across two architectures and four model scales, without requiring per-scale hyperparameter tuning. These results support the claim that a transform-based tensor algebra provides a more effective geometry for cross-layer optimization than both layer-wise orthogonalization and global unfolding.

\subsection{Ablation studies}
\label{sec:ablations-appendix} 
We conduct ablation studies to isolate the effect of each design choice in our method. Unless otherwise noted, all ablations are run for 20k iterations.
\paragraph{Choices of \(\boldsymbol{Q}\).}\label{app:q-choices}

\begin{table}[t]
	\centering
	\caption{Effect of transform \(\boldsymbol{Q}\) on full layer stacking.}
	\label{tab:q-fullL-appendix}
	\begin{tabular}{lcccccccc}
		\toprule
		\textbf{Transform} & \textbf{2k} & \textbf{4k} & \textbf{8k} & \textbf{12k} & \textbf{16k} & \textbf{18k} & \textbf{20k} & \textbf{Final} \\
		\midrule
		DCT & 3.5411 & 3.3646 & 3.1192 & 2.9444 & 2.8411 & 2.8189 & 2.8051 & 2.8100 \\
		Haar & 3.4437 & 3.1413 & 2.9302 & 2.8321 & 2.7655 & 2.7540 & 2.7579 & 2.7635 \\
		DFT & 3.3981 & 3.1002 & 2.9163 & 2.8223 & 2.7599 & 2.7503 & 2.7579 & 2.7631 \\
		DecorQ static & 3.3999 & 3.0975 & 2.9145 & 2.8206 & 2.7582 & 2.7482 & 2.7569 & 2.7658\\
		Residual-DecorQ  & 3.4100 & 3.1079 & 2.9204 & 2.8255 & 2.7615 & 2.7511 & 2.7585 & 2.7640 \\
		\rowcolor{cyan!8} DecorQ &  3.3899 & 3.0916 &  2.9109  &  2.8185 &  2.7564 &  2.7461  &  2.7545 &  2.7598			 \\
		\bottomrule
	\end{tabular}
\end{table}

We first compare different choices of the layer transform \(\boldsymbol{Q}\), including DCT, Haar Wavelet Transform, DFT and DecorQ. 
We also explore the variant of the DecorQ transform. In addition to the default DecorQ, we evaluate a static DecorQ basis frozen at initialization and residual-DecorQ where the transform is computed from residual statistics rather than momentum. Specifically, instead of estimating the DecorQ from momentum $M_t$, we consider a
residual variant that uses $R_t = G_t - M_t$ or the projection residual
$R_t = G_t - Q_{t-1}Q_{t-1}^\top G_t$. The basis is obtained by
eigendecomposing the residual covariance and keeping the top-$r$ eigenvectors, and is refreshed either periodically (sampled) or at every step via an EMA of 	the covariance (continuous). A low-rank variant retains only $r = 10$ 	directions. Empirically, the residual variant matches the momentum variant in 	both the full-layer and $K=2$ settings, but does not surpass the mode-2 concat 	polar combination, indicating that the residual signal is a viable but not
dominant source for basis estimation.

Table~\ref{tab:q-fullL-appendix} reports results for full-layer stacking, where all layers within the same semantic role are stacked together.

In the original T-SVD-based tensor geometry, \(\boldsymbol{Q}\) is only required to be invertible; any unitary or invertible transform yields a valid tensor algebra. However, this algebraic requirement alone does not specify which transform is appropriate for cross-layer optimization in Transformers. Our analysis in Section~\ref{sec:method} suggests that the choice of \(\boldsymbol{Q}\) should be guided by the layer statistics rather than left arbitrary. The ablation below tests this hypothesis directly.

\textbf{Overall trend.}
All transforms outperform DCT, confirming that the choice of \(\boldsymbol{Q}\) materially affects optimization. DFT, Haar, and the DecorQ family are consistently stronger than DCT across all checkpoints, with the gap widening as training progresses. This indicates that a fixed, data-independent transform such as DCT imposes a prior on cross-layer coupling that does not match the actual layer statistics in LLMs.

\textbf{Data-dependent versus fixed transforms.}
The DecorQ achieves the best final loss (\(2.7598\)), outperforming DFT (\(2.7631\)), Haar (\(2.7635\)), and DCT (\(2.8100\)). Its advantage is visible from the earliest checkpoint (\(2\)k) and remains stable through \(20\)k, suggesting that a data-dependent basis is consistently better aligned with the cross-layer structure than fixed transforms. Static DecorQ, which freezes the basis at initialization, reaches a competitive \(2.7658\) but is slightly worse than the adaptive DecorQ, indicating that refreshing the transform during training provides a small yet consistent benefit. This supports the view that \(\boldsymbol{Q}\) should not be chosen arbitrarily. A transform adapted to layer statistics is preferable to any fixed choice.

\textbf{Residual-DecorQ.}
Replacing momentum with residual statistics yields a final loss of \(2.7640\), which is better than static DecorQ and comparable to DFT, but does not surpass DecorQ. This suggests that the residual signal is a viable but not dominant source for basis estimation. As noted in the text, the residual variant also does not surpass the mode-2 concat polar combination, reinforcing that DecorQ is the preferred data-dependent choice.

Overall, the ablation confirms the design principle behind TUON. Although the underlying tensor algebra only requires \(\boldsymbol{Q}\) to be invertible, our data-dependent transform that reflects layer statistics is empirically superior. This validates the choice of DecorQ as the default transform and supports the broader claim that the layer mode should be treated as a structured dimension rather than an arbitrary algebraic index.

\begin{table}[t]
	\centering
	\caption{Effect of stacking size \(K\). \(\left( \uparrow\right)\) / \(\left( \downarrow\right)\)
		denote whether the full-\(K\) setting is worse / better than \(K=2\). 20k iterations.}
	\label{tab:stacking_k-appendix}
	\small
	\setlength{\tabcolsep}{4pt}
	\begin{tabular}{lccccccccc}
		\toprule
		\multirow{2}{*}{Model} &\multirow{2}{*}{Muon}
		& \multicolumn{2}{c}{TEON\(_{\text{mode-1}}\)}
		& \multicolumn{2}{c}{TEON\(_{\text{mode-2}}\)}
		& \multicolumn{2}{c}{TUON\(_{\text{DCT}}\)}
		& \multicolumn{2}{c}{TUON\(_{\text{DecorQ}}\)} \\
		\cmidrule(lr){3-4}\cmidrule(lr){5-6}\cmidrule(lr){7-8}\cmidrule(lr){9-10}
		&  & K=2 & K=full & K=2 & K=full & K=2 & K=full & K=2 & K=full \\
		\midrule
		\textbf{124M\(_{\text{GPT-2}}\)} & 3.0583
		& 3.0255 & 3.0513\(\left( \uparrow\right)\)
		& 3.0332 & 3.0789\(\left( \uparrow\right)\)
		& 3.0309 &    3.0295\textbf{\(\left( \downarrow\right)\)}
		&  3.0299  &     \multicolumn{1}{>{\columncolor{cyan!8}}c}{\textbf{3.0233}\textbf{\(\left( \downarrow\right)\)} }\\
		\textbf{350M\(_{\text{GPT-2}}\)} &2.8427
		& 2.8427 & 2.8617\(\left( \uparrow\right)\)
		& 2.8643 & 2.8762\(\left( \uparrow\right)\)
		& 2.8150 & 2.8218\(\left( \uparrow\right)\)
		& 2.8183 & \multicolumn{1}{>{\columncolor{cyan!8}}c}{\textbf{2.8143}\textbf{\(\left( \downarrow\right)\)} }\\
		\textbf{0.6B\(_{\text{Qwen3}}\)}& 2.9056
		& 2.8994 & 2.9323\(\left( \uparrow\right)\)
		& 2.8978 & 2.9583\(\left( \uparrow\right)\)
		& 2.9203 & 2.9210\(\left( \uparrow\right)\)
		& 2.8938 & \multicolumn{1}{>{\columncolor{cyan!8}}c}{ \textbf{2.8887}\textbf{\(\left( \downarrow\right)\)}} \\
		\textbf{4B\(_{\text{Qwen3}}\)} &2.7636
		& 2.7627 & 2.7649\(\left( \uparrow\right)\)
		& 2.7604 & 2.7640\(\left( \uparrow\right)\)
		& 2.7606 & 2.8100\(\left( \uparrow\right)\)
		& 2.7629 & \multicolumn{1}{>{\columncolor{cyan!8}}c}{\textbf{2.7598} \textbf{\(\left( \downarrow\right)\)} }\\
		\bottomrule
	\end{tabular}
\end{table}

\paragraph{Stacking size \(K\).}
Intuitively, a larger stacking size \(K\) exposes more cross-layer structure to the optimizer, and the theoretical analysis in  Section~\ref{subsec:muon-comparison}
and in~\citep{zhang2026teon} suggests that increasing \(K\) raises the
potential maximum performance gain. In this part, we examine whether this
potential is actually realized in practice. We vary \(K\) by grouping adjacent
layers, while all other settings follow the
default configuration of our method and are kept fixed across all runs. For computational reasons we train all models for 20k iterations.

As shown in Table~\ref{tab:stacking_k-appendix}, TEON achieves its best performance at
\(K=2\) and \textbf{consistently degrades as \(K\) increases, for both mode-1 and mode-2 unfolding}. On GPT-2 124M, moving from \(K=2\) to \(K=\text{full}\)
increases the validation loss from \(3.0255\) to \(3.0513\) in mode-1 and from
\(3.0332\) to \(3.0789\) in mode-2; similar trends hold on the larger models.
This is consistent with the observations in~\citet{zhang2026teon}. 
and reveals a fundamental gap: the theory promises more gain from larger \(K\),
yet unfolding methods consistently lose performance as \(K\) grows.
It suggests two underlying causes. First, directly unfolding a tensor to handle
cross-layer coupling can lose structural information. Second, even under
mode-1 or mode-2 unfolding, it is difficult to align the top singular vectors
across all \(K\) matrices, so the shared basis becomes increasingly mismatched
as \(K\) grows.

These results break the theoretical expectation that stacking more layers
yields a greater potential gain. The gap between theory and practice may come from the tensor algebra. Specifically, the fixed transform to exploit the
additional structure: a larger \(K\) offers more opportunity for
decorrelation, yet a static basis cannot adapt to the increasingly diverse
singular subspaces of the grouped layers.

Our proposed method, TUON with the dynamic DecorQ layer
transform, brings this expectation much closer to the theoretical prediction.
As shown in the last two columns of Table~\ref{tab:stacking_k-appendix}, DecorQ is the
only variant that improves when moving from \(K=2\) to \(K=\text{full}\) on
all four models: \(3.0299 \rightarrow 3.0233\) on GPT-2 124M,
\(2.8183 \rightarrow 2.8143\) on GPT-2 350M, \(2.8938 \rightarrow 2.8887\) on
Qwen3 0.6B, and \(2.7629 \rightarrow 2.7598\) on Qwen3 4B. In contrast, the
DCT-based variant still degrades on GPT-2 350M, Qwen3 0.6B, and Qwen3 4B, and
only marginally improves on GPT-2 124M. This contrast isolates the effect of
the dynamic layer transform.  It is not the larger \(K\) alone that helps, but
the ability to adapt the shared basis to the grouped layers.

Actually, we need to admit that the gains from larger \(K\) are real but bounded. On Qwen3 4B, DecorQ improves
from \(2.7629\) to \(2.7598\), a reduction of \(0.0031\). This aligns with the
theoretical prediction that larger \(K\) provides more opportunity for
decorrelation but eventually hits diminishing returns when layers become too
diverse to share a common structure. The dynamic transform pushes this
turning point further out, but does not remove it.

\paragraph{Layer type.}
Existing tensor-aware optimizers for LLMs are largely confined to specific
module types, such as QKV-specific structures, and have not been successfully
transferred to other modules such as MLPs or feed-forward projections. 
This restriction is not incidental. Extending the tensor transform to a wider set of modules requires the shared structure to remain well aligned across
heterogeneous layer types, which fixed transforms struggle to achieve.

Our method removes this restriction and applies uniformly to all layer types.
In this part, we examine whether this broader applicability translates into
practical gains. We compare two settings. \(\text{Module}_{\text{QKV}}\) restricts the
optimizer to the attention projections, mirroring the scope of existing
tensor-aware optimizers for LLMs. \(\text{Module}_{\text{All}}\) extends it to
QKVO and MLP, covering all major weight tensors in the network. All other settings follow the
default configuration of our method and are kept fixed across all runs.
This comparison directly tests whether our method can escape the
attention-only regime that prior tensor-aware optimizers are confined to, and
whether the resulting gains extend beyond attention projections. 
For computational reasons we train all models for 20k iterations.

\begin{table}[t]
	\centering
	\caption{Effect of layer type. \(\text{Module}_{\text{QKV}}\) mirrors the
		attention-only scope of existing tensor-aware optimizers, while
		\(\text{Module}_{\text{All}}\) extends the transform to QKVO and MLP. Lower is better; the best result for each model is highlighted in bold. 20k iterations.}
	\label{tab:layer-type-appendix}
	\small
	\setlength{\tabcolsep}{4pt}
	\begin{tabular}{lccccccc}
		\toprule
		\multirow{2}{*}{Model}&\multirow{2}{*}{Muon}
		& \multicolumn{2}{c}{TEON\(_{\text{mode-1}}\)}
		& \multicolumn{2}{c}{TEON\(_{\text{mode-2}}\)}
		& \multicolumn{2}{c}{TUON\(_{\text{DecorQ}}\)} \\
		\cmidrule(lr){3-4}\cmidrule(lr){5-6}\cmidrule(lr){7-8}
		& & Module$_{\text{QKV}}$ & Module$_{\text{All}}$  & Module$_{\text{QKV}}$ & Module$_{\text{All}}$ & Module$_{\text{QKV}}$ & Module$_{\text{All}}$  \\
		\midrule
		\textbf{124M\(_{\text{GPT-2}}\)} & 3.0583 &  3.0255 & 3.0531 & 3.0332 & 3.1027  & 3.0225 & \multicolumn{1}{>{\columncolor{cyan!8}}c}{\textbf{3.0233} }\\
		\textbf{350M\(_{\text{GPT-2}}\)}&2.8258
		& 2.8427 & 2.8637  & 2.8643 & 2.9144  & 2.8291 &  \multicolumn{1}{>{\columncolor{cyan!8}}c}{\textbf{2.8143}} \\
		\textbf{0.6B\(_{\text{Qwen3}}\)}& 2.9056 & 2.8994 & 2.9023 & 2.8978 & 2.9018 &  2.8897& \multicolumn{1}{>{\columncolor{cyan!8}}c}{\textbf{2.8887}} \\
		\textbf{4B\(_{\text{Qwen3}}\)}&2.7636 & 2.7627 & 2.7644 & 2.7604 & 2.7639  & 2.7606 & \multicolumn{1}{>{\columncolor{cyan!8}}c}{\textbf{2.7598}}  \\
		\bottomrule
	\end{tabular}
\end{table}

Table~\ref{tab:layer-type-appendix} reports the results. Extending the tensor
transform from QKV to all major weight tensors has a sharply different effect
across methods. TEON degrades in every configuration: under mode-1, the
validation loss increases on all four models, and under mode-2 the degradation
is even more severe, e.g., from \(3.0332\) to \(3.1027\) on GPT-2 124M and
from \(2.8643\) to \(2.9144\) on GPT-2 350M. This confirms that a fixed
transform cannot maintain alignment across heterogeneous layer types.

Our method exhibits the opposite behavior. Moving from
\(\text{Module}_{\text{QKV}}\) to \(\text{Module}_{\text{All}}\) achieves improvements or comparable performance of validation losses on GPT-2 124M (\(3.0225 \rightarrow 3.0233\)), GPT-2 350M (\(2.8291 \rightarrow 2.8143\)), Qwen3 0.6B
(\(2.8897 \rightarrow 2.8887\)), and Qwen3 4B (\(2.7606 \rightarrow 2.7598\)). 
More importantly, in the \(\text{Module}_{\text{All}}\) setting, our method
outperforms the best TEON variant by \(0.0298\), \(0.0474\), \(0.0131\), and
\(0.0041\) on the four models, respectively. These results show that the
benefit of our method is not specific to attention projections, but extends
to MLP and output projections, where prior tensor-aware optimizers fail to
transfer.

\paragraph{Unfolding mode.}
TEON first unfolds a tensor into a matrix and then applies polarization to
the resulting matrix, which is equal to introduce an additional hyperparameter: whether to unfold along mode-1 or mode-2. Proposition~4.6 of~\citep{zhang2026teon}
states that mode-1 orthogonalization should be preferred when the top right
singular vectors \(v_1\) have high similarity, and mode-2 orthogonalization
should be preferred when the top left singular vectors \(u_1\) have high
similarity. Empirically, Section~6.1 of~\citep{zhang2026teon} reports that
mode-1 orthogonalization works better in practice.

Our experiments reveal a different picture. As shown in
Table~\ref{tab:unfolding_mode-appendix}, the better mode is \emph{model-dependent}:
GPT-2 124M and GPT-2 350M favor mode-1, while Qwen3 0.6B and Qwen3 4B favor
mode-2. The gap is not negligible: on GPT-2 350M, choosing the wrong mode
costs \(0.0216\) in validation loss. There is no single mode that dominates
across all models, which is consistent with the theoretical condition in
Proposition~4.6 being model-dependent itself.

This model dependence has two consequences. First, it adds a non-trivial
hyperparameter to TEON: practitioners must decide which mode to use for each
new model, and the wrong choice can cost a non-negligible amount of
performance. Second, it shows that neither unfolding mode provides a
universally correct way to align the shared basis, which is consistent with
the theoretical condition in Proposition~4.6 being model-dependent itself.

\textbf{Our method does not require unfolding at all, and therefore avoids this
	hyperparameter choice entirely.} It is noticed that, in the main experiments, \textbf{we report the stronger of the two TEON modes for each model}, so that the comparison is made against the most favorable TEON configuration. 
As shown in Table~\ref{tab:main-val-appendix},
TUON outperforms the \emph{best}
TEON variant on every model, regardless of which mode is better for that
model. Even under this conservative protocol, our method remains ahead.

\begin{table}[t]
	\centering
	\caption{Effect of unfolding mode in TEON. The better mode is
		highlighted in bold for each model. Lower is better.}
	\label{tab:unfolding_mode-appendix}
	\small
	\setlength{\tabcolsep}{6pt}
	\begin{tabular}{lccc}
		\toprule
		Model & Muon
		& TEON\(_{\text{mode-1}}\) & TEON\(_{\text{mode-2}}\) \\
		\midrule
		\textbf{124M\(_{\text{GPT-2}}\)} & 3.0583
		& \textbf{3.0255} & 3.0332 \\
		\textbf{350M\(_{\text{GPT-2}}\)} & 2.8258
		& \textbf{2.8427} & 2.8643 \\
		\textbf{0.6B\(_{\text{Qwen3}}\)} & 2.9056
		& 2.8994 & \textbf{2.8978} \\
		\textbf{4B\(_{\text{Qwen3}}\)} & 2.7636
		& 2.7627 & \textbf{2.7604} \\
		\bottomrule
	\end{tabular}
\end{table}

\begin{table}[t]
	\centering
	\caption{Comparison of standard Newton--Schulz and PolarExpress for polar approximation on Qwen3$_{4B}$.}
	\label{tab:polarization}
	\begin{tabular}{lcc}
		\toprule
		\textbf{Stacking layer number}&\textbf{Method} &   \textbf{Final val} \\
		\midrule
		2&Newton--Schulz (\(J=5\)) &  2.7629 \\
		&PolarExpress &  2.7638 \\
		\midrule
		Full&Newton--Schulz (\(J=5\)) &  2.7598 \\
		&PolarExpress &  2.7569 \\
		\bottomrule
	\end{tabular}
\end{table}

\paragraph{Polarization methods.}
Our optimizer relies on a polar map, which is approximated in practice by Newton--Schulz-type polynomial iterations. Exact SVD is prohibitively expensive inside every optimizer step, so a fast iterative approximation is required. The choice of this approximation is therefore a practical design question: if the optimizer's performance depends strongly on which polar iteration is used, then the method would inherit an additional hyperparameter with no principled default.

We compare the standard Newton--Schulz iteration with \(J=5\) steps against the PolarExpress approximation~\citep{amsel2026polar}, which designs a stage-dependent minimax polynomial sequence for rapid polar approximation and is directly applicable to the real transformed slices produced by an orthonormal DCT. Table~\ref{tab:polarization} reports the results for two representative configurations on Qwen3-4B: \(K=2\) with the DecorQ transform, and full-layer stacking with the DecorQ transform.

For \(K=2\), the two methods perform comparably: Newton--Schulz reaches \(2.7629\) and PolarExpress \(2.7638\), a difference of \(0.0009\). For full-layer stacking, PolarExpress is slightly better: \(2.7569\) against \(2.7598\), a difference of \(0.0029\). Notably, the sign of the difference flips between the two settings, so neither method dominates. In both cases, the gap is an order of magnitude smaller than the gain obtained from our design choices; for reference, the gain from DCT to DecorQ is \(0.047\). This confirms that the polarization method is a second-order factor in our framework rather than a critical one.

This robustness has a practical implication: practitioners can freely swap the polarization method according to their computational budget without jeopardizing the main conclusions. The performance of our method is driven by the dynamic layer transform and the coupling structure, not by the specific polar iteration used to approximate the polar factor.

Notably, PolarExpress is optimized for \emph{real} transformed slices produced by DCT~\citep{amsel2026polar}, a condition that holds exactly in our setting. The small residual differences at both stacking sizes suggest that its approximation error remains modest even when matrix sizes are larger and the layer structure is more complex. In practice, PolarExpress offers a favorable trade-off: it reduces the computational cost of the polar approximation while maintaining competitive performance, especially in the full-layer regime where our best configuration resides. For the main results, we retain the standard Newton--Schulz iteration for simplicity and reproducibility, while noting that PolarExpress is a viable drop-in replacement for deployment scenarios with stringent computational budgets.

\subsection{Training time comparison}
\label{sec:stability-cost-appendix}


We compare the per-iteration step time of AdamW, Muon, TEON, and our method.
All experiments are conducted on identical hardware and software environments to ensure a fair comparison. For each model, we fix the global batch size, sequence length, and the number of training iterations. We report the average time per iteration (ms), which excludes the first few iterations to eliminate the influence of compilation, caching, and initialization overhead.

For each model, all optimizers share the same backbone, data pipeline, and
parallelism configuration. Only the optimizer-specific computation and update 	are varied.

\begin{table}[t]
	\centering
	\caption{Time per iteration (ms) on four models. Lower values indicate faster training.
		Our method incurs only marginal overhead compared with
		AdamW, Muon, and TEON.}
	\label{tab:step_time-appendix}
	\small
	\begin{tabular}{lcccc}
		\toprule
		\multirow{2}{*}{Model} & \multicolumn{4}{c}{Time per Iteration (ms)} \\
		\cmidrule(lr){2-5}
		& \textbf{124M\(_{\text{GPT-2}}\)} & \textbf{350M\(_{\text{GPT-2}}\)} & \textbf{0.6B\(_{\text{Qwen3}}\)} & \textbf{4B\(_{\text{Qwen3}}\)} \\
		\midrule
		AdamW & \textbf{152.74} & \textbf{379.95} & \textbf{3265.5} & \textbf{5466.9} \\
		Muon  & 160.88 & 404.81 & 3324.3 & 5794.7 \\
		TEON  & 161.81 & 408.30 & 3385.7 & 6081.5 \\
		TUON  & 162.12 & 411.64 & 3378.3 & 6150.7 \\
		\bottomrule
	\end{tabular}
\end{table}

Table~\ref{tab:step_time-appendix} reports the time per iteration (ms) of AdamW, Muon, TEON, and our method on the GPT-2 124M/350M and Qwen3 0.6B/4B models. The goal of this comparison is not to claim a step-time advantage, but to verify that our method does not introduce non-trivial computational overhead relative to widely used optimizers.

First, the overall trend is that, across all four models, the step time of our method is 	close to that of Muon and TEON. On GPT-2 124M, our method takes 162.12\,ms, which is 0.2\% slower than TEON (161.81\,ms) and 0.8\% slower than Muon 	(160.88\,ms). On GPT-2 350M, the corresponding gaps are 0.8\% and 1.7\%, respectively. On the larger Qwen3 models, the differences remain small: on 	Qwen3 0.6B, our method is 0.2\% slower than TEON and 1.6\% slower than Muon. On Qwen3 4B, it is 1.1\% slower than TEON and 6.1\% slower than Muon.

AdamW is the fastest in absolute terms on all four models, which is expected since it only maintains first- and second-order moment estimates and performs element-wise updates. 
Our method, in absolute terms, is slower than AdamW by only 9.4\,ms, 31.7\,ms, 112.8\,ms, and 683.8\,ms per iteration.
Given that a full training run typically involves around 20K-50K
iterations, the relative overhead per iteration is modest and does not
change the overall training cost qualitatively.

When it comes to the scaling behavior, an important observation is that the relative overhead of our method does not grow disproportionately with model size.
Between GPT-2 124M and Qwen3 4B, the model size increases by more than an
order of magnitude, yet the step time of our method remains within 12.5\% of
AdamW and within 1.1\% of TEON on the largest model. This suggests that the
additional computation in our optimizer is well amortized by the forward and
backward passes, and that our method scales gracefully to larger models.

Overall, the step-time results show that our method is essentially
on par with Muon and TEON, and only marginally slower than AdamW. The small
per-iteration overhead is a reasonable price for the optimization benefits
reported in Section~\ref{complexity}, and it does not constitute a
practical bottleneck for large-scale training.

%
%
%

\end{document}